\documentclass[journal,onecolumn]{IEEEtran}

\usepackage{times}

\usepackage{fancyhdr}
\usepackage{xcolor} 
\usepackage{algorithm}
\usepackage{algorithmic}
\usepackage[numbers,sort&compress]{natbib}
\usepackage{eso-pic} 
\usepackage{forloop}
\usepackage{url}
\usepackage{caption}
\usepackage{microtype}
\usepackage{graphicx}
\usepackage{subcaption}
\usepackage{booktabs} 

\usepackage{hyperref}
\usepackage{bookmark}

\usepackage{amsmath}
\usepackage{amssymb}
\usepackage{mathtools}
\usepackage{amsthm}
\usepackage{colortbl}

\usepackage{multirow}

\usepackage{pifont}
\def \Yes {\ding{51}}
\def \No {\ding{55}}

\newcommand{\cA}{\mathcal{A}}
\newcommand{\cB}{\mathcal{B}}

\newcommand{\cD}{\mathcal{D}}

\newcommand{\cH}{\mathcal{H}}

\newcommand{\cX}{\mathcal{X}}

\usepackage[capitalize,noabbrev]{cleveref}

\theoremstyle{plain}
\newtheorem{theorem}{Theorem}[section]
\newtheorem{proposition}[theorem]{Proposition}
\newtheorem{lemma}[theorem]{Lemma}
\newtheorem{corollary}[theorem]{Corollary}
\theoremstyle{definition}
\newtheorem{definition}[theorem]{Definition}
\newtheorem{assumption}[theorem]{Assumption}
\theoremstyle{remark}
\newtheorem{remark}[theorem]{Remark}
\usepackage[textsize=tiny]{todonotes}

\usepackage{enumitem}
\usepackage{tcolorbox}
\usepackage{xcolor}
\definecolor{lightroyalblue}{HTML}{F6F8FD} 
\definecolor{lightorange}{RGB}{252,236,219}
\definecolor{royalblue}{HTML}{4169E1}
\definecolor{lighterblue}{HTML}{f2fafd}  
\newtcolorbox{abox}{colback=lightorange,frame empty}
\definecolor{LightCyan}{rgb}{.9, .95, 1.}
\definecolor{rowblue}{RGB}{232,241,255}
\newcommand{\inner}[2]{\left\langle #1, #2 \right\rangle}
\newcommand{\regret}{\mathfrak{R}}
\newcommand{\tpvr}{{\sc TP-VR-Opt}}
\newcommand{\tpvrplus}{{\sc TP-VR-Opt+}}
\newcommand{\tpvrplusplus}{{\sc TP-VR-Opt++}}
\newcommand{\argmin}{\mathop{\mathrm{argmin}}}

\title{Bandit Convex Optimization with Gradient Prediction Adaptivity}

\begin{document}

\author{\textbf{Shuche Wang}\IEEEauthorrefmark{1},  
\textbf{Adarsh Barik}\IEEEauthorrefmark{2},
and \textbf{Vincent Y.~F.~Tan}\IEEEauthorrefmark{1}\IEEEauthorrefmark{3}\\
\IEEEauthorblockA{
\IEEEauthorrefmark{1}  Department of Mathematics, National University of Singapore, Singapore \\[0.5mm]
\IEEEauthorrefmark{2} Department of Computer Science and Engineering, Indian Institute of Technology Delhi, India\\[0.5mm]
\IEEEauthorrefmark{3} Department of Electrical and Computer Engineering, National University of Singapore, Singapore\\[0.5mm]
}
{Emails:\, shuche.wang@u.nus.edu, adarshbarik1@iitd.ac.in, vtan@nus.edu.sg}}




\maketitle

\begin{abstract}
Bandit convex optimization (BCO) is a fundamental online learning framework with partial feedback, where the learner observes only the loss incurred at the chosen decision point in each round. In this work, we investigate whether optimistic gradient predictions can improve worst-case regret guarantees in a prediction-adaptive manner. Specifically, given gradient predictions $m_t$, we seek regret bounds that scale with the cumulative prediction error
$
S_T=\sum_{t=1}^T \|\nabla f_t(x_t)-m_t\|^2.
$
We first establish a negative result: under the single-point feedback protocol, an unavoidable $\Omega(\sqrt{T})$ regret lower bound persists even when $S_T=o(T)$, showing that the variance of gradient estimation fundamentally obscures the benefit of accurate predictions. To overcome this barrier, we propose \emph{Two-Point Variance-Reduced Optimistic Gradient Descent} (TP-VR-OPT) for the two-point feedback setting. The key idea is a novel variance-reduced gradient estimator whose variance scales with the prediction error rather than the gradient norm. This yields a regret bound of
$
O\big(\sqrt{d\,\mathbb{E}[S_T]}\big),
$
where $d$ is the decision dimension. Complementing this result, we establish an information-theoretic lower bound that scales as $\Omega(\sqrt{\mathbb{E}[S_T]})$, providing a fundamental characterization of the best achievable prediction-adaptive regret and showing that TP-VR-OPT is optimal up to a factor of $\sqrt d$.   We further develop adaptive variants that eliminate the need for prior knowledge of $\mathbb{E}[S_T]$ or the horizon $T$, and extend our framework to non-stationary environments, establishing dynamic regret guarantees that adapt simultaneously to the cumulative prediction error and the comparator path length.
\end{abstract}

    \section{Introduction}
\label{sec:intro} Online convex optimization (OCO) provides a powerful and unified theoretical framework for sequential decision-making problems~\cite{cesa2006prediction,hazan2016introduction,orabona2019modern,cesa2004generalization}. In the OCO setting, at each round $t\in[T]$, where $[T]=\{1,\ldots,T\}$, the learner selects a decision point $x_t\in \cX$ from a convex set $\cX\subseteq \mathbb{R}^d$. The environment then reveals a loss function $f_t:\cX\rightarrow \mathbb{R}$, after which the learner incurs a loss $f_t(x_t)$ and updates the decision point to $x_{t+1}$. The performance is measured via the \emph{regret}, defined as the cumulative loss relative to the best fixed decision in hindsight.

In many practical applications, however, the learner does not have access to the full loss function or its gradient. Instead, a more common scenario is that the environment only reveals the loss incurred at the chosen decision point, rather than the entire function. This setting, in which the learner only observes the zero-order (function value) information, is known as \emph{Bandit Convex Optimization (BCO)}~\cite{flaxman2004online,lattimore2024bandit}. The central challenge in BCO is that, due to the lack of direct gradient information, the learner must construct gradient estimates using random perturbations and a limited number of loss evaluations to update the learner's decision.

From a worst-case perspective, in the OCO setting, an $O(\sqrt{T})$ static regret is typically achievable under the full-information feedback protocol, which serves as a standard baseline in online learning. In contrast, due to the limited feedback available in BCO, regret bounds are generally worse as a consequence of the stability-variance tradeoff. Specifically, under the two-point feedback protocol, one can still achieve $O(\sqrt{T})$ regret bound~\cite{shamir2017optimal}, whereas under the single-point feedback protocol, the bounds are $O(T^{3/4})$~\cite{flaxman2004online} for Lipschitz loss functions and $O(T^{2/3})$~\cite{saha2011improved} for smooth loss functions. However, all of these guarantees are worst-case. Even when the environment is significantly more benign in practice, such analyses yield only worst-case regret bounds and do not adapt to favorable structure in the loss sequence.

In contrast, an important line of work in the full-information OCO setting shifts the focus from purely worst-case optimization to adaptivity to benign environments, while still retaining robustness to adversarial scenarios~\cite{chiang2012online,zhao2024adaptivity}. A standard approach to achieving such adaptivity is through optimism or prediction~\cite{rakhlin2013online,steinhardt2014adaptivity}. Specifically, before the beginning of each round, the learner has the ability to compute a prediction $m_t$ of the next gradient based on historical information and incorporate this prediction into the update rule. Therefore, the regret bound need not scale solely with $T$; instead, it can be expressed in terms of the cumulative prediction error, typically of the form $S_T=\sum_{t=1}^T \|\nabla f_t(x_t)-m_t\|^2$. Intuitively, when the environment becomes predictable, slowly changing, or exhibits small gradient variation, the predictions $m_t$ can be accurate, leading to regret bounds that are significantly smaller than the worst-case $O(\sqrt{T})$ guarantee. If instead   the environment is fully adversarial, such methods still recover the standard worst-case performance.

One natural question is: \emph{Can prediction-based adaptivity and acceleration in full-information OCO be extended to the BCO setting?} In other words, can we design algorithms whose regret is dominated by the prediction error, rather than being fundamentally limited by the worst-case dependence on the time horizon $T$ under  bandit feedback? Furthermore, is it possible to construct a mechanism under which the regret in BCO is significantly better than $O(\sqrt{T})$ when the prediction is accurate, and then achieve  acceleration on favorable instances?

The key challenge is that the gradient estimator in BCO inherently exhibits non-negligible randomness and variance, whereas the benefit of prediction relies on the algorithm’s ability to detect and exploit small residual errors. If the noise in the gradient estimator dominates  this residual, then even highly accurate predictions $m_t$ cannot be effectively leveraged in the update. Consequently, to achieve the adaptive regret bound in terms of prediction error, it is not sufficient to merely incorporate $m_t$ into the update rule. The gradient estimator itself must be carefully designed such that the variance of the estimator decreases as the prediction residual becomes smaller. This observation suggests that the central question in prediction-adaptive BCO is not only how to use the prediction in the update, but also whether the feedback protocol makes the prediction residual observable with sufficiently small variance.

Based on the above motivations, we study the BCO problem with optimistic gradient prediction. We define the cumulative prediction error and its expectation, respectively, as 
\begin{equation}\label{eq:defst}
    S_T=\sum_{t=1}^T \|\nabla f_t(x_t)-m_t\|^2\quad\mbox{and}\quad \bar S_T=\mathbb{E}[S_T].
\end{equation} 
Note that the expectation is with respect to the (random) choice of the action $x_t$. 
Our   objective is to obtain regret guarantees that depend purely on $\bar S_T$. When $\bar S_T=o(T)$, the regret should improve accordingly, while for large $\bar S_T$, the algorithm should inherit the $O(\sqrt{T})$ worst-case performance. Furthermore, considering   practical non-stationary environments, we also study the \emph{dynamic regret}, which allows the regret to jointly reflect both the variation of the environment and the cumulative prediction error. The  variation of the environment is captured by  the path length $P_T(u_{1:T})=\sum_{t=2}^T \|u_t -u_{t-1}\|$ of the comparator sequence $\{u_t\}_{t=1}^T$. 

\subsection{Main Contributions}

We   summarize our main contributions in this section.


\definecolor{LightCyan}{rgb}{0.88,1,1}
\newcolumntype{g}{>{\columncolor{LightCyan}}c}

\begin{table*}[t!]
    \centering
    \caption{Summary of the regret bounds of our proposed algorithms and comparisons with most related works. The notation $\widetilde{O}(\cdot)$ hides polylogarithmic factors. Our algorithms leverage the smoothness assumption to achieve prediction adaptivity, whereas baselines typically assume only Lipschitz continuity.}
    \label{tab:summary}
    \renewcommand{\arraystretch}{0.8} 
    \resizebox{\textwidth}{!}{%
    \begin{tabular}{c g ggg}
    \hline
    \toprule
    \rowcolor{white} \textbf{Measure} &  \textbf{Algorithm} &  \textbf{Prediction Adaptivity} &   \textbf{Free of $\bar S_T/T/P_T$} &  \textbf{Regret Guarantees} \\
    \midrule
    \rowcolor{white} \multirow{3}{*}{\shortstack{Static\\Regret}} 
    &  Shamir~\cite{shamir2017optimal} & \No & \No & $O(\sqrt{d T})$ \\
    &  \textbf{\tpvr} \scriptsize{(Theorem~\ref{thm:two_point})} & \Yes & \No & $O(\sqrt{d\bar S_T})$ \\
    &  \textbf{\tpvrplus} \scriptsize{(Theorem~\ref{thm:tp-vr-adaptive})} & \Yes & \Yes & $\widetilde{O}(\sqrt{d\bar S_T})$ \\
    \midrule
    \rowcolor{white} 
    &  Zhao et al.~\cite{zhao2021bandit} & \No & \No & $O(d\sqrt{T(1+ P_{T})})$ \\
    \rowcolor{white} &  He et al.~\cite{he2025non} & \No & \No & $O(\sqrt{d\,T(1+ P_{T})})$ \\
    &  \textbf{\tpvr} \scriptsize{(Theorem~\ref{thm:dynamic_regret})} & \Yes & \No & $O(\sqrt{d\,\bar{S}_{T}(1+ P_{T})})$ \\
    \multirow{-4}{*}{\shortstack{Dynamic\\Regret}} &  \textbf{\tpvrplusplus} \scriptsize{(Theorem~\ref{thm:tp-vr-opt-pp})} & \Yes & \Yes & $\widetilde{O}(\sqrt{d\,\bar{S}_{T}(1+ P_{T})})$ \\
    \bottomrule
    \end{tabular}%
    }
\end{table*}

$\bullet$ \textbf{Single-Point Feedback Barrier.} We first show that under the single-point feedback protocol, one cannot overcome the unavoidable $\Omega(\sqrt{T})$ barrier, even when the cumulative prediction error is negligible, i.e., $\bar{S}_T=o(T)$. This lower bound reveals an information-theoretic bottleneck insofar that the inherently large variance of the single-point estimator masks the benefit of the optimistic prediction, thereby preventing the establishment of an adaptive regret bound that depends purely on $S_T$. This motivates our study of the two-point feedback protocol in the following sections.

$\bullet$  \textbf{Optimal Adaptivity with Two-Point Feedback.} To overcome the $\Omega(\sqrt{T})$ barrier under the single-point feedback protocol, we turn to studying the case of two-point feedback. We propose a new algorithm, \textbf{T}wo-\textbf{P}oint \textbf{V}ariance-\textbf{R}educed \textbf{Opt}imistic Gradient Descent (\tpvr). The key difference between our proposed algorithm compared to classical two-point feedback estimators lies in how we introduce the optimistic prediction when designing the estimator. Instead of directly estimating the gradient $\nabla f_t(x_t)$, we incorporate the optimistic prediction $m_t$ as a baseline and construct an estimator for the residual $\nabla f_t(x_t)-m_t$. This design ensures that the estimation variance scales with the prediction error $\|\nabla f_t(x_t)-m_t\|^2$, rather than the gradient norm. As a consequence, our proposed algorithm \tpvr\, achieves the regret bound $O(\sqrt{d\bar S_T})$. We also show that this regret bound  matches the information-theoretic lower bound  up to a 
$\sqrt{d}$ factor. 

$\bullet$  \textbf{Adaptive Implementation via Observable Residual Doubling.} The optimal regret bound of \tpvr\, relies on  tuning certain parameters with oracle information of the expectation of the cumulative prediction error $\bar S_T$ and the horizon $T$. To make our proposed algorithm practically applicable, we propose \tpvrplus\, using a nested doubling-trick approach, which can achieve the same regret guarantee as \tpvr\, up to logarithmic factors, without requiring any prior knowledge of $\bar S_T$ and $T$. Specifically, the bandit-specific issue
is that $\bar S_T$ is unobserved; we therefore use the observable residual
$\sum_t\|\widehat g_t-m_t\|^2$ as the doubling trigger.

$\bullet$  \textbf{Applications to Dynamic Regret.} To demonstrate the effectiveness of our proposed prediction-dependent algorithms in non-stationary environments, we extend our analysis from static regret to dynamic regret. We first show that our proposed \tpvr\, can achieve a dynamic regret bound $O(\sqrt{d\,\bar{S}_{T}(D^{2}+D P_{T})})$, which depends jointly on the cumulative prediction error $\bar S_T$ and the comparator path-length $P_T$. To remove the need for prior knowledge of these quantities, we further propose \tpvrplusplus\, built upon \tpvrplus, which leverages the meta-expert framework by maintaining and aggregating a grid of candidate step sizes to adapt to the best choice in hindsight. We show that \tpvrplusplus\, attains the same dynamic regret bound $\tilde{O}(\sqrt{d\,\bar{S}_{T}(D^{2}+D P_{T})})$
up to logarithmic factors, without requiring any prior knowledge of $\bar S_T$, $P_T$, and $T$.

A summary of our results, together with  comparisons with related existing results, is included in Table~\ref{tab:summary}.

\subsection{Related Works}
\textbf{Bandit Convex Optimization (BCO).} BCO can be viewed as a fundamental extension of OCO under partial feedback~\cite{lattimore2024bandit}. Two standard feedback protocols are commonly studied in BCO: single-point feedback and two-point feedback. Under the single-point feedback protocol, the learner observes only a single loss value $f_t(x_t)$ at each round and constructs a biased or unbiased gradient estimator using the smoothed version of $f_t$ together with a random perturbation direction. The seminal work \cite{flaxman2004online} introduced the framework of gradient descent without gradient information. The static regret bounds under single-point feedback are $O(T^{3/4})$~\cite{flaxman2004online} for Lipschitz loss functions and $O(T^{2/3})$ for smooth~\cite{saha2011improved} or strongly-convex~\cite{agarwal2010optimal} loss functions. When loss functions are both smooth and strongly convex, an $O(\sqrt{T\log T})$ regret bound is achieved~\cite{hazan2014bandit}. Moreover, for general convex losses, the regret has been further improved to $O(\mathrm{poly}(\log T)\sqrt{T})$, albeit with a strong dependence on the dimension $d$~\cite{bubeck2015bandit,bubeck2021kernel,lattimore2020improved}. Under the two-point feedback protocol, the learner can query the loss at two nearby points and construct a symmetric difference estimator, which significantly reduces the estimation variance and yields an $O(\sqrt{T})$ order regret bound~\cite{shamir2017optimal}. Lower bounds for BCO under various feedback models have been studied in \cite{shamir2013complexity,hu2016bandit}. 

\textbf{Prediction/Optimism Adaptivity.} In the full information OCO
setting, the regret can be expressed in terms of the prediction error rather than the time horizon $T$ by employing optimistic updates when the environment exhibits a predictable structure. A canonical framework for this approach is the optimistic mirror descent (OMD) algorithm~\cite{rakhlin2013online,rakhlin2013optimization}. Another important line of work studies regret bounds based on gradual variation. Chiang et al.~\cite{chiang2012online} introduced bounds based on gradient variation as a problem-dependent measure. The bounds based on gradient variation imply small-loss bounds~\cite{zhao2024adaptivity}. More broadly, prediction can be viewed as additional hints provided to the learner~\cite{dekel2017online,bhaskara2020online}. Recently, the {\em Stochastically Extended Adversarial} or SEA  model~\cite{sachs2023accelerated} was proposed to bridge the gap between adversarial and stochastic OCO.

Incorporating prediction and optimism into bandit feedback presents challenges. Yang and Mohri~\cite{yang2016optimistic} studied optimistic bandit convex optimization by constructing predictions from historical gradients. Chiang et al.~\cite{chiang2013beating} obtained a regret bound in terms of gradient variation under two-point feedback. Wei and Luo~\cite{wei2018more} and Rakhlin and Sridharan~\cite{rakhlin2013online} studied the multi-arm bandits and linear bandits under this setting, respectively. Wei et al.~\cite{wei2020taking} achieved an improved regret bound by leveraging loss predictors in the contextual bandit setting. Sequential prediction and regret under probabilistic losses have also been
studied from an information-theoretic and Bayesian perspective~\cite{wu2023regret}. In contrast, our predictions are gradient-valued hints
used under bandit convex feedback, and the main challenge is to make the
estimation variance scale with the gradient-prediction residual. Recent work has also highlighted connections between regret guarantees and
time-uniform statistical guarantees in sequential decision problems, for
example through universal-portfolio-based confidence sequences~\cite{orabona2023tight}. Our focus is different that we use gradient predictions
to reduce the variance of bandit gradient estimators, so that regret scales
with the gradient-prediction residual.

\textbf{Dynamic Regret.} The dynamic regret with path-length measure was first introduced by \cite{zinkevich2003online} and studied in \cite{zinkevich2003online,gyorgy2016shifting,zhang2022no}. Zhang et al.~\cite{zhang2018adaptive} obtained dynamic regret bounds that match the lower bound using an online ensemble framework. Zhao et al.~\cite{zhao2024adaptivity} incorporated gradient variation into dynamic regret analysis and derived problem-dependent dynamic regret bounds. In the bandit feedback case, Zhao et al.~\cite{zhao2021bandit} systematically studied the dynamic regret under the single-point and two-point feedback protocols. Recently, He et al.~\cite{he2025non} improved two-point feedback dynamic regret bounds of \cite{zhao2021bandit} by a factor of $O(\sqrt{d})$. The use of expert aggregation to compete with time-varying benchmarks is also
closely related to the tracking-the-best-expert literature~\cite{gyorgy2012efficient}. Our meta-expert layer differs in that each expert
is itself a prediction-centered two-point bandit algorithm with a different
stepsize.

    \section{Problem Setup and Preliminaries}
\label{sec:setup}

In this section, we describe the problem setup of  BCO with \emph{prediction} and present notations and preliminaries.  

\textbf{Notation.}
Let $[T]=\{1,\dots,T\}$. We denote $\langle \cdot,\cdot\rangle$ for the Euclidean
inner product and $\|\cdot\|$ for the Euclidean norm. Let
$\cB=\{x\in\mathbb{R}^{d}:\|x\|\le 1\}$ be the unit ball and
$\mathbb{S}^{d-1}=\{v\in\mathbb{R}^{d}:\|v\|=1\}$ be the unit sphere.

The learner chooses points in a nonempty convex set $\cX\subset \mathbb{R}^{d}$
with finite diameter i.e., $D=\sup_{x,y\in\cX}\|x-y\|< \infty$. We assume that the
convex set $\cX$ is well-rounded, i.e., there exist $0< r\leq R<\infty$ such
that $r\cB\subseteq \cX\subseteq R\cB$. We also assume that $0\in\cX$.

At each round  $t\in[T]$, the environment selects a convex function
$f_{t}:\cX\rightarrow \mathbb{R}$. The learner incurs a loss $f_{t}(x_{t})$ when
its action is $x_{t}$. The performance metric is the {\em regret} 
\begin{equation}
    \label{eq:regret_def}\regret_{T}=\sum_{t=1}^{T}f_{t}(x_{t})-\min_{x\in\cX}\sum
    _{t=1}^{T}f_{t}(x).
\end{equation}
For a fixed comparator $u\in\cX$, we also write the regret as
$\regret_{T}(u)=\sum_{t=1}^{T}(f_{t}(x_{t})-f_{t}(u))$.

\begin{assumption}
    [Lipschitz Continuity] \label{ass:lipschitz} For all $t \in [T]$,
    the loss function $f_{t}: \cX \to \mathbb{R}$ is convex on $\cX$ and $L$-Lipschitz
    continuous. That is, $|f_{t}(x) - f_{t}(y)| \le L \|x - y\|$.
    Therefore, if $f_{t}$ is differentiable, the gradient is
    bounded as $\|\nabla f_{t}(x)\|\le L$ for all $x\in\cX$.
\end{assumption}

\begin{assumption}
    [Smoothness]\label{ass:smoothness} For all $t \in [T]$, the loss function $f_{t}$
    is differentiable on $\cX$ and $\beta$-smooth. That is, for all $x, y \in \cX$: $\|\nabla f_{t}(x) - \nabla f_{t}(y)\|\le \beta \|x - y\|$.
\end{assumption}

\begin{remark}[The necessity of the smoothness assumption]
The worst-case baselines listed in Table~\ref{tab:summary} typically rely only on the Lipschitz continuity without requiring the smoothness assumption as stated in Assumption \ref{ass:smoothness}. However, to achieve the prediction adaptive regret bound that depends solely on $\bar S_T$ defined in \eqref{eq:defst}, we impose the smoothness assumption throughout this work. This smoothness assumption is also consistent with the most closely related
variation-adaptive and prediction-adaptive literature. In the bandit
setting, gradient-variation bounds such as~\cite{chiang2013beating,yu2026improved} rely on smoothness to control the
bias introduced by local function-value queries. In full-information
settings, analogous prediction or gradient variation-adaptive bounds are also commonly obtained under smoothness assumptions~\cite{zhang2018adaptive,zhao2024adaptivity,sachs2023accelerated}.

Under the two-point feedback protocol, the ``local first-order information'' observed along a direction $v$ at scale $\delta$ is given by the local secant slope $\big(f_t(y_t+\delta v)-f_t(y_t-\delta v)\big)/(2\delta)$. In worst-case analysis with only the Lipschitz continuity assumption, this symmetric difference can result in the local secant slope being uniformly bounded by the Lipschitz constant $L$, which implies that $\|\hat{g}_t\|\leq dL$ for the gradient estimator $\hat{g}_t$. Plugging this into the standard stability–variance tradeoff yields $O(\sqrt{T})$ regret guarantee.

Our objective is substantially stronger: we wish to bound the regret in terms of only the expected gradient-prediction error $\bar S_T$. The natural observable residual signal associated with  direction $v$ is the difference between the secant slope and the predicted directional derivative $\inner{m}{v}$, i.e.,
\begin{equation*}
    \frac{f(y + \delta v)-f(y-\delta v)}{2\delta} - \langle m, v\rangle=\langle \nabla f(y)-m, v\rangle + \frac{r(v)}{2\delta},
\end{equation*}
where $r(v)=f(y+\delta v)-f(y-\delta v)-2\delta\langle \nabla f(y),v\rangle$ is the observable residual. To ensure this observable residual is small and thus to obtain the regret in terms of prediction adaptivity, we should ensure that $r(v)$ vanishes as $\delta\to 0$. Under the smoothness condition, we have $|r(v)|\leq \beta\delta^2$ implying $|r(v)|/2\delta=O(\beta\delta)$. However, under Lipschitz continuity alone, it is possible that $|r(v)|/(2\delta)=\Theta(L)$, i.e., this observable residual remains constant  and non-vanishing (as $\delta\to0$) even when $m=\nabla f(y)$, as shown in Appendix~\ref{app:smooth}. Therefore, this term introduces a constant independent of  the prediction error, leading to an  $\Omega(\sqrt{T})$ order floor akin to that in the worst-case analysis.  This justifies the need for us to impose Assumption \ref{ass:smoothness}.
\end{remark}

We now formally describe the standard BCO protocol. 

\textbf{One-point and two-point bandit feedback.}
We adopt a standard ``internal point + perturbation'' protocol throughout this work.
We denote the perturbation radius as $\delta\in(0,r)$. To ensure feasibility, we
define the shrinkage coefficient $\alpha=\delta/r$ and the shrunk decision set
$\cX_{\alpha}=(1-\alpha)\cX$.

\emph{At each round $t$}: The learner chooses a center point $y_{t}\in\cX_{\alpha}$ and samples a direction
$v_{t}\sim \text{Unif}(\mathbb{S}^{d-1})$, where $\text{Unif}(\mathbb S^{d-1})$ denotes the uniform distribution on
$\mathbb S^{d-1}$. The learner plays $x_{t}=y_{t}+\delta v_{t}$. Under the single-point feedback, the learner observes $f_{t}(x_{t})$, while the learner can observe $f_{t}(y_{t}+\delta v_{t})$ and $f_{t}(y_{t}-\delta v_{t})$ under the two-point feedback.

Since $y_{t}\in(1-\alpha)\cX$ and $\delta=\alpha r$, we have $(1-\alpha)\cX+\alpha r \cB \subseteq (1-\alpha)\cX+\alpha \cX=\cX$, which implies $y_{t}+\delta v_{t}\in \cX$ for all $v_t$ such that $\|v_{t}\|\leq 1$.
\begin{remark}
    In the two-point setting, $y_{t}-\delta v_{t}$ is an auxiliary query point, used only for gradient estimation. For computing regret, we only consider loss at the point $x_t = y_t + \delta v_t$ in this work. 
    We note that for the two-point feedback bandit convex optimization problem, different types of regret definitions
    are considered in the literature. Some works measure performance at the
    center point $y_{t}$~\citep{shamir2017optimal} as $\sum_{t=1}^{T}f_{t}(y_{t})
    - \min_{x \in \cX}\sum_{t=1}^{T}f_{t}(x)$, while others consider the average
    loss over the query pair~\citep{zhao2021bandit}:
    $\sum_{t=1}^{T}\left(\frac{1}{2}f_{t}(y_{t}+\delta v_{t}) + \frac{1}{2}f_{t}(
    y_{t}-\delta v_{t})\right) - \min_{x \in \cX}\sum_{t=1}^{T}f_{t}(x)$. However,
    since each $f_{t}$ is assumed to be $L$-Lipschitz and the perturbation
    radius is controlled by $\|x_{t}- y_{t}\| = \delta$, these definitions are equivalent
    up to lower-order terms of order $O(L \delta T)$. For convenience and to
    reflect the true cost, we adopt the definition as Eqn.~\eqref{eq:regret_def}.
\end{remark}

Below, we will formally define the prediction settings.

\textbf{History and Prediction.}
Let $\cH_{t-1}$ denote the $\sigma$-field generated by all historical information till
time $t-1$, which includes the sequence of actions, predictions, and observed feedbacks up to round $t-1$, i.e., $\mathcal{H}_{t-1} = \sigma\big(\{y_{\tau}, m_{\tau}, f_{\tau}(y_{\tau}\!+\!\delta v_{\tau})\}_{\tau=1}^{t-1}\big)$ for single-point feedback, or $\mathcal{H}_{t-1} = \sigma\big(\{y_{\tau}, m_{\tau}, f_{\tau}(y_{\tau} \pm \delta v_{\tau})\}_{\tau=1}^{t-1}\big)$ for two-point feedback. Before taking the action at round $t$, the learner generates a raw prediction
vector $\widetilde m_t\in\mathbb R^d$, which is $\mathcal H_{t-1}$-measurable.
Throughout the algorithm and the analysis, we use its clipped version $m_t = \Pi_{L\mathbb B}(\widetilde m_t)$, where $L\mathbb B=\{x\in\mathbb R^d:\|x\|\le L\}$. Hence $\|m_t\|\le L$.
This clipping is without loss of generality. By Assumption~\ref{ass:lipschitz},
$\|\nabla f_t(x)\|\le L$ for all $x\in\mathcal X$, and Euclidean projection onto
the ball $L\mathbb B$ does not increase the distance to any vector in $L\mathbb B$.
Therefore, replacing $\widetilde m_t$ by $m_t$ does not increase the prediction
error $\|\nabla f_t(x)-m_t\|$.


This question of how to exploit a history-based gradient prediction sequence to improve regret is well-studied in the full-information OCO setting. 
In particular, optimistic algorithms, such as Optimistic Mirror Descent (OMD)~\cite{rakhlin2013online}, leverage past gradient information to adapt the updates so that the regret scales with the prediction error rather than the horizon. However, in the bandit setting, it is not immediately clear that such predictors can be effectively exploited.
Nevertheless, it is straightforward to construct admissible predictors using past observations: for example, a \emph{coordinate-persistent predictor} stores the most recent finite-difference estimate for each coordinate. 
At each round, the predicted directional derivative can be compared against the observed two-point slope, allowing a variance-reduced residual estimator to be constructed whose variance scales with the prediction error $\sum_{t=1}^T \|\nabla f_t(x_t)-m_t\|^2$ rather than the gradient magnitude itself. 
This ensures that using $m_t$ remains meaningful and well-defined even in the bandit context.

This viewpoint is standard in the hints/prediction literature. 
For instance, Wei et al.~\cite{wei2020taking} study contextual bandits with predictors available before action selection, while Lyu and Cheung~\cite{lyu2023online} consider online resource allocation with bandit feedback and advice. 
Our focus is the complementary question: given an arbitrary admissible predictor, can a bandit algorithm exploit it so that the regret scales with the realized prediction error? 
Specifically, in Section~\ref{sec:comp_chiang}, we further illustrate that, under coordinate sampling, a coordinate-persistent predictor recovers the gradually evolving worlds result of \cite{chiang2013beating} as a special case. 
Thus, our framework generalizes prior constructions while remaining fully implementable in the bandit setting.

    \section{The Single-Point Feedback Barrier}
\label{sec:single_point}

The main goal of this work is to bound the regret in terms of the cumulative prediction error $S_{T}=\sum_{t=1}^{T}\|\nabla f_{t}(x_{t})-m_{t}\|^{2}$, which
effectively replaces the dependence on the horizon $T$. The worst-case
regret bound with smooth convex function under single-point feedback is $O(T^{2/3})$~\citep{saha2011improved}. Ideally, one
might expect that with the help of optimistic predictions, this regret could be
improved to an adaptive bound of order $O(\bar S_{T}^{2/3})$ in this setting where $\bar{S}_{T}= \mathbb{E}[S_{T}]$. However,
we will show that obtaining such a purely $\bar S_{T}$-dependent bound is impossible under single-point feedback. We show an unavoidable lower bound on the regret to be $\Omega(\sqrt{T})$. This holds
even in favorable environments where the cumulative prediction error is negligible, i.e., $S_{T}=o(T)$, which implies that the $\Omega(\sqrt{T})$ cost is inherent in the one-point feedback mechanism regardless
of the prediction accuracy.

We capture this single-point feedback barrier via a two-hypothesis test with linear
losses together with independent and identically distributed Gaussian noise components.

\begin{theorem}
    \label{thm:single_point_barrier} Fix a convex set $\cX\subset \mathbb{R}^{d}$
    with Euclidean diameter $D$. Given $T\ge 1$, a Lipschitz constant $L>0$, and
    a noise scale $\sigma > 0$ such that $\sigma \le 2LD\sqrt{T}$. In the single-point
    feedback setting, for any algorithm $\cA$ with optimistic prediction $m_{t}$,
    there exists a random environment for which each $f_{t}$ is convex, $L$-Lipschitz,
    and $\beta$-smooth, and the expected regret is lower bounded as
    \begin{equation*}
        \mathbb{E}[\regret_{T}]\geq \frac{3}{16}\sigma\sqrt{T}.
    \end{equation*}
\end{theorem}

Theorem~\ref{thm:single_point_barrier} demonstrates that  additive offsets
create an information-theoretic bottleneck for single-point feedback, enforcing a
universal $\Omega(\sqrt{T})$ lower bound. Specifically, we construct a hard instance with a linear loss function $f_t(x) = -M \varepsilon \langle v, x - x_0 \rangle + \xi_t$, where $M\in \{-1,+1\}$ is sampled uniformly at random  and then  fixed for all rounds, $\xi_t \sim \mathcal{N}(0, \sigma^2)$ and $\varepsilon=\sigma/(2D\sqrt{T})$. Since the gradient is constant over time and has norm $\varepsilon$, the
best history-based cumulative prediction error is at most $T\varepsilon^2=\sigma^2/(4D^2)=O(1)$. 

Hence, Theorem~\ref{thm:single_point_barrier} shows that the obstruction under
single-point feedback is observational rather than algorithmic. Even if the
gradient sequence is highly predictable, the learner only observes an absolute
function value at the queried point. Such an observation mixes the local
first-order signal with zero-order offsets. In the hard instance above, the
linear component determines the optimal decision, but the additive offset makes
the two possible gradient directions statistically difficult to distinguish from
single-point observations. Consequently, a small prediction error $S_T$ alone does not guarantee that the
learner can exploit the prediction. The predictor may be accurate in the sense
that $\nabla f_t(x_t)-m_t$ is small, but the bandit observation may still fail to
reveal this residual with sufficiently small variance. Therefore, no regret
bound depending only on $S_T$ can hold in the single-point model. This motivates
the two-point feedback protocol studied next: by comparing two symmetric
function values, the learner can cancel common zero-order offsets and directly
estimate a local residual relative to the prediction.

    \section{Two-Point Feedback: Prediction-sensitive Variance Reduction}
\label{sec:twopoint_vr}

The lower bound established in Theorem~\ref{thm:single_point_barrier} reveals a fundamental limitation of single-point feedback. Because the learner relies heavily on absolute function values to estimate gradients and lacks a comparative baseline within each epoch, it cannot distinguish between local gradient changes and zero-order noisy bases (e.g., additive offsets $\xi_{t}$).

In this section, we overcome this information bottleneck by considering the
\emph{two-point feedback} model. The structural advantage of this protocol lies
in its use of \emph{differences}. Through the query of two relevant points, the learner
is able to filter out the zero-order noise effectively and focus on the local variations
of the function. This process overcomes the $\Omega(\sqrt{T})$ lower bound and achieves an adaptive regret bound which solely depends on $\bar S_{T}$.

\subsection{The \tpvr\, Algorithm}
\label{sec:tp-vr-alg}

Recalling the two-point feedback protocol shown in Section \ref{sec:setup}: In
round $t$, the learner selects a center point $y_{t}\in \cX_{\alpha}$, samples a
random direction $v_{t}\sim \mathrm{Unif}(\mathbb{S}^{d-1})$, and performs a perturbation
on the point $x_{t}=y_{t}+\delta v_{t}\in \cX$. Under the two-point feedback
mechanism, the learner observes the function values at two symmetrical points around
the center point: $f_{t}(y_{t}+\delta v_{t})$ and $f_{t}(y_{t}-\delta v_{t})$. 

Our main goal is to obtain the adaptive regret with the prediction quality. We
retain the definition of prediction sensitivity from \eqref{eq:defst} as
$S_{T} = \sum_{t=1}^{T}\|\nabla f_{t}(x_t)- m_{t}\|^{2}$. Since the chosen point $x_{t}$
depends on the learner's internal randomization $v_{t}$, $S_{T}$ is a random
variable. Our bounds are thus  expressed in terms of its expected value
 $\bar S_{T}= \mathbb{E}[S_{T}]$.

To achieve this goal, we incorporate the optimistic prediction $m_{t}$ into our
algorithm. Based on insights of optimistic algorithms such as OMD, we maintain and update two sequences $\{y_{t}\}$ and $\{y'_{t}\}$. However, trivially incorporating $m_t$ into the update rule is insufficient since
the classic two-point feedback estimator
$\hat{g}_{t}=\frac{d}{2\delta}\left(f_{t}(y_{t}+ \delta v_{t}) - f_{t}(y_{t}- \delta
v_{t})\right)v_{t}$
suffers from high variance that scales with the gradient norm. Instead, we introduce a variance-reduced
estimator defined as $\hat{g}_{t}= m_{t}+ \frac{d}{2\delta}\Delta_{t}v_{t}$, where
$\Delta_{t}= f_{t}(y_{t}+ \delta v_{t}) - f_{t}(y_{t}- \delta v_{t}) - 2\delta \langle
m_{t}, v_{t}\rangle$. This estimator uses $m_{t}$ as a baseline to reduce
variance, ensuring that the estimation error scales with the prediction error rather
than the gradient norm. We present the detailed procedure of our proposed Two-Point
Variance-Reduced Optimistic Gradient Descent (\tpvr) algorithm in 
Algorithm \ref{alg:tp-vr-opt}.

\begin{algorithm}
    [t]
    \caption{\textsc{TP-VR-Opt}}
    \label{alg:tp-vr-opt}

    \textbf{Input:} Convex set $\cX$ with in-radius $r$, perturbation
    $\delta \in (0, r)$, step size $\eta > 0$. Let $\Pi_{\cX_{\alpha}}(\cdot)$ denote
    Euclidean projection onto $\cX_{\alpha}$. \\
    \textbf{Initialize:} $\alpha = \delta/r$, $\cX_{\alpha}= (1-\alpha)\cX$, and
    $y'_{1}\in \cX_{\alpha}$.

    \begin{algorithmic}
        [1] \FOR{$t=1,\ldots, T$} \STATE Choose prediction
        $m_{t}\in \mathbb{R}^{d}$ (with $\|m_{t}\| \le L$). \STATE Compute center
        point $y_{t}= \Pi_{\cX_\alpha}(y'_{t}- \eta m_{t})$. \STATE Sample
        $v_{t}\sim \mathrm{Unif}(\mathbb{S}^{d-1})$, play $x_{t}= y_{t}+ \delta v
        _{t}$, and observe $f_{t}(y_{t}\pm \delta v_{t})$. \STATE Construct the variance-reduced
        estimator $\hat{g}_{t}= m_{t}+ \frac{d}{2\delta}\Delta_{t}v_{t}$, where
        $\Delta_{t}= f_{t}(y_{t}+ \delta v_{t})-f_{t}(y_{t}- \delta v_{t}) - 2\delta
        \langle m_{t}, v_{t}\rangle$. \STATE Update $y'_{t+1}= \Pi_{\cX_\alpha}(y
        '_{t}- \eta \hat{g}_{t})$. \ENDFOR
    \end{algorithmic}
\end{algorithm}

The estimator in Line 5 should be viewed as estimating the residual around
$m_t$, rather than estimating the full gradient from scratch. To see this, denote
$g_t^y=\nabla f_t(y_t)$ and $a_t=g_t^y-m_t$. Define the symmetric-difference remainder as
\begin{equation*}
    r_t=
    f_t(y_t+\delta v_t)-f_t(y_t-\delta v_t)
    -2\delta\langle g_t^y,v_t\rangle.
\end{equation*}
Then the centered difference used in Algorithm~\ref{alg:tp-vr-opt} satisfies
\begin{equation*}
    \Delta_t
    =
    2\delta\langle a_t,v_t\rangle+r_t.
\end{equation*}
Plugging the identity $\Delta_t=2\delta\langle a_t,v_t\rangle+r_t$ into the
definition of $\widehat g_t$ gives
\begin{equation*}
    \widehat g_t-m_t
    =
    d\langle a_t,v_t\rangle v_t
    +
    \frac{d}{2\delta}r_t v_t .
\end{equation*}
This decomposition makes explicit that the random part of the estimator is
centered around the residual $a_t=g_t^y-m_t$, rather than around the full
gradient $g_t^y$. In particular, the leading term
$d\langle a_t,v_t\rangle v_t=dv_tv_t^\top a_t$ is a randomized reconstruction of
the residual direction, while the second term is the finite-difference bias
caused by using a nonzero smoothing radius $\delta$.

To analyze the error of the estimator as a gradient estimator, we subtract
$g_t^y=m_t+a_t$ from both sides. This yields
\begin{equation*}
    \widehat g_t-g_t^y
    =
    (d v_t v_t^\top-I)a_t
    +
    \frac{d}{2\delta}r_t v_t .
\end{equation*}
The first term is the sampling fluctuation induced by the random direction
$v_t$, and it is conditionally zero-mean because
$\mathbb E[v_t v_t^\top]=I/d$. The second term is the bias term controlled by
smoothness. Under
$\beta$-smoothness, $|r_t|\le \beta\delta^2$, which yields
\begin{equation*}
    \mathbb E\|\widehat g_t-m_t\|^2
    \le
    2d\|g_t^y-m_t\|^2
    +
    \frac{d^2}{2}\beta^2\delta^2 .
\end{equation*}
Thus, when the prediction is accurate, the estimator itself has small second
moment. This is the key distinction from the classical two-point estimator,
whose second moment scales with the full gradient norm. The role of $m_t$ is
therefore not only to shift the optimistic update, but also to reduce the
variance of the bandit gradient estimate.

We now present the main theoretical guarantee for \textsc{TP-VR-Opt}. We first establish
a general regret bound that depends on the step size $\eta$ and the perturbation
radius $\delta$, without tuning for these parameters.
\begin{theorem}
    \label{thm:exp-regret} Suppose the functions $f_{t}$ are convex, $L$-Lipschitz,
    and $\beta$-smooth on $\cX$. If \tpvr\, is run with parameters $\eta, \delta
    > 0$ and $\alpha = \delta/r$, then the expected regret is bounded as
    \begin{align*}
        \mathbb{E}[\regret_{T}]&\le\frac{D^{2}}{2\eta}+\eta\left( 4d\,\bar S_{T}+
        \left(\frac{d^{2}}{2}+4d\right)\beta^{2}\delta^{2}T \right)+\frac{Dd\beta\delta
        T}{2}+\left(L+\frac{LD}{r}\right)\delta T. \label{eq:main-exp-bound}
    \end{align*}
\end{theorem}

By optimally tuning the parameters $\eta$ and $\delta$ to minimize the bound shown in
Theorem~\ref{thm:exp-regret}, we obtain the following adaptive regret guarantee in terms of the expectation of the cumulative prediction error $\sqrt{\bar{S}_{T}}$.

\begin{theorem}
    \label{thm:two_point} Suppose the functions $f_{t}$ are convex, $L$-Lipschitz,
    and $\beta$-smooth on $\cX$. When \textsc{TP-VR-Opt} is run with the oracle-tuned
    parameters:
    \begin{equation*}
        \eta = \frac{D}{\sqrt{8d (\bar{S}_{T}+1)}}, \quad \delta = \min \left\{ \frac{\sqrt{\bar{S}_{T}+1}}{d\beta
        T}, \frac{1}{\beta\sqrt{T}}, \frac{r}{2}\right\},
    \end{equation*}
    the expected regret is bounded as $
        \mathbb{E}[\regret_{T}]=O\big( D \sqrt{d \, \bar S_{T}}\big)$.
\end{theorem}

The parameter choice in Theorem~\ref{thm:two_point} reflects two separate
tradeoffs. The stepsize $\eta$ balances the usual stability term
$D^2/(2\eta)$ with the residual-variance term $\eta d\bar S_T$. This is the
same stability--variance tradeoff as in optimistic online gradient methods, but
with the cumulative prediction error $\bar S_T$ replacing the horizon-dependent
quantity that would arise from a worst-case second-moment bound. The perturbation radius $\delta$ plays a different role. It controls the price
of using a finite difference to approximate local first-order information. In
the bound of Theorem~\ref{thm:exp-regret}, the $\delta$-dependent terms
come from the smoothness bias of the symmetric difference and from the
perturbation/shrinkage cost. The choice $\delta$ ensures that these bias terms are dominated by the leading
$D\sqrt{d\bar S_T}$ term, while also preserving feasibility of the perturbed
queries. Thus, after tuning, the regret is driven by the residual variance
rather than by the full gradient magnitude or the time horizon.

From Theorem~\ref{thm:two_point}, the regret upper bound is determined by the
predictive error $\bar S_{T}$ when the feedback is given by two points, overcoming
the barrier of $\Omega(\sqrt{T})$ in the single-point feedback setting. When the
optimistic predictions are good, i.e., when $\bar S_{T}$ is small (e.g., $\bar{S}_T=o(T)$), the performance of the algorithm improves based on the predictions.

\subsection{Connection to Gradually Evolving Worlds~\cite{chiang2013beating}}\label{sec:comp_chiang}

With the same Lipschitz continuity and smoothness assumption, Chiang et al.~\cite{chiang2013beating} obtained the regret bound $\widetilde{O}(d^2\sqrt{D_T})$ in terms of gradient variation $D_T$ under   two-point feedback, where $D_T=\sum_{t=1}^T \max_{x\in\cX}\|\nabla f_t(x)-\nabla f_{t-1}(x)\|^2$ characterizes the gradually evolving worlds. Interestingly, when we specialize our proposed \tpvr\, by restricting the sample distribution over $\mathbb{S}^{d-1}$ to the standard basis vectors, i.e, $v_t\sim \mathrm{Unif}(\{e_1,\dotsc,e_d\})$ and employing the coordinate-persistent predictor $m_t$ whose $i$-th coordinate stores the most
recent finite-difference observation on that coordinate, \tpvr\  recovers the results in~\cite{chiang2013beating}. In this subsection, we formalize this specialization and highlight how the general framework encompasses the gradually evolving worlds construction as a special case. The variant of \tpvr\, can be written shown as the following Algorithm~\ref{alg:variant_tpvr}. Our notation $\Delta_{t,i_t}$ in \tpvr\, is written as $\Delta_{t,i_t}=2\delta(\widehat{v}_{t,i_t}-m_{t,i_t})$ for convenience here.

\begin{algorithm}
    [t]
    \caption{Variant of \tpvr\, with coordinate sampling}
    \label{alg:variant_tpvr}

    \textbf{Input:} Convex set $\cX$ with in-radius $r$, perturbation
    $\delta \in (0, r)$, step size $\eta > 0$. Let $\Pi_{\alpha}(\cdot)$ denote
    Euclidean projection onto $\cX_{\alpha}$. \\
    \textbf{Initialize:} $\alpha = \delta/r$, $\cX_{\alpha}= (1-\alpha)\cX$, and
    $y'_{1}\in \cX_{\alpha}$.

    \begin{algorithmic}
        [1] \FOR{$t=1,\ldots, T$} 
        \STATE Given
        $m_{t}\in \mathbb{R}^{d}$. 
        \STATE Compute center
        point $y_{t}= \Pi_{\cX_\alpha}(y'_{t}- \eta m_{t})$. 
        \STATE Sample
        $i_t\sim \mathrm{Unif}([d])$, play $x_{t}= y_{t}+ \delta e
        _{i_t}$, and observe $f_{t}(y_{t}\pm \delta e
        _{i_t})$. 
        \STATE Construct the variance-reduced
        estimator $\widehat{g}_{t}= m_{t}+ d(\widehat{v}_{t,i_t}-m_{t,i_t})e_{i_t}$, where
        $\widehat{v}_{t,i_t}= \frac{1}{2\delta}( f_{t}(y_{t}+ \delta e
        _{i_t})-f_{t}(y_{t}- \delta e
        _{i_t}))$. \STATE Update $y'_{t+1}= \Pi_{\cX_\alpha}(y
        '_{t}- \eta \widehat{g}_{t})$. 
        \STATE Predictor update: $m_{t+1}=m_t+(\widehat{v}_{t,i_t}-m_{t,i_t})e_{i_t}$
        \ENDFOR
    \end{algorithmic}
\end{algorithm}

To recover the result shown in~\cite{chiang2013beating}, the key idea is that for each coordinate $i_t$, the finite-difference slope $\widehat{v}_{t,i_t}$ approximates the true gradient $\nabla_{i_t} f_t(y_t)$ up to a smoothness-dependent bias of order $\beta \delta$. 
By using the variance-reduced estimator $\widehat{g}_t$ and updating $m_t$ coordinate-wise, the accumulated estimation error over the trajectory can be controlled in terms of the gradient variation $D_T = \sum_{t=1}^T \max_{x\in \cX}\|\nabla f_t(x)-\nabla f_{t-1}(x)\|^2$. 
The full technical derivation is deferred to Appendix~\ref{app:comp_chiang}.

\begin{corollary}[Recovery of Gradually Evolving Worlds \cite{chiang2013beating}]
Suppose the functions $f_{t}$ are convex, $L$-Lipschitz, and $\beta$-smooth on $\cX$. When Algorithm~\ref{alg:variant_tpvr} is run with the oracle-tuned parameters:
\begin{equation*}
    \eta=\min\left\{\frac{D}{4d^2\sqrt{D_T+1}},\frac{1}{16\beta d^{3/2}\sqrt{\log T}}\right\},\qquad 
    \delta=\min\left\{\frac{\sqrt{D_T+1}}{\beta T},\frac{1}{\beta\sqrt{T}},\frac{r}{2}\right\},
\end{equation*}
we have
\begin{equation*}
    \mathbb{E}[\regret_{T}]\le\widetilde{O}(Dd^2 \sqrt{D_T}),
\end{equation*}
matching the regret bound in \cite{chiang2013beating}.\end{corollary}

Note that \cite{chiang2013beating} readily provides a scheme to generate hints in the BCO setup, this corollary demonstrates that our general \tpvr\ framework naturally encompasses the gradually evolving worlds setting as a special case. 
While the original analysis in \cite{chiang2013beating} explicitly considers the coordinate-persistent predictor and basis-vector sampling, our framework allows arbitrary direction sampling and gradient predictions. 

Importantly, the $d^2$ dependence arises from two sources: (i) the dimension $d$ itself due to coordinate-wise updates, and (ii) the second-moment scaling in the variance-reduced estimator when tracking the residual along each coordinate. In particular, for each sampled coordinate, the variance contribution is proportional to $d$, and summing across $d$ coordinates gives the $d^2$ factor. For sufficiently smooth losses and bounded gradient variation, the bias terms remain lower order and do not affect the leading-order regret. Recent work~\cite{yu2026improved} demonstrates that this bound can be further tightened by refining the analysis of non-consecutive gradient variations. 
In this setting, each coordinate is sampled intermittently, so the gradient difference is measured between the current and the most recent observation of the same coordinate. 
By carefully analyzing the expected maximal interval between consecutive samplings (using a coupon-collector argument) and stabilizing the step-size schedule, the dimension dependence of the regret is further reduced from $O(d^2 \sqrt{D_T})$ to $O(d^{3/2} \sqrt{D_T})$ for convex functions, and from $O(d^2 \log(d D_T)/\lambda)$ to $O(d \log(D_T)/\lambda)$ for $\lambda$-strongly convex functions.

Thus, the corollary highlights both the \textit{generality} of \tpvr\ and the \textit{compatibility} with previously studied special cases. Any alternative sampling scheme or improved gradient prediction immediately extends to gradually evolving worlds, providing a unified perspective on prediction-adaptive bandit convex optimization.

\subsection{Parameter Tuning with Unknown Prediction Error and Time Horizon}
\label{sec:tp-vr-adaptive}

The regret bound in Theorem \ref{thm:exp-regret} holds for any valid choice of $\eta$
and $\delta$. To achieve the desired adaptive regret bound, these parameters need
to be tuned according to the complexity of the problem instance, specifically
the expected prediction error $\bar S_{T}$ and time horizon $T$. By optimally tuning
the parameters $\eta$ and $\delta$ to minimize the bound in Theorem~\ref{thm:exp-regret},
we obtain the adaptive regret bound shown in Theorem~\ref{thm:two_point}.

While Theorem \ref{thm:two_point} derives a general adaptive bound in terms of $\bar{S}_T$, one could see that the optimal parameters for $\eta$ and $\delta$ cannot
be determined without knowledge of $\bar S_{T}$ and time horizon $T$. In this subsection, we present an algorithm \tpvrplus\, based on a nested doubling-trick
approach to adaptively tune the parameters $\eta$ and $\delta$ without prior
knowledge of $\bar S_{T}$ and $T$. Let $H$ denote the current time budget. We
initialize $H=1$ and double $H$ whenever the current time step $t$ exceeds $H$. Let
$I$ denote the set of rounds within the current time phase with $|I|\leq H$ and define
the phase prediction sensitivity as $\bar S(I)=\mathbb{E}\left[\sum_{t\in I}\|\nabla
f_{t}( x_{t})-m_{t}\|^{2}\right]$. We denote $S$ as the budget guess for the
phase prediction sensitivity, initialized to $S=1$ and doubled whenever $\sum_{t\in
I}\| \hat g_{t}-m_{t}\|^{2}>8dS$. Within each phase, we run \tpvr\, with
parameters tuned using $H$ and the local prediction sensitivity guess $S$. The nontrivial issue here is that the target
quantity $\bar S_T$ is trajectory-dependent and unobserved under bandit feedback.
We therefore use the observable residual $\sum_t\|\widehat g_t-m_t\|^2$ as a
proxy for the prediction error and apply a standard doubling scheme to this
observable quantity. We summarize
the complete procedure in Algorithm~\ref{alg:tp-vr-adaptive} and the regret guarantee of \tpvrplus\, is presented in the following Theorem.

\begin{algorithm}
    [t]
    \caption{\tpvrplus}
    \label{alg:tp-vr-adaptive}
    \begin{algorithmic}
        [1] \STATE Initialize time budget $H\leftarrow 1$ and $y'_1\in \cX$.
        \WHILE{$t\geq 1$} \STATE Start a phase with a time of length at most $H$ and initialize a counter $c$ to 0. 
        \STATE Initialize prediction-sensitivity budget $S\leftarrow
        S_{\min}=\max\{1,L^2\}$ and residual $R\leftarrow 0$. 
        \STATE Set parameters: $\eta \leftarrow
        \frac{D}{\sqrt{8d\,S}}$, $\delta \leftarrow \min\Big\{\frac{\sqrt{S}}{d\beta
        H}, \frac{1}{\beta\sqrt{H}},\ \frac{r}{2}\Big\}$, $\alpha\leftarrow \delta/
        r.$ 
        \STATE Project to the current shrunken set:
        $y'_t\leftarrow \Pi_{\cX_\alpha}(y'_t)$. \WHILE{$c < H$} \STATE Run one round
        of \tpvr\, with fixed $(\eta,\delta,\alpha)$. \STATE Update: $R\leftarrow\!
        R+\|\hat g_{t}-m_{t}\|^{2}$, $c\leftarrow c+1$, $t\leftarrow t+1$.
        \IF{$R > 8d\,S$}
        \STATE \textbf{Prediction-Sensitivity doubling:} $S\leftarrow 2S$, \ $R\leftarrow
        0$. \STATE Update parameters $(\eta, \delta,\alpha)$ according to Line~5. \STATE
        Project to new shrunken set: $y'_t\leftarrow \Pi_{\cX_\alpha}(y'_t)$. \ENDIF \ENDWHILE
        \STATE \textbf{Time doubling:} $H \leftarrow 2H$. \ENDWHILE
    \end{algorithmic}
\end{algorithm}

\begin{theorem}
    \label{thm:tp-vr-adaptive} Suppose the functions $f_{t}$ are convex, $L$-Lipschitz,
    and $\beta$-smooth on $\cX$. When \tpvrplus\, is run as Algorithm~\ref{alg:tp-vr-adaptive},
    the expected regret is bounded as
    \begin{equation*}
        \mathbb{E}[\regret_{T}]\le \widetilde{O}\left( D\sqrt{d\,\bar S_{T}}\right).
    \end{equation*}
\end{theorem}

Theorem \ref{thm:tp-vr-adaptive} establishes that \tpvrplus\, achieves a regret
bound comparable to the oracle bound in Theorem~\ref{thm:two_point}, up to logarithmic
factors. This result confirms that the benefits of prediction can be
realized in an adaptive manner, without requiring prior knowledge of the cumulative
prediction error $\bar S_{T}$ or the time horizon $T$.

\subsection{Lower Bound in Terms of Prediction Error}
\label{sec:tp-vr-minimax}

In previous subsections, we have first established that the \tpvr\, algorithm
achieves an adaptive regret bound that scales with the expected cumulative prediction error
$\bar S_{T}$. Then, we also presented a variant \tpvrplus\, which does not require
prior knowledge of $\bar S_{T}$ and $T$ while still achieving a similar adaptive
regret rate up to logarithmic factors. In this subsection, we establish an information-theoretic lower bound showing
that the $\sqrt{\bar S_T}$-dependence in the above upper bounds is unavoidable. Information-theoretic lower bounds are a standard tool for characterizing
the oracle complexity of convex optimization~\cite{agarwal2012information};
our lower bound is different in that the hardness is parameterized by the
irreducible prediction error rather than by the number of oracle queries. Specifically, we show
that no algorithm can guarantee regret $o(D\sqrt{\bar S_{T}})$ even
with full-information feedback. Equivalently, for any target prediction error
level $S \in [0, L^{2}T]$, there exist problem instances whose expected cumulative prediction errors
satisfy $\bar S_{T}= S$ while the expected regret is at least $\Omega(D\sqrt{S})$.
As a consequence, the dependence on $O(D\sqrt{d\bar S_{T}})$ in Theorem~\ref{thm:two_point}
is optimal up to a $\sqrt{d}$ factor.

We begin our analysis by establishing an information-theoretic lower bound via a
simple linear function construction. We show that a regret scaling as   $\Omega(D\sqrt{S})$
is unavoidable for any feedback model (whether single-point, two-point feedback,
or even full-information). This result highlights the inherent difficulty of the
problem with the cumulative prediction errors $S$. For a distribution $\mathcal D$ over loss sequences, a learning algorithm
$A$, and a prediction policy $\pi\in\Pi$, define $\bar S_T^\pi(\mathcal D,A)=
\mathbb E_{\mathcal D,A,\pi}
\sum_{t=1}^T\|\nabla f_t(x_t^\pi)-m_t^\pi\|^2$. We further define the intrinsic prediction error as $\bar S_T^\star(\mathcal D,A)=\inf_{\pi\in\Pi}\bar S_T^\pi(\mathcal D,A)$.

\begin{theorem}
    \label{thm:lb-linear} 
    Let $\mathcal X\subset\mathbb R^d$ be a convex set with Euclidean diameter $D$.
Fix any horizon $T\ge1$ and any learning algorithm $A$ under any feedback model
(one-point, two-point, or full-information). Let $\Pi$ denote the class of all
admissible non-anticipating prediction policies, i.e., for every $\pi =   \{m_t^\pi\}_{t=1}^T\in\Pi$,
$m_t^\pi$ is $\mathcal H_{t-1}$-measurable and $\|m_t^\pi\|\le L$ for all $t\in [T]$. Then, for every $S\in[0,L^2T]$, there exists a distribution $\mathcal D_S$ over
sequences of convex, $L$-Lipschitz, $0$-smooth linear losses
$\{f_t\}_{t=1}^T$ such that $\bar S_T^\star(\mathcal D_S,A)=S$, and for every admissible prediction policy $\pi\in\Pi$,
\begin{equation}
\mathbb E_{\mathcal D_S,A,\pi}[\regret_T]
    \ge
    cD\sqrt S, \label{eqn:lower_bd_S}
\end{equation}
where $c>0$ is a universal constant.
    
\end{theorem}

Then, we specialize our analysis to the two-point feedback setting. Comparing Theorem~\ref{thm:lb-linear} with Theorem~\ref{thm:two_point}, we see
that the $\sqrt{\bar S_T}$ dependence is unavoidable (due to the lower bound $\Omega(D\sqrt{S})$ in~\eqref{eqn:lower_bd_S}). Even if the
prediction error is interpreted in the strongest possible sense as the
irreducible prediction error $\bar S_T^\star$, no algorithm can improve the
dependence below $D\sqrt{\bar S_T^\star}$ in general. The only remaining gap is
the $\sqrt d$ factor. Theorem~\ref{thm:lb-linear} does not rule out removing this
factor; it only shows that there is no additional lower bound depending solely
on $T$ in the two-point/full-information comparison.

\begin{remark}[The $\sqrt{d}$ Gap]
Theorem~\ref{thm:lb-linear} establishes a universal lower bound on the
prediction-error dependence. Even with full-information feedback, no algorithm
can improve the $\sqrt S$ scaling in general. The factor $\sqrt d$
in Theorem~\ref{thm:two_point} arises from the second-moment analysis of the
two-point residual estimator, whose variance scales as
$d\|\nabla f_t(x_t)-m_t\|^2$. Whether this factor is information-theoretically
unavoidable for the full class of two-point algorithms is a separate question.
\end{remark}
    \section{Dynamic Regret for \tpvr}
\label{sec:implications}

In this section, we extend our proposed \tpvr\, algorithm to the \emph{dynamic
regret} setting. Throughout this section, except considering the non-stationarity
of the comparator sequence, we keep the same two-point feedback framework as previously
discussed within Section \ref{sec:setup} and Section \ref{sec:twopoint_vr}.

\textbf{Dynamic Regret.}
The performance in the non-stationary environment is measured with respect to a fixed sequence of comparators $u_{1:T}=
(u_{1},\ldots,u_{T}) \in  \cX^T$ leading to the  the dynamic regret being defined as 
\begin{equation*}
\regret^{\mathrm{dyn}}_{T}(u_{1:T})=\sum_{t=1}^{T}\bigl(f_{t}(x_{t})-f_{t}(u_{t})\bigr).
\end{equation*}
In order to estimate the level of non-stationarity exhibited by the comparator
sequence, the {\em path length} of the comparator sequence is defined as $P_{T}(u_{1:T}) = \sum_{t=2}^{T}\|u_{t}-u_{t-1}\|.$

We now present our main result for the dynamic regret in a non-stationary environment. The following theorem
establishes a dynamic regret bound that depends on both the path-length $P_{T}$, which
measures the fluctuation of the comparator sequence, and the expectation of cumulative prediction error $\bar
{S}_{T}$.

\begin{theorem}
    \label{thm:dynamic_regret} Assume that each loss function $f_{t}$ is convex,
    $L$-Lipschitz, and $\beta$-smooth on $\cX$. Run \tpvr\, with fixed parameters
    $\eta>0$ and $\delta\in(0,r)$. Then, for any comparator sequence $u_{1:T}\in
    \cX^T$, the expected dynamic regret is bounded by
    \begin{align*}
        &\mathbb{E}\bigl[\regret^{\mathrm{dyn}}_{T}(u_{1:T})\bigr]\le  \frac{D^{2}+ 2D P_{T}}{2\eta}+\eta \Bigl( 4d\,\bar{S}_{T}+ (d^{2}/2 + 4d)\beta^{2}\delta^{2}T \Bigr)  +\frac{D d \beta \delta T}{2}+ \Bigl( L + \frac{LD}{r}\Bigr) \delta T,
    \end{align*}
    where $P_{T}= P_{T}(u_{1:T})$ is the path-length. In particular, if
    $\bar{S}_{T}$ and $P_{T}$ are known, choosing
    \begin{equation*}
        \eta = \sqrt{\frac{D^{2}+ 2D P_{T}}{8d(\bar{S}_{T}+1)}},\quad
        \delta = \min\left\{ \frac{\sqrt{\bar{S}_{T}+1}}{d \beta T}, \; \frac{1}{\beta
        \sqrt{T}}, \; \frac{r}{2}\right\}
    \end{equation*}
    yields the regret bound as
    \begin{equation*}
        \mathbb{E}\bigl[\regret^{\mathrm{dyn}}_{T}(u_{1:T})\bigr] =O\Bigl( \sqrt{d\,\bar{S}_{T}(D^{2}+
        D P_{T})}\Bigr).
    \end{equation*}
\end{theorem}

Theorem \ref{thm:dynamic_regret} demonstrates that \tpvr\, achieves a dynamic
regret that scales with the path-length $P_{T}$ and the prediction error
$\bar{S}_{T}$ as $O\big(\sqrt{ d\bar{S}_T  (1+P_T)}\big)$. However, similar to the static regret result shown in Theorem~\ref{thm:two_point},
the dynamic regret bound in Theorem~\ref{thm:dynamic_regret} relies on prior
knowledge of the path-length $P_{T}$ to tune the step size $\eta$.

To achieve the optimal dynamic regret rate without prior knowledge of the path-length
$P_{T}$, we introduce the \tpvrplusplus\, algorithm. This algorithm builds on the
adaptive framework of \tpvrplus\, (Algorithm~\ref{alg:tp-vr-adaptive}) by
inheriting its nested phase-epoch structure along with the doubling schedule for
the sensitivity budget $S$ and time horizon $H$. However, to effectively handle
unknown non-stationarity, \tpvrplusplus\, replaces the fixed-step update within
each epoch with a \emph{Meta-Expert} subroutine. The learning model maintains a
geometric grid of expert examples and combines their predictions using an
optimistic Hedge meta-algorithm. The oracle choice of $\eta$ in Theorem~\ref{thm:dynamic_regret}
depends on the unknown path-length $P_T$. A single fixed stepsize cannot be
simultaneously optimal for all levels of comparator variation: small $P_T$
requires a more aggressive stepsize, while large $P_T$ requires a more
conservative one. TP-VR-OPT++ addresses this issue by maintaining a geometric
grid of candidate stepsizes and aggregating the corresponding experts.

Importantly, the meta layer is used only to adapt to the unknown dynamic scale.
Each expert still uses the same prediction-centered two-point estimator, and
therefore retains the residual-variance control in terms of $\bar S_T$. The
optimistic meta-update uses the predicted linear losses induced by $m_t$ before
observing feedback and then updates the weights using the estimated losses
induced by $\widehat g_t$. This preserves the prediction-adaptive nature of the
bound that the price of aggregation is logarithmic, rather than introducing a new
$\sqrt{T}$-type term. Therefore, the learning model automatically adapts
to the optimal step size $\eta^{\star}$ with regard to local path-length using this
technique. The specifics of the meta-expert approach are presented in Algorithm~\ref{alg:tp-vr-opt-pp}.

We note that, in contrast to the parameter-free Bandit Gradient Descent algorithm proposed in \cite{zhao2021bandit}, our approach employs an optimistic meta-algorithm, which ensures that the meta-regret depends on $\bar S_T$ rather than relying on the standard Hedge-based aggregation used in~\cite{zhao2021bandit}. Moreover, each expert in our framework is instantiated with our proposed \tpvrplus\, algorithm, rather than the classical Bandit Gradient Descent procedure.

\begin{algorithm}
    [t]
    \caption{\tpvrplusplus}
    \label{alg:tp-vr-opt-pp}
    \begin{algorithmic}
        [1] \STATE Run the outer double-loop structure (Phases \& Epochs) from \textbf{Algorithm~\ref{alg:tp-vr-adaptive}}.

        \STATE \textbf{Initialize Meta-Learner (at start of each Epoch):} 
        \STATE Input: Current time budget $H$, Sensitivity budget $S$, Convex set $\mathcal{X}_\alpha$, current point $y'_t$
        \STATE Set parameters $\delta, \alpha$ as in Algorithm~\ref{alg:tp-vr-adaptive}.
        \STATE Set base step size $\eta_{0}\leftarrow \frac{D}{\sqrt{16dS}}$.
        \STATE Set experts number
        $N \leftarrow \max\{2, \lceil \log_{2}(DH/\eta_{0}) \rceil + 1\}$. 
        \STATE Set meta step size $\varepsilon \leftarrow \frac{\sqrt{\log N}}{D\sqrt{8dS}}$. 
        \STATE \textbf{Experts:} For $i=1,\dots,N$:
        $\eta_{i}\leftarrow 2^{i-1}\eta_{0}$, $y'_{t,i}\leftarrow \Pi_{\mathcal{X}_\alpha}
        (y'_t)$. 
        \STATE\textbf{Weights:} Initialize $w_{i}\leftarrow \frac{N+1}{N}\cdot\frac{1}{i(i+1)}$.

        \STATE \textbf{Run Round $t$ (Meta-Aggregation):} \STATE \quad Choose
        prediction $m_{t}$. \STATE \quad \textbf{Expert Predictions:} For each
        $i$, $y_{t,i}\leftarrow \Pi_{\mathcal{X}_\alpha}(y'_{t,i}- \eta_{i}m_{t})$
        and $\hat{\ell}_{t,i}\leftarrow \langle m_{t}, y_{t,i}\rangle$. 
        \STATE
        \quad \textbf{Aggregation:} Compute $p_{t,i}\propto w_{i}\exp(-\varepsilon
        \hat{\ell}_{t,i})$ and play $y_{t}\leftarrow \sum_{i=1}^N p_{t,i}y_{t,i}$.
        \STATE \quad Sample $v_{t}$, observe feedback $f_{t}( y_{t}\pm \delta v_{t}
        )$, construct~$\hat{g}_{t}$. \STATE \quad \textbf{Expert Updates:} For
        each $i$, update $y'_{t,i}\leftarrow \Pi_{\mathcal{X}_\alpha}(y'_{t,i}- \eta_{i}
        \hat{g}_{t})$. \STATE \quad \textbf{Weight Update:} Receive cost $\ell_{t,i}
        \leftarrow \langle \hat{g}_{t}, y_{t,i}\rangle$; update
        $w_{i}\leftarrow w_{i}\exp(-\varepsilon \ell_{t,i})$. \STATE \quad \textbf{Output:}
        Return accumulated residual $\|\hat{g}_{t}- m_{t}\|^{2}$ to the outer
        loop. \STATE \quad Update $y' \leftarrow y_{t}$.
    \end{algorithmic}
\end{algorithm}

\begin{theorem}
    \label{thm:tp-vr-opt-pp} Assume that the loss functions $f_{t}$ are convex,
    $L$-Lipschitz, and $\beta$-smooth on the convex set $\cX$ with diameter $D$. Running \tpvrplusplus\,
    guarantees that for any comparator sequence $u_{1:T}\in \cX^T$ of path-length
    $P_{T} = \sum_{t=2}^{T} \|u_{t} - u_{t-1}\|$, the expected dynamic regret is
    bounded as:
    \begin{equation*}
        \mathbb{E}\bigl[\regret_{T}^{\mathrm{dyn}}(u_{1:T})\bigr] \le\widetilde{O}
        \left( \sqrt{d \bar{S}_{T} (D^{2} + D P_{T})}\right).
    \end{equation*}
\end{theorem}

Theorem \ref{thm:tp-vr-opt-pp} highlights the adaptivity of \tpvrplusplus. The algorithm provides a dynamic regret bound of the form
$\widetilde{O}(\sqrt{d \bar{S}_{T} (1+P_{T})})$ without requiring prior knowledge of
the path-length $P_{T}$. The result matches the optimally adjusted $P_T$ of Theorem~\ref{thm:dynamic_regret} for dynamic regret to within logarithmic terms, thereby resolving the problem caused by the non-stationarity.

    \section{Conclusion and Future Work}
    Our work studies the bandit convex optimization problem with gradient prediction adaptivity, moving beyond worst-case regret guarantees to achieve improved performance. By establishing the fundamental limitations of single-point feedback, we design \tpvr\, to overcome these barriers. We further extend our analysis to non-stationary environments via dynamic regret and develop fully adaptive variants that require less prior knowledge.

    There are several avenues for future research. First, our lower bound in Theorem~\ref{thm:single_point_barrier} shows that under single-point feedback, the regret is fundamentally lower bounded by  $\Omega( \sqrt{T})$ even when the cumulative prediction error is negligible. However, this barrier does not preclude a \emph{partial} benefit from predictions. Since the worst-case regret for smooth convex losses under single-point feedback is $O(T^{2/3})$, which already improves upon the $O(T^{3/4})$ rate for Lipschitz losses, a natural question is whether predictions can further interpolate between this rate and the $\Omega(\sqrt{T})$ floor. That is, whether there exist algorithms achieving $\widetilde{O}(\min\{\sqrt{T},\, \bar{S}_T^a T^b\})$ with $a + b = 2/3$. The main difficulty is that the competing terms $\delta T + T/\delta^2$ in the single-point bias--variance tradeoff cannot be reduced by the prediction $m_t$, since the absolute function value observation mixes the local first-order signal with zero-order offsets. Another promising direction is to extend this work to the strongly convex setting.

    \clearpage

    {\appendices
    \section{Proposition~\ref{prop:linear-remainder-no-smooth}}\label{app:smooth}
\begin{proposition}
\label{prop:linear-remainder-no-smooth}
Let $L>0$ and $\delta\in(0,1)$. There exists a convex,
$L$-Lipschitz function $f:[-1,1]\to\mathbb{R}$ such that, for $y=0$ and direction $v=1$,
the symmetric difference remainder   defined as
\begin{equation*}
r(v)= f(y+\delta v)-f(y-\delta v)-2\delta \langle \nabla f(y),v\rangle
\end{equation*}
satisfies
\begin{equation*}
|r(v)|=\frac{2}{3}L\delta.
\end{equation*}
In other words, the normalized difference between the secant slope and the point gradient is constant:
\begin{equation*}
\left|\frac{f(\delta)-f(-\delta)}{2\delta}-f'(0)\right|=\frac{L}{3}.
\end{equation*}
\end{proposition}

\begin{proof}
Let $g:[-1,1]\to\mathbb{R}$ with
\begin{equation*}
g(x)=
\begin{cases}
-L, & x\in[-1,-\delta],\\[1mm]
\frac{L}{3}+\frac{4L}{3}\frac{x}{\delta}, & x\in[-\delta,0],\\[2mm]
\frac{L}{3}, & x\in[0,1].
\end{cases}
\end{equation*}
Then $g$ is continuous, nondecreasing, and satisfies that for all $x$ it holds that $|g(x)|\le L$.
Let
\begin{equation*}
f(x)=\int_{-\delta}^{x} g(s)\,ds.
\end{equation*}
Since $f\in C^1([-1,1])$ and $g(x)=f'(x)$ is nondecreasing, $f$ is convex. Additionally, since $|f'(x)|\le L$, the function $f$ is $L$-Lipschitz.

We have $f(-\delta)=0$ and $f'(0)=g(0)=L/3$. Also,
\begin{equation*}
f(\delta)-f(-\delta)=\int_{-\delta}^{\delta} g(s)\,ds
=\int_{-\delta}^{0}\Big(\frac{L}{3}+\frac{4L}{3}\frac{s}{\delta}\Big)\,ds
+\int_{0}^{\delta}\frac{L}{3}\,ds.
\end{equation*}
It follows that $f(\delta)=f(-\delta)$ which implies that
\begin{equation*}
r(1)=f(\delta)-f(-\delta)-2\delta f'(0)=0-2\delta\cdot \frac{L}{3}=-\frac{2}{3}L\delta.
\end{equation*}
Thus, we have $|r(1)|=\frac{2}{3}L\delta$ and $\Big|\frac{f(\delta)-f(-\delta)}{2\delta}-f'(0)\Big|=\frac{L}{3}$.
\end{proof}

    \section{Omitted Details of Section~\ref{sec:single_point}}\label{app:single_point}

\subsection{Useful definitions and notations}


\begin{definition}[Total Variation Distance]
\label{def:tv}
Given two probability distributions $P$ and $Q$ on the measurable space $ (\Omega,\mathcal{F})$, the total variation distance (denoted by $\mathrm{TV}(P,Q)$) can be computed as follows:
\begin{equation*}
    \mathrm{TV}(P, Q) = \sup_{A \in \mathcal{F}} |P(A) - Q(A)|.
\end{equation*}
In the case where the distributions $P$ and $Q$ can be described by their density functions $p$ and $q$ with respect to a base measure $\mu$, then we can express $\mathrm{TV}(P,Q)$ as
\begin{equation*}
    \mathrm{TV}(P, Q) = \frac{1}{2} \int |p(x) - q(x)| \, d\mu(x).
\end{equation*}
\end{definition}

\begin{definition}[Kullback-Leibler Divergence]
\label{def:kl}
Given two probability distributions $P$ and $Q$, the Kullback-Leibler (KL) divergence between these distributions can be calculated as follows:
\begin{equation*}
    \mathrm{KL}(P \| Q) = \int \log\left(\frac{dP}{dQ}\right) dP.
\end{equation*}
\end{definition}

Pinsker's Inequality is used to provide a standard upper bound for the total variation distance in terms of the Kullback-Leibler Divergence.

\begin{lemma}[Pinsker's Inequality]
\label{lem:pinsker}
Consider two probability distributions $P$ and $Q$, the following inequality holds:
\begin{equation*}
    \mathrm{TV}(P, Q) \le \sqrt{\frac{1}{2} \mathrm{KL}(P \| Q)}.
\end{equation*}
\end{lemma}

In the following, we will provide a corollary about the Kullback-Leibler divergence between two Gaussian distributions. This will be useful for proving Theorem~\ref{thm:single_point_barrier}.

\begin{corollary}[KL Divergence for Gaussians]
\label{fact:kl-gaussian}
Suppose that the two Gaussian distributions $ \mathcal{N}(\mu_1,\Sigma)$ and $ \mathcal{N}(\mu_2,\Sigma)$ have the same covariance matrix (i.e., $ \Sigma$). Then, we have:
\begin{equation*}
    \mathrm{KL}(\mathcal{N}(\mu_1, \Sigma) \| \mathcal{N}(\mu_2, \Sigma)) = \frac{1}{2} (\mu_1 - \mu_2)^\top \Sigma^{-1} (\mu_1 - \mu_2).
\end{equation*}
In the case of univariate Gaussians with common variance $\sigma^2$, this reduces to $\frac{(\mu_1 - \mu_2)^2}{2\sigma^2}$.
\end{corollary}

\subsection{Proof of Theorem~\ref{thm:single_point_barrier}}

Let $x_{+}$ and $x_{-}$ be two points satisfying $\|x_{+}-x_{-}\|=D$ in $\cX$, where the Euclidean diameter of $\cX$ is $D$. Then define the direction vector $v=(x_{+}-x_{-})/D$ and the midpoint $x_0=\frac{x_{+}+x_{-}}{2}$. Thus, for all $x\in \cX$, we see that 
\[
|\inner{v}{x-x_0}| \leq \frac{D}{2}.
\]

Let $M$ be a uniform Rademacher variable taking values in $\{+1,-1\}$. Assume that the Gaussian noise sequence $\{\xi_t\}_{t=1}^T$ are i.i.d (independent and identically distributed) Gaussian $\mathcal{N}(0,\sigma^2)$. Define $\varepsilon=\frac{\sigma}{2D\sqrt{T}}$. The assumption $\sigma \leq 2LD\sqrt{T}$ implies $\varepsilon \leq L$. Fix $t\in[T]$, then conditioned on the event that we observe $M=m$, we can consider the loss function as
\[
f_t^{(m)}(x)=-m\,\varepsilon\langle v,x-x_0\rangle+\xi_t .
\]
The loss function $f_t$ has $L$-Lipschitz continuity since $\|\nabla f_t^{(m)}(x)\|=\varepsilon \le L$ and is $0$-smooth (hence satisfies $\beta$-smoothness properties) since $\nabla^{2}f_t^{(m)}(x)=0$.

Let $P_+$ and $P_-$ be the distributions over the sequence of observations in the worlds $M=+1$ and $M=-1$, respectively.
Conditioned on the history $\mathcal{H}_{t-1}$, the observation $Y_t = f_t^{(M)}(x_t)$ is distributed as:
\begin{equation*}
Y_t \mid \mathcal{H}_{t-1} \sim \mathcal{N}(-m\varepsilon z_t, \sigma^2), \quad \text{where } z_t= \langle v, x_t - x_0 \rangle.
\end{equation*}

Utilizing Corollary~\ref{fact:kl-gaussian}, the conditional term at step $t$ is computed as follows:
\begin{align*}
\mathrm{KL}( \mathcal{N}(-\varepsilon z_t,\sigma^2) \| \mathcal{N}(+\varepsilon z_t,\sigma^2))  =
\frac{((- \varepsilon z_t ) - (+ \varepsilon z_t ))^2}{2 \sigma^2} =
\frac{(-2 \varepsilon z_t )^2}{2 \sigma^2} =
\frac{2 \varepsilon^2 z_t^2}{\sigma^2}.
\end{align*}
Substituting this Gaussian KL divergence back into the chain rule decomposition gives
\begin{equation*}
\mathrm{KL}(P_+ \| P_-) =\sum_{t=1}^T \mathbb{E}_{P_+}\left[ \mathrm{KL}(P_+(Y_t \ | \ \mathcal{H}_{t-1})\| P_-(Y_t \ | \ \mathcal{H}_{t-1})) \right] = \sum_{t=1}^T \mathbb{E}_{P_+}\left[ \frac{2 \varepsilon^2 z_t^2}{\sigma^2} \right] .
\end{equation*}
Then, we rewrite this expression as follows:
\begin{equation*}
\mathrm{KL}(P_+ \| P_-) =\mathbb{E}_{P_+}\left[ \sum_{t=1}^T \frac{2 \varepsilon^2 z_t^2}{\sigma^2} \right].
\end{equation*}
From the definition of $|z_t| \leq \frac{D}{2}$,
\begin{equation*}
\mathrm{KL}(P_+ \| P_-) \leq \frac{2 \varepsilon^2}{\sigma^2} \cdot \frac{D^2}{4} \cdot T=\frac{\varepsilon^2 D^2 T}{2\sigma^2}.
\end{equation*}
Substituting back in $ \varepsilon = \frac{\sigma}{2D\sqrt{T}} $, we get $ \mathrm{KL}(P_+ \| P_-) \leq \frac{1}{8}$, thus by Pinsker's inequality shown in Lemma~\ref{lem:pinsker} then
\begin{equation}\label{eq:pinsker}
\mathrm{TV}(P_+,P_-)\leq \sqrt{ \frac{1}{2} \mathrm{KL}(P_+ \|P_-)}\leq \sqrt{ \frac{1}{16}} =\frac{1}{4} .
\end{equation}

Static regret is defined by measuring against the best possible decision made in hindsight, i.e., the decision that minimizes cumulative losses,
\begin{equation*}
x^{(m)} \in \text{argmin}_{x\in \mathcal{X}} \sum_{t=1}^{T} f_t^{(m)}(x) ,
\end{equation*}
and since
\begin{equation*}
\sum_{t=1}^T f_t^{(m)}(x) = \sum_{t=1}^T \Big[ -m\varepsilon \langle v,x-x_0 \rangle + \xi_t \Big] = -m\varepsilon T\langle v,x-x_0 \rangle + \sum_{t=1}^T \xi_t,
\end{equation*}
the term \( \sum_{t=1}^T \xi_t \) is unaffected by changing \( x \); hence the minimization of cumulative losses through the term \( \sum_{t=1}^T f_t^{(m)}(x) \) results in maximizing the projection in the direction of \( mv \), therefore the points \( x_m^* = x_0 +m \cdot \frac{D}{2} v \) represent a maximum.

Returning to our previous definition of regret,
\begin{align*}
\regret_T^{(m)} &= \sum_{t=1}^T \Big[ -m\varepsilon z_t + \xi_t \Big] - \left(-m\varepsilon T\langle v,x_m^*-x_0\rangle + \sum_{t=1}^T\xi_t\right) \\
&= -m\varepsilon \sum_{t=1}^T z_t + m\varepsilon T \langle v ,x_m^* - x_0\rangle \\
&= -m\varepsilon Z + m\varepsilon T\cdot \Big(m\frac{D}{2}\Big) \\&= \frac{\varepsilon DT}{2} - m\varepsilon Z ,
\end{align*}
where $ Z = \sum_{t=1}^{T}z_t , z_t = \langle v,x_t - x_0\rangle$. Also note since $|z_t| \leq \frac{D}{2} , |Z| \leq \frac{DT}{2}$.

Thus, we have found that the average regret over $M$ is given by the following equation for the average regret:
\begin{equation*}
\mathbb{E}[\regret_T]
= \frac{1}{2} \mathbb{E}_{P_+}\left[\regret_T^{(+)}\right] + \frac{1}{2} \mathbb{E}_{P_-}\left[\regret_T^{(-)}\right]
= \frac{\varepsilon D T}{2} - \varepsilon \underbrace{\frac{1}{2}\left( \mathbb{E}_{P_+}[Z] - \mathbb{E}_{P_-}[Z]\right)}_{\mathbb{E}[M Z]}.
\end{equation*}

Let $Z$ be a random variable with values in $[-K,K]$, then for $|Z|\leq K$, $|\mathbb{E}_{P_+}[Z]-\mathbb{E}_{P_-}\left[Z\right]|\leq 2K\cdot\mathrm{TV}(P_+,P_-)$. Since $Z$ has a maximum absolute value of $\frac{DT}{2}$, we put $K=\frac{DT}{2}$. By applying the inequality $\mathrm{TV}(P_+,P_-) \leq \frac{1}{4}$ from \eqref{eq:pinsker} gives:
\begin{equation*}
\left|\mathbb{E}_{P_+}[Z]-\mathbb{E}_{P_-}\left[Z\right]\right| \leq 2\cdot\left(\frac{DT}{2}\right)\cdot \frac{1}{4} = \frac{DT}{4}.
\end{equation*}
Plugging this result into our previous equation, we find,
\begin{equation*}
\mathbb{E}[\regret_T] \ge \frac{\varepsilon D T}{2} - \varepsilon\left(\frac{1}{2}\cdot\frac{DT}{4}\right) = \frac{3}{8}\varepsilon D T.
\end{equation*}
Substituting $\varepsilon = \frac{\sigma}{2D\sqrt{T}}$ results in:
\begin{equation*}
\mathbb{E}[\regret_T] \geq \frac{3}{8}D T\cdot\left(\frac{\sigma}{2D\sqrt{T}}\right) = \frac{3}{16}\sigma\sqrt{T}.
\end{equation*}
    \section{Omitted Details of Section~\ref{sec:twopoint_vr}}
\label{app:two_point}

\subsection{Useful Lemmas in Section~\ref{sec:tp-vr-alg}}

\begin{lemma}[Difference Approximation Error]
    \label{lem:smoothness-remainder} Assume that $f_{t}$ is $\beta$-smooth. For any point $y_{t}$,
    direction $v_{t}\in \mathbb{S}^{d-1}$ and perturbation radius $\delta > 0$, the remainder in the difference approximation between points is defined as:
    \begin{equation*}
        r_{t}= f_{t}(y_{t}+ \delta v_{t}) - f_{t}(y_{t}- \delta v_{t}) - 2\delta
        \langle \nabla f_{t}(y_{t}), v_{t}\rangle.
    \end{equation*}
    Then, the magnitude of this remainder is bounded by:
    \begin{equation*}
        |r_{t}| \le \beta \delta^{2}.
    \end{equation*}
\end{lemma}

\begin{proof}
    Let $\phi(s)=f_t(y_t+sv_t)$. The difference in values of the functions along the direction of the unit vector $v_{t}$ expressed as an integral of the directional derivative is:
    \begin{equation*}
        f_{t}(y_{t}+ \delta v_{t}) - f_{t}(y_{t}- \delta v_{t}) = \phi(\delta) -
        \phi(-\delta) = \int_{-\delta}^{\delta}\phi'(s) ds = \int_{-\delta}^{\delta}
        \langle \nabla f_{t}(y_{t}+ s v_{t}), v_{t}\rangle ds.
    \end{equation*}
    The $2\delta \langle \nabla f_{t}(y_{t}), v_{t}\rangle$ term can similarly be
    expressed as:
    \begin{equation*}
        2\delta \langle \nabla f_{t}(y_{t}), v_{t}\rangle = \int_{-\delta}^{\delta}
        \langle \nabla f_{t}(y_{t}), v_{t}\rangle ds.
    \end{equation*}
    The above equations give us:
    \begin{equation*}
        r_{t}= \int_{-\delta}^{\delta}\left\langle \nabla f_{t}(y_{t}+ s v_{t}) -
        \nabla f_{t}(y_{t}), v_{t}\right\rangle ds.
    \end{equation*}
    Applying the Cauchy-Schwarz inequality gives:
    \begin{align*}
        |r_{t}| & \le \int_{-\delta}^{\delta}\left| \left\langle \nabla f_{t}(y_{t}+ s v_{t}) - \nabla f_{t}(y_{t}), v_{t}\right\rangle \right| ds \\
                & \le \int_{-\delta}^{\delta}\left\| \nabla f_{t}(y_{t}+ s v_{t}) - \nabla f_{t}(y_{t}) \right\| \|v_{t}\| ds.
    \end{align*}
    Because $\|\nabla f_{t}(y_{t}+ sv_{t}) - \nabla f_{t}(y_{t})\| \le \beta \| y_{t}+ sv_{t} - y_{t}\| = \beta|s|$ for the $\beta$-smoothness of $f_{t}$, and $\|v_{t}\|=1$, we have:
    \begin{equation*}
        |r_{t}| \le \int_{-\delta}^{\delta}\beta |s| ds = 2\beta \int_{0}^{\delta}
        s ds = \beta \delta^{2}.
    \end{equation*}
    This proves the lemma.
\end{proof}

\begin{lemma}[Gradient Difference Between Query and Center Points]
    \label{lem:gradient-stability} Denote $g_t^x=\nabla f_t(x_t)$ and $g_t^y=\nabla f_t(y_t)$. Assume $f_{t}$ is $\beta$-smooth on $\cX$.
    Then the difference between gradients sampled at the query point $x_{t}=y_{t}+\delta v_{t}$ and the center point $y_{t}$ is proportionally smaller than $\beta\delta$. In particular, we have bounds:
    \begin{align*}
        \|g_{t}^{y}- m_{t}\|^{2} & \le 2\|g_{t}^{x}- m_{t}\|^{2}+ 2\beta^{2}\delta^{2}, \\
        \|g_{t}^{x}- m_{t}\|^{2} & \le 2\|g_{t}^{y}- m_{t}\|^{2}+ 2\beta^{2}\delta^{2}.
    \end{align*}
\end{lemma}

\begin{proof}
    First, since $f_{t}$ is $\beta$-smooth and
    $\|x_{t}- y_{t}\| = \|\delta v_{t}\| = \delta$, we have:
    \begin{equation*}
        \|g_{t}^{x}- g_{t}^{y}\| = \|\nabla f_{t}(x_{t}) - \nabla f_{t}(y_{t})\|
        \le \beta\|x_{t}- y_{t}\| = \beta\delta.
    \end{equation*}
    Then, by applying the inequality $\|a+b\|^{2}\le 2\|a\|^{2}+ 2\|b\|^{2}$ on the error term, we can write the following:
    \begin{align*}
        \|g_{t}^{y}- m_{t}\|^{2} & = \|(g_{t}^{y}- g_{t}^{x}) + (g_{t}^{x}- m_{t})\|^{2}        \\
        & \le 2\|g_{t}^{y}- g_{t}^{x}\|^{2}+ 2\|g_{t}^{x}- m_{t}\|^{2} \\
        & \le 2\beta^{2}\delta^{2}+ 2\|g_{t}^{x}- m_{t}\|^{2}.
    \end{align*}
    The bound on $\|g_{t}^{x}- m_{t}\|^{2}$ follows symmetrically.
\end{proof}

\paragraph{Estimator Decomposition.}
In the analysis of the variance-reduced estimator, we consider its error as consisting of two parts: a zero-mean term and a bias term. The prediction error at the center point is denoted as $a_{t}= g_{t}^{y}- m_{t}$. Using the definitions of $\hat{g}_{t}$ and $r_{t}$ set forth in Lemma~\ref{lem:smoothness-remainder} above, we have:
\begin{equation}
    \hat g_{t}=m_{t}+ \frac{d}{2\delta}\left( 2\delta \langle a_{t}, v_{t}\rangle
    + r_{t}\right) v_{t}=m_{t}+ d\langle a_{t},v_{t}\rangle v_{t}+ \frac{d}{2\delta}
    r_{t}v_{t}. \label{eq:estimator-expanded}
\end{equation}
By simply rearranging, we have:
\begin{equation*}
    \hat g_{t}- g_{t}^{y}= \underbrace{(d v_tv_t^\top - I)a_t}_{= Z_t
    \text{ (Zero-mean Term)}}+ \underbrace{\frac{d}{2\delta}r_t v_t}_{= B_t
    \text{ (Bias)}}. \label{eq:ZB}
\end{equation*}
Here, we denote two terms as $Z_{t} = (dv_{t}v_{t}^{\top}-I)a_{t}$ and $B_{t}=\frac{d}{2\delta}r_{t}v_{t}$, respectively.

\begin{lemma}[Covariance of Uniform Sphere Sampling]
    \label{lem:sphere-covariance} Let $v_{t}$ be sampled uniformly from the unit
    sphere $\mathbb{S}^{d-1}$. Then
    $\mathbb{E}[v_{t}v_{t}^{\top}] = \frac{1}{d}I_{d}$.
\end{lemma}

\begin{proof}
    Let $A= \mathbb{E}[v_{t}v_{t}^{\top}]$ with entries $A_{ij}=\mathbb{E}[v_{t,i}v_{t,j}]$. Since the uniform distribution on $\mathbb{S}^{d-1}$ is invariant to any orthogonal transformation applied to it, then for any distinct indices $i,j$, the transformation that maps $v_{t,j}$ to $-v_{t,j}$ while keeping the remaining coordinates fixed produces the same distribution of $v_{t}$. Consequently, 
    \begin{equation*}
        \mathbb{E}[v_{t,i}v_{t,j}] = \mathbb{E}[v_{t,i}(-v_{t,j})] = -\mathbb{E}[v_{t,i}v_{t,j}],
    \end{equation*} this implies all off-diagonal entries of $A_{ij}$ are zero.
    
    For the diagonal entries of $A$, the symmetry implies that the marginal distribution for each coordinate individually is the same, and therefore for all $i,j$, we have
    \begin{equation*}
        \mathbb{E}[v_{t,i}^{2}] = \mathbb{E}[v_{t,j}^{2}].
    \end{equation*}
    Let $c$ denote
    this value, then since $v_{t}$ lies on the surface of a unit sphere, we have with probability $1$ that; 
    \begin{equation*}
        1 = \mathbb{E}\left[\sum_{i=1}^{d}v_{t,i}^{2}\right] = \sum_{i=1}^{d}\mathbb{E}
        [v_{t,i}^{2}] = \sum_{i=1}^{d}c = d \cdot c.
    \end{equation*}
    Hence, we can conclude that $c=1/d$. Therefore, $A$ becomes the diagonal matrix $A = \frac{1}{d}I_{d}$.
\end{proof}

\begin{lemma}[Bias Analysis]
    \label{lem:bias} Assume $f_{t}$ is $\beta$-smooth. Conditioned on the
    history $\mathcal{H}_{t-1}$, $Z_{t}=(d v_tv_t^\top - I)a_t$ is zero-mean  $\mathbb{E}[Z_{t}\mid \mathcal{H}_{t-1}] = 0$, and the bias term $B_{t}=\frac{d}{2\delta}r_t v_t$
    is bounded as:
    \begin{equation*}
        \|b_{t}\| \le \frac{d}{2}
        \beta\delta, \quad \text{where }b_{t}= \mathbb{E}[B_{t}\mid \mathcal{H}_{t-1}
        ].
    \end{equation*}
    Hence, the expectation of estimated gradient conditioned on the history is:
    \begin{equation*}
        \mathbb{E}[\hat g_{t}\mid \mathcal{H}_{t-1}] = g_{t}^{y}+ b_{t}.
    \end{equation*}
\end{lemma}

\begin{proof}
    Since $v_{t}$ is sampled uniformly from $\mathbb{S}^{d-1}$ independently of $\mathcal{H}
    _{t-1}$, we have the identity
    $\mathbb{E}[v_{t}v_{t}^{\top}] = \frac{1}{d}I_{d}$ from Lemma~\ref{lem:sphere-covariance}.
    For the term $Z_{t}$:
    \begin{equation*}
        \mathbb{E}[Z_{t}\mid \mathcal{H}_{t-1}]= \big(d\mathbb{E}[v_{t}v_{t}^{\top}
        ]-I_{d}\big)a_{t}= 0.
    \end{equation*}
    For the bias term $b_{t}$, we have $|r_{t}|\le \beta\delta^{2}$ from Lemma~\ref{lem:smoothness-remainder}
    and the fact that $\|v_{t}\|=1$:
    \begin{equation*}
        \|b_{t}\|\le\mathbb{E}\big[\|B_{t}\|\mid \mathcal{H}_{t-1}\big]= \mathbb{E}
        \left[ \frac{d}{2\delta}|r_{t}| \|v_{t}\| \ \Big|\ \mathcal{H}_{t-1}\right]\le \frac{d}{2\delta}\cdot \beta\delta^{2}=\frac{d}{2}\beta\delta.
    \end{equation*}
    Also, we have:
    \begin{equation*}
        \mathbb{E}[\hat g_{t}\mid \mathcal{H}_{t-1}] = g_{t}^{y}+ b_{t}.
    \end{equation*}
\end{proof}

\begin{lemma}[Second Moment Bound]
    \label{lem:second-moment-center} Assume $f_{t}$ is $\beta$-smooth. Then the
    second moment of the estimator error around the prediction $m_{t}$ is
    bounded by:
    \begin{equation*}
        \mathbb{E}\left[\|\hat g_{t}-m_{t}\|^{2}\mid \mathcal{H}_{t-1}\right]\le
        2d\|g_{t}^{y}-m_{t}\|^{2}+ \frac{d^{2}}{2}\beta^{2}\delta^{2}.
    \end{equation*}
\end{lemma}

\begin{proof}
    From~\eqref{eq:estimator-expanded},  we can expand on the definition of the estimator as follows: 
    \begin{equation*}
        \hat g_{t}- m_{t}\;=\; \left( d\langle a_{t}, v_{t}\rangle + \frac{d}{2\delta}
        r_{t}\right) v_{t}.
    \end{equation*}
    Taking the squared norm and $\|v_{t}\|=1$, we get the following:
    \begin{equation*}
        \|\hat g_{t}- m_{t}\|^{2}\;=\; \left( d\langle a_{t}, v_{t}\rangle + \frac{d}{2\delta}
        r_{t}\right)^{2}.
    \end{equation*}
    By the fact that $(x+y)^{2}\le 2x^{2}+ 2y^{2}$, the right-hand side of the above equation is written as:
    \begin{equation*}
        \|\hat g_{t}- m_{t}\|^{2}\le 2d^{2}\langle a_{t}, v_{t}\rangle^{2}\;+\; 2
        \left(\frac{d}{2\delta}\right)^{2}r_{t}^{2}.
    \end{equation*}
    Since
    $\mathbb{E}[\langle a_{t}, v_{t}\rangle^{2}] = \frac{1}{d}\|a_{t}\|^{2}$ and
    $r_{t}^{2}\le \beta^{2}\delta^{4}$ from $|r_{t}| \le \beta\delta^{2}$. Hence,
    we have
    \begin{align*}
        \mathbb{E}\left[\|\hat g_{t}-m_{t}\|^{2}\mid \mathcal{H}_{t-1}\right] & \;\le\; 2d^{2}\left( \frac{\|a_{t}\|^{2}}{d}\right) \;+\; \frac{d^{2}}{2\delta^{2}}(\beta^{2}\delta^{4}) \\
        & \;=\; 2d\|a_{t}\|^{2}\;+\; \frac{d^{2}}{2}\beta^{2}\delta^{2}.
    \end{align*}
    By substituting $a_{t}= g_{t}^{y}- m_{t}$, we complete this proof.
\end{proof}

\begin{corollary}[Second Moment Bound with Gradient at $x_{t}$]
    \label{cor:second-moment-play} Under the assumptions of Lemma~\ref{lem:second-moment-center},
    \begin{equation*}
        \mathbb{E}\left[\|\hat g_{t}-m_{t}\|^{2}\middle| \mathcal{H}_{t-1}\right
        ] \le 4d\,\mathbb{E}\left[\|g_{t}^{x}-m_{t}\|^{2}\;\middle|\; \mathcal{H}
        _{t-1}\right] + (d^{2}/2+4d)\beta^{2}\delta^{2}.
    \end{equation*}
    And in particular,
    \begin{equation*}
        \sum_{t=1}^{T}\mathbb{E}\left[\|\hat g_{t}-m_{t}\|^{2}\right]\le 4d\bar S
        _{T}+ (d^{2}/2+4d)\beta^{2}\delta^{2}T.
    \end{equation*}
\end{corollary}

\begin{proof}
    We know that from Lemma~\ref{lem:gradient-stability}, we have:
    \begin{equation*}
        \|g_{t}^{y}-m_{t}\|^{2}\le 2\|g_{t}^{x}-m_{t}\|^{2}+ 2\beta^{2}\delta^{2}.
    \end{equation*}
    Taking the conditional expectation and substituting this into Lemma~\ref{lem:second-moment-center},
    we get:
    \begin{align*}
        \mathbb{E}\left[\|\hat g_{t}-m_{t}\|^{2}\;\middle| \mathcal{H}_{t-1}\right] & \le 2d\left( 2\mathbb{E}\left[\|g_{t}^{x}-m_{t}\|^{2}\;\middle| \mathcal{H}_{t-1}\right] + 2\beta^{2}\delta^{2}\right) + \frac{d^{2}}{2}\beta^{2}\delta^{2} \\
        & = 4d\,\mathbb{E}\left[\|g_{t}^{x}-m_{t}\|^{2}\;\middle| \mathcal{H}_{t-1}\right] + (d^{2}/2+4d)\beta^{2}\delta^{2}.
    \end{align*}
    To get the total over the time horizon $T$, we take two iterated expectations over the history $\mathcal{H}_{t-1}$ on both sides by applying the law of total expectation. We have: $\mathbb{E}[\mathbb{E}[\cdot|\mathcal{H}_{t-1}]] = \mathbb{E}[\cdot]$:
    \begin{equation*}
        \mathbb{E}\left[\|\hat g_{t}-m_{t}\|^{2}\right]\le 4d\,\mathbb{E}\left[\|
        g_{t}^{x}-m_{t}\|^{2}\right] + (d^{2}/2+4d)\beta^{2}\delta^{2}.
    \end{equation*}
    Finally, adding this inequality up for all $t=1,\dots,T$ gives us the final conclusion.
\end{proof}

\begin{lemma}[Regret Decomposition for \tpvr]
    \label{lem:regret-decomposition} Assume that $0 \in \cX$ and each function $f_{t}$
    is convex and $L$-Lipschitz. Let
    $x^{\star}\in \arg\min_{x\in \cX}\sum_{t=1}^{T}f_{t}(x)$ denote the optimal
    comparator in the original set, and let $y^{\star}= (1-\alpha)x^{\star}$ be
    its projection onto the set $\cX_{\alpha}$. Then, the regret of \tpvr\, is bounded by:
    \begin{equation}
        \regret_{T}\le \sum_{t=1}^{T}\big(f_{t}(y_{t})-f_{t}(y^{\star})\big)+\left(L+\frac{LD}{r}\right)\delta T. \label{eq:regret-reduce}
    \end{equation}
\end{lemma}

\begin{proof}
      By definition, we can write the regret of a player at time $t$ as the sum of three terms: the perturbation cost, the regret relative to the shrunken comparator, and the approximation error of the comparator.
    \begin{equation*}
        f_{t}(x_{t}) - f_{t}(x^{\star})= \underbrace{\big(f_t(x_t) - f_t(y_t)\big)}
        _{\text{Term (A)}}+ \big(f_{t}(y_{t}) - f_{t}(y^{\star})\big)+ \underbrace{\big(f_t(y^\star)
        - f_t(x^\star)\big)}_{\text{Term (B)}}.
    \end{equation*}
    For term (A), since $\|x_{t}- y_{t}\| = \delta$, we have:
    \begin{equation*}
        f_{t}(x_{t}) - f_{t}(y_{t})\le L\|x_{t}- y_{t}\| = L\delta.
    \end{equation*}
    For term (B), since $y^{\star}= (1-\alpha)x^{\star}$ and $0 \in \cX$, we have
    $\|y^{\star}- x^{\star}\| = \alpha \|x^{\star}\| \le \alpha D$. Then, due to $\alpha
    = \delta/r$, we obtain:
    \begin{equation*}
        f_{t}(y^{\star}) - f_{t}(x^{\star}) \;\le\; L\|y^{\star}- x^{\star}\| \;
        \le\; L\alpha D \;=\; \frac{LD}{r}\delta.
    \end{equation*}
    Combining the above for $t=1,\dots,T$ gives us all terms on the right-hand side of \eqref{eq:regret-reduce}. 
\end{proof}

\begin{lemma}[Optimistic projected-gradient inequality]
    \label{lem:optimistic-ogd} Given that any closed convex set $C\subset\mathbb{R}^{d}$ is defined, $\Pi_{C}$ is the projection onto set $C$ and $\eta\ >\ 0$ is fixed constant. 
    Let two update sequences be determined as follows for each value of $t$:
    \begin{equation*}
        y_{t}= \Pi_{C}(y'_{t}-\eta m_{t}), \qquad y'_{t+1}= \Pi_{C}(y'_{t}-\eta \hat
        g_{t}).
    \end{equation*}
    Assume that $y'_{t}\in C$. For any fixed comparator $y^{\star}\in C$ and an arbitrary horizon $T\ge 1$, we get the following inequality:
    \begin{equation*}
        \sum_{t=1}^{T}\langle \hat g_{t}, y_{t}-y^{\star}\rangle \le \frac{\|y^{\star}-y'_{1}\|^{2}}{2\eta}
        + \eta\sum_{t=1}^{T}\|\hat g_{t}-m_{t}\|^{2}. \label{eq:optimistic-ogd}
    \end{equation*}
\end{lemma}

\begin{proof}
    Let $x^{+}= \Pi_{C}(x-\eta g)$. By using the first-order conditions of optimality, for all $z\in C$:
    \begin{equation*}
        \langle (x - \eta g) - x^{+}, z - x^{+}\rangle \le 0.
    \end{equation*}
    The above inequality leads to:
    \begin{equation*}
        \eta \langle g, x^{+}- z \rangle \le \langle x - x^{+}, x^{+}- z \rangle.
    \end{equation*}
    We can rewrite the inner product on the right-hand side as: 
     $2\langle a, b\rangle = \|a+b\|^{2}- \|a\|^{2}- \|b\|^{2}$, where $a = x - x^{+}$ and $b = x^{+}- z$. Therefore, $a+b = x-z$, and we have:
    \begin{equation*}
        \langle x - x^{+}, x^{+}- z \rangle = \frac{1}{2}\left( \|x - z\|^{2}- \|
        x - x^{+}\|^{2}- \|x^{+}- z\|^{2}\right).
    \end{equation*}
    By combining the two inequalities above, we obtain the following bound:
    \begin{equation}
        \langle g, x^{+}-z\rangle \le \frac{\|z-x\|^{2}- \|z-x^{+}\|^{2}- \|x^{+}-x\|^{2}}{2\eta}. \label{eq:proj-step-ineq}
    \end{equation}

    We apply \eqref{eq:proj-step-ineq} twice at each round $t$. Applying \eqref{eq:proj-step-ineq}
    with $(x, g, x^{+}, z) = (y'_{t}, m_{t}, y_{t}, y'_{t+1})$, we obtain:
    \begin{equation}
        \langle m_{t}, y_{t}-y'_{t+1}\rangle \le \frac{\|y'_{t+1}-y'_{t}\|^{2}-
        \|y'_{t+1}-y_{t}\|^{2}- \|y_{t}-y'_{t}\|^{2}}{2\eta}. \label{eq:step1}
    \end{equation}

    Applying \eqref{eq:proj-step-ineq} with
    $(x, g, x^{+}, z) = (y'_{t}, \hat g_{t}, y'_{t+1}, y^{\star})$, we obtain:
    \begin{equation}
        \langle \hat g_{t}, y'_{t+1}-y^{\star}\rangle \le \frac{\|y^{\star}-y'_{t}\|^{2}-
        \|y^{\star}-y'_{t+1}\|^{2}- \|y'_{t+1}-y'_{t}\|^{2}}{2\eta}. \label{eq:step2}
    \end{equation}

    On the other hand, 
    \begin{equation*}
        \|y_{t}-y'_{t+1}\| =\|\Pi_{C}(y'_{t}-\eta m_{t}) - \Pi_{C}(y'_{t}-\eta \hat
        g_{t})\| \le \|(y'_{t}-\eta m_{t}) - (y'_{t}-\eta \hat g_{t})\|= \eta\|\hat
        g_{t}-m_{t}\|.
    \end{equation*}
    From the Cauchy-Schwarz inequality, we have:
    \begin{equation}
        \langle \hat g_{t}-m_{t}, y_{t}-y'_{t+1}\rangle \le \|\hat g_{t}-m_{t}\|
        \cdot \|y_{t}-y'_{t+1}\| \le \eta\|\hat g_{t}-m_{t}\|^{2}. \label{eq:cross}
    \end{equation}

    From the above equations, we can express the regret term at time $t$ as:
    \begin{equation*}
        \langle \hat g_{t}, y_{t}-y^{\star}\rangle = \underbrace{\langle \hat
        g_t, y'_{t+1}-y^\star\rangle}_{\text{via } \eqref{eq:step2}}+ \underbrace{\langle
        \hat g_t-m_t, y_t-y'_{t+1}\rangle}_{\text{via } \eqref{eq:cross}}+ \underbrace{\langle
        m_t, y_t-y'_{t+1}\rangle}_{\text{via } \eqref{eq:step1}}.
    \end{equation*}
    We note the term $\frac{1}{2\eta} \|y'_{t+1}-y'_{t}\|^{2}$ appears in \eqref{eq:step2} with a negative sign and in \eqref{eq:step1} with a positive sign, leading to cancellation. Removing the remaining negative terms ($-\|y'_{t+1}-y_{t}\|^{2}$ and $-\|y_{t}-y'_{t}\|^{2}$), we obtain:
    \begin{equation*}
        \langle \hat g_{t}, y_{t}-y^{\star}\rangle \le \frac{\|y^{\star}-y'_{t}\|^{2}-
        \|y^{\star}-y'_{t+1}\|^{2}}{2\eta}+ \eta\|\hat g_{t}-m_{t}\|^{2}.
    \end{equation*}
    Finally, we can sum over $t=1,\ldots,T$ to obtain:
    \begin{equation*}
        \sum_{t=1}^{T}\langle \hat g_{t}, y_{t}-y^{\star}\rangle \le \frac{\|y^{\star}-y'_{1}\|^{2}-
        \|y^{\star}-y'_{T+1}\|^{2}}{2\eta}+ \eta\sum_{t=1}^{T}\|\hat g_{t}-m_{t}\|
        ^{2}.
    \end{equation*}
    Dropping the non-positive term $-\frac{\|y^{\star}-y'_{T+1}\|^{2}}{2\eta}$ completes this
    proof.
\end{proof}

\subsection{Proof of Theorem~\ref{thm:exp-regret}}
\label{app:proof-expregret}
\begin{proof}
    We start with the regret decomposition given by Lemma~\ref{lem:regret-decomposition}:
    \begin{equation*}
        \regret_{T}\le \sum_{t=1}^{T}\big(f_{t}(y_{t})-f_{t}(y^{\star})\big) + \left(L+\frac{LD}{r}\right)\delta T.
    \end{equation*}
    Using the convexity of $f_{t}$, we establish the following:
    \begin{equation*}
        f_{t}(y_{t})-f_{t}(y^{\star}) \le \langle \nabla f_{t}(y_{t}), y_{t}-y^{\star}
        \rangle = \langle g_{t}^{y}, y_{t}-y^{\star}\rangle.
    \end{equation*}
    Next, we decompose the term $\langle g_{t}^{y}, y_{t}-y^{\star}\rangle$ into the following two terms:
    \begin{equation*}
        \langle g_{t}^{y}, y_{t}-y^{\star}\rangle=\underbrace{\langle \hat g_t, y_t-y^\star\rangle}
        _{\text{Algorithmic Regret}}+\underbrace{\langle g_t^y-\hat g_t, y_t-y^\star\rangle}
        _{\text{Estimation Bias}}.
    \end{equation*}

    Applying the bound from Lemma~\ref{lem:optimistic-ogd} to the convex set $C=\cX_{\alpha}$ gives us: $\sum_{t=1}^{T}\langle \hat g_{t}, y_{t}-y^{\star}\rangle \le \frac{\|y^{\star}-y'_{1}\|^{2}}{2\eta}+ \eta\sum_{t=1}^{T}\|\hat g_{t}-m_{t}\|^{2}$.
    Since both $y^{\star}$ and $y'_{1}$ lie in $\cX_{\alpha}\subseteq \cX$, we have
    $\|y^{\star}-y'_{1}\|^{2}\le D^{2}$, which gives us:
    \begin{equation*}
        \sum_{t=1}^{T}\langle \hat g_{t}, y_{t}-y^{\star}\rangle\le \frac{D^{2}}{2\eta}
        + \eta\sum_{t=1}^{T}\|\hat g_{t}-m_{t}\|^{2}.
    \end{equation*}
    Taking expectations of this expression gives the following bound using Corollary~\ref{cor:second-moment-play}:
    \begin{equation*}
    \mathbb{E}\left[\sum_{t=1}^{T}\langle \hat g_{t}, y_{t}-y^{\star}\rangle\right] \;\le\; \frac{D^{2}}{2\eta}\;+\; \eta\left(4d\,\bar S_{T}+ \left(\frac{d^{2}}{2}+4d\right)\beta^{2}\delta^{2}T\right).
    \end{equation*}

    From Lemma~\ref{lem:bias}, we know that
    $\mathbb{E}[\hat g_{t}\mid \mathcal{H}_{t-1}] = g_{t}^{y}+ b_{t}$, which implies
    $\mathbb{E}[g_{t}^{y}- \hat g_{t}\mid \mathcal{H}_{t-1}] = -b_{t}$. Therefore, we can use the tower property to find:
    \begin{align*}
        \mathbb{E}\big[\langle g_{t}^{y}-\hat g_{t}, y_{t}-y^{\star}\rangle\big] & \;=\; \mathbb{E}\Big[\big\langle \mathbb{E}[g_{t}^{y}-\hat g_{t}\mid \mathcal{H}_{t-1}],  y_{t}-y^{\star}\big\rangle\Big] \\
        & = -\mathbb{E}\big[\langle b_{t}, y_{t}-y^{\star}\rangle\big].
    \end{align*}
    By the Cauchy-Schwarz inequality and the bound on
    $\|b_{t}\| \le \frac{d}{2}\beta\delta$ shown in Lemma~\ref{lem:bias}, we have:
    \begin{equation*}
        -\langle b_{t}, y_{t}-y^{\star}\rangle \leq |\langle b_{t}, y_{t}-y^{\star}\rangle|\leq \|b_{t}\| \cdot \|y_{t}-y^{\star}
        \| \le \frac{d}{2}\beta\delta \cdot D.
    \end{equation*}
    Summing the above over $t=1,\dots,T$  gives us a total expected bias bound by $\frac{Dd\beta\delta
    T}{2}$.

    Substituting the above into the regret decomposition including the perturbation cost gives us the final bound:
    \begin{equation*}
        \mathbb{E}[\regret_{T}]\le\frac{D^{2}}{2\eta}+\eta\left( 4d\,\bar S_{T}+
        \left(\frac{d^{2}}{2}+4d\right)\beta^{2}\delta^{2}T \right) +\frac{Dd\beta\delta
        T}{2}+\left(L+\frac{LD}{r}\right)\delta T. \label{eq:main-exp-bound}
    \end{equation*}
\end{proof}

\subsection{Proof of Theorem~\ref{thm:two_point}}

\begin{proof}
    The proof starts with the expected regret bound established in Theorem~\ref{thm:exp-regret}, as indicated below:
    \begin{equation*}
        \mathbb{E}[\regret_{T}] \le \underbrace{\frac{D^{2}}{2\eta} + 4d\eta \bar{S}_T}
        _{\text{Term (A)}}+ \underbrace{\eta\left(\frac{d^{2}}{2}+4d\right)\beta^2\delta^2T}
        _{\text{Term (B)}}+ \underbrace{\left[ \frac{Dd\beta}{2} + \left(L+\frac{LD}{r}\right)
        \right] \delta T}_{\text{Term (C)}}.
    \end{equation*}

    \paragraph{Term (A):}
    Substituting the learning rate $\eta = \frac{D}{\sqrt{8d(\bar{S}_{T}+1)}}$ into
    it, we have:
    \begin{align*}
        \frac{D^{2}}{2\eta}+ 4d\eta \bar{S}_{T} & = \frac{D^{2}}{2}\cdot \frac{\sqrt{8d (\bar{S}_{T}+1)}}{D}+ 4d \bar{S}_{T}\cdot \frac{D}{\sqrt{8d (\bar{S}_{T}+1)}} \\
        & = \frac{D}{2}\sqrt{8d (\bar{S}_{T}+1)}+ D\sqrt{\frac{16 d^{2}\bar{S}_{T}^{2}}{8d (\bar{S}_{T}+1)}}\\
        & \leq D\sqrt{2d (\bar{S}_{T}+1)}+ D\sqrt{2d \bar{S}_{T}}\\
        &= O\left( D\sqrt{d\bar{S}_{T}}\right).
    \end{align*}

    \paragraph{Term (B):}
    Since $\delta \le \frac{\sqrt{\bar{S}_{T}+1}}{d\beta T}$, we
    substitute $\delta^{2}\le \frac{\bar{S}_{T}+1}{d^{2}\beta^{2}T^{2}}$ into Term
    (B):
    \begin{align*}
        \eta\left(\frac{d^{2}}{2}+4d\right)\beta^{2}\delta^{2}T & \le \eta\left(\frac{d^{2}}{2}+4d\right)\beta^{2}T \cdot \frac{\bar{S}_{T}+1}{d^{2}\beta^{2}T^{2}} \\
        & = \eta \left(\frac{1}{2}+ \frac{4}{d}\right) \frac{\bar{S}_{T}+1}{T}.
    \end{align*}
    Substituting $\eta = \frac{D}{\sqrt{8d (\bar{S}_{T}+1)}}$ gives us:
    \begin{align*}
        \frac{D}{\sqrt{8d (\bar{S}_{T}+1)}}\left(\frac{1}{2}+ \frac{4}{d}\right) \frac{\bar{S}_{T}+1}{T}= O\left( \frac{D \sqrt{\bar{S}_{T}}}{T}\right).
    \end{align*}

    \paragraph{Term (C):}
    Also substituting $\delta \le \frac{\sqrt{\bar{S}_{T}+1}}{d\beta T}$, we have:
    \begin{align*}
        \left[ \frac{Dd\beta}{2}+ \left(L+\frac{LD}{r}\right) \right] \delta T & \le \left[ \frac{Dd\beta}{2}+ \left(L+\frac{LD}{r}\right) \right] \frac{\sqrt{\bar{S}_{T}+1}}{d\beta T}\cdot T \\
        & = \left[ \frac{D}{2}+ \frac{L(1+D/r)}{d\beta}\right] \sqrt{\bar{S}_{T}+1}\\
        & = O\left( D \sqrt{\bar{S}_{T}}\right).
    \end{align*}

    The perturbation radius $\delta$ defined in Theorem~\ref{thm:two_point} is defined as the minimum of the three quantities: $\{ \frac{\sqrt{\bar{S}_{T}+1}}{d\beta T}, \frac{1}{\beta\sqrt{T}}, \frac{r}{2}\}$. The constraint $\delta\leq \frac{r}{2}$ ensures that there is sufficient room for random perturbations while maintaining feasibility. Indeed, if $\delta= r$, then $\alpha=\delta/r=1$ and consequently $\cX_\alpha=\cX$. In this case, to guarantee $x_t=y_t+\delta v_t\in \cX$ for $\|v\|=1$, the point $y_t$ would have to lie at the center of $\cX$, which is overly restrictive in general. Moreover, the condition $\delta\leq 1/(\beta\sqrt{T})$ is imposed to control the accumulation of the second-order error term. Specifically, this ensures that the cumulative contribution $\beta^2\delta^2 T$ remains uniformly bounded over the horizon $T$. Since both Terms (B) and (C) are monotonically increasing in $\delta$, if $\delta$ is constrained by the second and third terms, the final regret should be smaller than the result by substituting $\delta\leq \frac{\sqrt{\bar{S}_{T}+1}}{d\beta T}$.  To derive the upper bound here, it is sufficient for us to substitute the first term $\delta\leq \frac{\sqrt{\bar{S}_{T}+1}}{d\beta T}$

    Combining the bounds from Terms (A), (B), and (C), the total expected regret is
    written as:
    \begin{equation*}
        \mathbb{E}[\regret_{T}] = O\left( D \sqrt{d \bar S_{T}}\right).
    \end{equation*}
\end{proof}

\subsection{Omitted Details of Section~\ref{sec:comp_chiang}}\label{app:comp_chiang}

We first recall the definition of gradient variation proposed in~\cite{chiang2013beating} as:
\begin{equation*}
    D_T=\sum_{t=1}^T \max_{x\in\cX}\|\nabla f_t(x)-\nabla f_{t-1}(x)\|^2.
\end{equation*}
By specializing our proposed \tpvr\, with restricting the sample distribution over $\mathbb{S}^{d-1}$ to the standard basis vectors, i.e, $v_t\sim \mathrm{Unif}(\{e_1,\dotsc,e_d\})$, the variant of \tpvr\, can be written shown as the following Algorithm~\ref{alg:variant_tpvr}. Our notation $\Delta_{t,i_t}$ in \tpvr\, is written as $\Delta_{t,i_t}=2\delta(\widehat{v}_{t,i_t}-m_{t,i_t})$ for convenience here.

Let $\alpha_t=\max\{\tau<t: i_{\tau}=i_t\}$ with $\alpha_t=0$ if the set is empty. By the definition of $m_t$ whose $i$-th coordinate stores the most
recent finite-difference observation on that coordinate, we have $m_{t,i_t}=\widehat{v}_{\alpha_t,i_t}$.

\begin{lemma}
    Assume each $f_t$ is $\beta$-smooth. For any $y_t\in\cX_{\alpha}$ and any coordinate $i$, we have
    \begin{equation*}
        |\widehat{v}_{t,i}(y_t)-\nabla_i f_t(y_t)|\leq \frac{\beta\delta}{2}.
    \end{equation*}
    Then, 
    \begin{equation*}
        \sum_{t=1}^T \mathbb{E}\left[\|\widehat{g}_t-m_t\|^2\right]\leq 3d^2\mathbb{E}\left[\sum_{t=1}^T\widehat{D}_t\right]+\frac{3}{2}d^2\beta^2\delta^2T,
    \end{equation*}
    where $\widehat{D}_t=\|\nabla_{i_t} f_t(y_t)-\nabla_{i_t}f_{\alpha_t}(y_{\alpha_t})\|^2$.
\end{lemma}
\begin{proof}
    Replacing the $v_t$ with $e_{t,i}$ in Lemma~\ref{lem:smoothness-remainder} and dividing $2\delta$, we have
    \begin{equation*}
        |\widehat{v}_{t,i}(y_t)-\nabla_i f_t(y_t)|\leq \frac{\beta\delta}{2}.
    \end{equation*}

    By the defintion of the esimator $\widehat{g}_t$, we have
\begin{equation*}
    \|\widehat{g}_t-m_t\|^2=d^2\|\widehat{v}_{t,i_t}-m_{t,i_t}\|^2
\end{equation*}
Then, we can decompose 
\begin{equation*}
    \widehat{v}_{t,i_t}-m_{t,i_t}=\widehat{v}_{t,i_t}-\widehat{v}_{\alpha_t,i}=(\nabla_{i_t} f_t(y_t)-\nabla_{\alpha_t} f_t(y_t))+(\widehat{v}_{t,i_t}-\nabla_{\alpha_t} f_t(y_t))+(\nabla_{i_t} f_t(y_t)-\widehat{v}_{\alpha_t,i}).
\end{equation*}

Hence, we have:
\begin{equation*}
    \|\widehat{v}_{t,i_t}-m_{t,i_t}\|^2\leq 3\|\nabla_{i_t} f_t(y_t)-\nabla_{\alpha_t} f_t(y_t)\|^2+\frac{3}{2}\beta^2\delta^2
\end{equation*}
\begin{equation*}
    \sum_{t=1}^T \mathbb{E}\left[\|\widehat{g}_t-m_t\|^2\right]\leq 3d^2\mathbb{E}\left[\sum_{t=1}^T\widehat{D}_t\right]+\frac{3}{2}d^2\beta^2\delta^2T,
\end{equation*}
where $\widehat{D}_t=\|\nabla_{i_t} f_t(y_t)-\nabla_{i_t}f_{\alpha_t}(y_{\alpha_t})\|^2$.
\end{proof}

By applying Lemmas 8 and 12 in~\cite{chiang2013beating}, we have
\begin{equation}\label{eq:hatD}
\mathbb{E}\left[\sum_{t=1}^T\widehat{D}_t\right]\leq 4d^2 D_T+8d\beta^2\log T\cdot\mathbb{E}\left[\sum_{t=2}^T\|y_t-y_{t-1}\|^2\right]
\end{equation}

\begin{lemma}
Let $y_t=\Pi_{\mathcal X_\alpha}(y'_t-\eta m_t)$ and
$y'_{t+1}=\Pi_{\mathcal X_\alpha}(y'_t-\eta \widehat g_t)$.
Then for any $y^\star\in\mathcal X_\alpha$,
\begin{align*}
\sum_{t=1}^T \langle \widehat g_t,\, y_t-y^\star\rangle &\le
\frac{\|y^\star-y'_1\|^2}{2\eta}
+\eta\sum_{t=1}^T \|\widehat g_t-m_t\|^2-\frac{1}{2\eta}\sum_{t=1}^T \Bigl(\|y_t-y'_t\|^2+\|y'_{t+1}-y_t\|^2\Bigr),\\
&\leq \frac{\|y^\star-y'_1\|^2}{2\eta}
+\eta\sum_{t=1}^T \|\widehat g_t-m_t\|^2-\frac{1}{4\eta}\sum_{t=2}^{T}\|y_t-y_{t-1}\|^2.
\end{align*} 
\end{lemma}

\begin{proof}
    The first inequality above can be directly obtained by Lemma~\ref{lem:optimistic-ogd} by not dropping the two negative terms.

    Then, we have
    \begin{align*}
        \sum_{t=1}^T \Bigl(\|y_t-y'_t\|^2+\|y'_{t+1}-y_t\|^2\Bigr)&=\|y_1-y'_1\|^2+\sum_{t=2}^T(\|y_t-y'_t\|^2+\|y'_{t}-y_{t-1}\|^2)+\|y'_{T+1}-y_{T}\|^2\\
        &\geq \sum_{t=2}^T(\|y_t-y'_t\|^2+\|y'_{t}-y_{t-1}\|^2)\\
        &\geq \frac{1}{2}\sum_{t=2}^T \|y_t-y_{t-1}\|^2.
    \end{align*}
These complete the proof.
\end{proof}

By substituting the above lemmas and \eqref{eq:hatD} into the proof of Theorem~\ref{thm:two_point} as shown in Appendix~\ref{app:proof-expregret}, we have  
\begin{equation*}
        \mathbb{E}[\regret_{T}]\le\frac{D^{2}}{2\eta}+12\eta d^4 D_T+(\eta\cdot 24d^3\beta^2\log T-\frac{1}{4\eta})\sum_{t=2}^T \|y_t-y_{t-1}\|^2+\frac{3}{2}d^2\eta\beta^2\delta^2T +\frac{Dd\beta\delta
        T}{2}+\left(L+\frac{LD}{r}\right)\delta T. 
\end{equation*}

By tuning the parameters with
\begin{equation*}
    \eta=\min\left\{\frac{D}{4d^2\sqrt{D_T+1}},\frac{1}{16\beta d^{3/2}\sqrt{\log T}}\right\},\qquad 
    \delta=\min\left\{\frac{\sqrt{D_T+1}}{\beta T},\frac{1}{\beta\sqrt{T}},\frac{r}{2}\right\},
\end{equation*}
we have
\begin{equation*}
    \mathbb{E}[\regret_{T}]\le\widetilde{O}(Dd^2 \sqrt{D_T})
\end{equation*}
This recovers the result shown in \cite{chiang2013beating}.

\subsection{Useful Lemmas in Section~\ref{sec:tp-vr-adaptive}}

Let $T$ be the unknown time horizon. The algorithm runs in time phases, where the phases are indexed from $k=1, \ldots, K$. Each time phase has an increasing length of $H_{k}=2^{k-1}$. Let $I_{k}$ denote the set of rounds
executed in phase $k$.  We have $|I_{k}|\le H_{k}$ and the phases partition the
horizon such that $\sum_{k=1}^{K}|I_{k}| = T$, which implies that $K \le \lceil\log
_{2}T\rceil+1 = O(\log T)$. For each phase $k$, we define the \emph{phase-wise
expected prediction-sensitivity} as follows:
\begin{equation*}
    \bar S_{k}= \mathbb{E}\bigg[\sum_{t\in I_k}\|\nabla f_{t}(x_{t})-m_{t}\|^{2}\bigg].
\end{equation*}
From the definition, it can be seen that: $\sum_{k=1}^{K}\bar S_{k}= \bar S_{T}$. Within each time phase $k$, the sensitivity-doubling mechanism generates a set of epochs: $e=1,\ldots,E_{k}$, where each epoch has a sensitivity budget of $S_{k,e}=2^{e-1}$. Let us define the rounds executed in the epoch $(k,e)$ as $J_{k,e}\subseteq I_{k}$, and let the accumulated epoch residual be defined as follows:
\begin{equation*}
    R_{k,e}= \sum_{t\in J_{k,e}}\|\hat g_{t}-m_{t}\|^{2}.
\end{equation*}

\begin{lemma}
    \label{lem:oneshot-residual-anytime} Assume each function $f_{t}$ is $L$-Lipschitz
    on $\cX$ and the optimistic prediction $m_{t}$ satisfies $\|m_{t}\|\le L$. Then,
    for every round $t$, the estimator error is bounded by:
    \begin{equation*}
        \|\hat g_{t}-m_{t}\| \le 2dL.
    \end{equation*}
    For each epoch $(k,e)$, let $J_{k,e}$ be the set of rounds in this epoch. If the epoch terminates because the residual threshold is exceeded, let $\tau_{k,e}$ denote its last round and set
\begin{equation*}
    O_{k,e}=\{\tau_{k,e}\},
    \qquad
    J_{k,e}^{\circ}=J_{k,e}\setminus O_{k,e}.
\end{equation*}
If the epoch terminates because the time budget is reached, set $O_{k,e}=\emptyset$ and $J_{k,e}^{\circ}=J_{k,e}$. Then $|O_{k,e}|\le 1$ and
\begin{equation*}
    \sum_{t\in J_{k,e}^{\circ}}\|\hat g_t-m_t\|^2
    \le
    8dS_{k,e}.
\end{equation*}
Moreover, the full epoch residual satisfies
\begin{equation*}
    \sum_{t\in J_{k,e}}\|\hat g_t-m_t\|^2
    \le
    8dS_{k,e}+4d^2L^2.
\end{equation*}
\end{lemma}

\begin{proof}
    From the definition of the estimator $\hat g_{t} = m_{t} + \frac{d}{2\delta}\Delta
    _{t} v_{t}$ with
    \begin{equation*}
        \Delta_{t} = f_{t}(y_{t}+\delta v_{t}) - f_{t}(y_{t}-\delta v_{t}) - 2\delta
        \langle m_{t},v_{t}\rangle,
    \end{equation*}
    we have
    \begin{equation*}
        |\Delta_{t}|\le \big|f_{t}(y_{t}+\delta v_{t}) - f_{t}(y_{t}-\delta v_{t}
        )\big| + \big|2\delta\langle m_{t},v_{t}\rangle\big| \le 2L\delta + 2\delta
        \|m_{t}\|\|v_{t}\|.
    \end{equation*}
    Since $\|m_{t}\|\le L$ and $\|v_{t}\|=1$, we have $|\Delta_{t}| \le 4L\delta$.
    Substituting it back to $\hat g_{t} = m_{t} + \frac{d}{2\delta}\Delta
    _{t} v_{t}$, we obtain:
    \begin{equation*}
        \|\hat g_{t}-m_{t}\|= \left\| \frac{d}{2\delta}\Delta_{t} v_{t} \right\|
        =\frac{d}{2\delta}|\Delta_{t}| \le\frac{d}{2\delta}(4L\delta)=2dL.
    \end{equation*}

    When we consider the stopping rule for this epoch $(k,e)$. Let $\tau$ be the index of the last round executed in this epoch. The final round lasts until \emph{either} the accumulated residual $R_{k,e}$ becomes larger than the threshold \emph{or} we have reached the end of the time phase; therefore, we can analyze these two different scenarios by doing the following:
    \begin{itemize}
        \item \textbf{Case 1 (Time Limit):} If an epoch ends at the maximum length of the time phase $H_{k}$, the accumulated residual did not exceed the threshold; therefore:
        \begin{equation*}
            R_{k,e}\le 8d S_{k,e}< 8d S_{k,e}+ 4d^{2}L^{2}.
        \end{equation*}

        \item \textbf{Case 2 (Residual Exceed):} If an epoch exceeded the accumulated residual before reaching $H_{k}$, then we need to check out the last added term at round $\tau$. If we had met or exceeded the residual threshold before reaching $H_{k}$, then we can establish the preceding residual before adding term $\tau$: 
            \begin{equation*}
                \sum_{t \in J_{k,e} \setminus \{\tau\}}\|\hat g_{t} - m_{t}\|^{2}
                \le 8d S_{k,e}.
            \end{equation*}
            Adding the final term and applying the per-round as shown above, we have:
            \begin{equation*}
                R_{k,e}= \sum_{t \in J_{k,e} \setminus \{\tau\}}\|\hat g_{t} - m_{t}
                \|^{2} \;+\; \|\hat g_{\tau} - m_{\tau}\|^{2} \le 8d S_{k,e}+ 4d^{2}
                L^{2}.
            \end{equation*}
    \end{itemize}
    Combining both cases completes the proof.
\end{proof}

\begin{lemma}[Epoch Regret Bound with Fixed Parameters]
    \label{lem:epoch-regret-anytime} Assume each $f_{t}$ is convex, $L$-Lipschitz,
    and $\beta$-smooth on $\cX$.  Considering a specific epoch, for fixed time budget $H$ and prediction-sensitivity budget $S$, we set parameters $\eta=D/\sqrt{8dS}$ and $\delta
    =\min\{\frac{\sqrt{S}}{d\beta H},\frac{1}{\beta\sqrt{H}},\frac{r}{2}\}$. Let
    $J$ be the set of rounds in this epoch. The expected regret $\mathbb{E}[\regret(J)]$ on $J$
    is written as:
    \begin{equation*}
        \mathbb{E}[\regret(J)] \le 3\sqrt{2}\,D\sqrt{dS}+\frac{D}{2}\sqrt{S}+ \left(L+\frac{LD}{r}\right)\frac{\sqrt{S}}{d\beta}+LD.
    \end{equation*}
\end{lemma}

\begin{proof}
    To show this result, we apply the regret decomposition as our proof for \tpvr\, derived in Theorem~\ref{thm:exp-regret},
    to the interval $J$ and split the interval $J$ into $J=J^{\circ}\cup O$. We first consider $t\in J^{\circ}$. Since the parameters $(\eta,\delta,\alpha)$ are
    fixed within the epoch, we have:
    \begin{equation}
        \mathbb{E}[\regret(J^{\circ})]\le \underbrace{\frac{D^{2}}{2\eta} + \eta\,\mathbb{E}\bigg[\sum_{t\in J}\|\hat
        g_t-m_t\|^2\bigg]}_{\text{Term (A)}}+\underbrace{\left( \frac{Dd\beta}{2}
        + L + \frac{LD}{r} \right) \delta\,\mathbb{E}[|J^{\circ}|]}_{\text{Term (B)}}.
    \end{equation}

    By Lemma~\ref{lem:oneshot-residual-anytime}, the accumulated residual is bounded by:
    \begin{equation*}
        \sum_{t\in J^{\circ}}\|\hat g_{t}-m_{t}\|^{2} \;\le\; 8dS.
    \end{equation*}
    Substituting $\eta = \frac{D}{\sqrt{8dS}}$, we have:
    \begin{equation*}
        \text{Term (A)}=3\sqrt{2}D\sqrt{dS}
    \end{equation*}
    Since the stopping rule states $|J^{\circ}| \le H$ and
    $\delta\leq \frac{\sqrt{S}}{d\beta H}$, we have
    \begin{align*}
        \frac{Dd\beta}{2}(\delta \mathbb{E}[|J^{\circ}|])            & \le \frac{Dd\beta}{2}(\delta H) \le \frac{Dd\beta}{2}\cdot \frac{\sqrt{S}}{d\beta}= \frac{D}{2}\sqrt{S}, \\
        \left(L+\frac{LD}{r}\right) (\delta \mathbb{E}[|J^{\circ}|]) & \le \left(L+\frac{LD}{r}\right) \frac{\sqrt{S}}{d\beta}.
    \end{align*}
    Besides, for $t\in O$, we have 
    \begin{equation*}
        \sum_{t\in O}(f_t(x_t)-f_t(x^\star))\leq LD|O|\leq LD.
    \end{equation*}
    Thus, all terms combined yield the conclusion of the proof.
\end{proof}

\begin{lemma}\label{lem:sum-epochs-in-phase}
Let $S_{k,1},\cdots,S_{k,E_k}$ denote the budgets of prediction sensitivity used during phase $k$ where $S_{k,e}=S_{\min}\cdot2^{e-1}$ for $e=1,\ldots,E_k$. Then, we have the following result:
\begin{equation*}
    \sum_{e=1}^{E_k}\sqrt{S_{k,e}}
    \le 4\sqrt{S_{k,E_k}}.
\end{equation*}
\end{lemma}

\begin{proof}
Since $S_{k,e}=S_{\min}\cdot2^{e-1}$, we first have
\begin{equation*}
\begin{aligned}
    \sum_{e=1}^{E_k}\sqrt{S_{k,e}}
    &=
    \sqrt{S_{\min}}
    \sum_{e=1}^{E_k}2^{(e-1)/2}  \\
    &=
    \sqrt{S_{\min}}
    \sum_{j=0}^{E_k-1}(\sqrt{2})^j  \\
    &=
    \sqrt{S_{\min}}\,
    \frac{(\sqrt{2})^{E_k}-1}{\sqrt{2}-1}.
\end{aligned}
\end{equation*}
Using $(\sqrt{2})^{E_k}-1<(\sqrt{2})^{E_k}$, we obtain
\begin{equation*}
\begin{aligned}
    \sum_{e=1}^{E_k}\sqrt{S_{k,e}}
    &<
    \sqrt{S_{\min}}\,
    \frac{(\sqrt{2})^{E_k}}{\sqrt{2}-1}  \\
    &=
    \frac{\sqrt{2}}{\sqrt{2}-1}
    \sqrt{S_{\min}}(\sqrt{2})^{E_k-1}  \\
    &=
    (2+\sqrt{2})\sqrt{S_{k,E_k}}
    \le 4\sqrt{S_{k,E_k}} .
\end{aligned}
\end{equation*}
This completes the proof.
\end{proof}

\begin{lemma}
    \label{lem:finalS-phase} Let $R^{\mathrm{obs}}_{k}= \sum_{t\in I_k}\|\hat g
    _{t}-m_{t}\|^{2}$ denote the accumulated observable residual in phase $k$,
    and let $S_{k,E_k}$ be the final sensitivity budget used during phase $k$. Then, we have:
    \begin{equation*}
        S_{k,E_k}\le S_{\min} + \frac{R^{\mathrm{obs}}_{k}}{8d}.
    \end{equation*}
\end{lemma}

\begin{proof}
    Given the number of epochs $E_k$ in phase $k$, we will consider the proofs in two cases.

    \textbf{$E_{k}=1$:} The prediction-sensitivity budget is at its starting value of $S_{k,1}=S_{\min}$.
    As $R^{\mathrm{obs}}_{k} \ge 0$, we see $S_{\min} \le S_{\min} + R^{\mathrm{obs}}_{k}/(8d)$ trivially.

    \textbf{$E_{k} \ge 2$:} The fact that the algorithm has progressed to epoch $E_{k}$
    implies that all previous epochs $e=1,\dots,E_{k}-1$ were terminated because
    the residual threshold has been exceeded. For all $e < E_{k}$, this gives $R_{k,e}> 8
    d S_{k,e}$. Summing these residuals, we obtain:
    \begin{equation*}
        R^{\mathrm{obs}}_{k} \ge \sum_{e=1}^{E_k-1}R_{k,e}> 8d \sum_{e=1}^{E_k-1}
        S_{k,e}.
    \end{equation*}
    Since $S_{k,e}= S_{\min}\cdot 2^{e-1}$, we obtain the telescopic series:
    \begin{equation*}
        \sum_{e=1}^{E_k-1}S_{k,e}= S_{\min}\cdot\sum_{e=1}^{E_k-1}2^{e-1}= S_{\min}\cdot(2^{E_k-1}- 1) = S_{k,E_k}
        - S_{\min}.
    \end{equation*}
    Substituting this back, we obtain: $R^{\mathrm{obs}}_{k} > 8d(S_{k,E_k}-S_{\min})$.
    Rearranging terms yields the desired bound.
\end{proof}

\begin{lemma}[Observed residual in a phase]
\label{lem:phase-residual-vs-Sbar}
Let $\regret_k^{\mathrm{obs}}=
\sum_{t\in I_k}\|\hat g_t-m_t\|^2$ be the accumulated observable residual in phase $k$. Suppose that the
sensitivity budgets in phase $k$ satisfy
\begin{equation*}
    S_{k,e}=S_{\min}2^{e-1},\qquad e=1,\ldots,E_k,
\end{equation*}
and that, within each epoch $(k,e)$, the perturbation radius is chosen as
\begin{equation*}
    \delta_{k,e}
    =
    \min\left\{
        \frac{\sqrt{S_{k,e}}}{d\beta H_k},
        \frac{1}{\beta\sqrt{H_k}},
        \frac r2
    \right\}.
\end{equation*}
Then, we have
\begin{equation*}
\mathbb E[\regret_k^{\mathrm{obs}}]
\le
\frac{
    4d\,\bar S_k+\frac{9}{2}\frac{S_{\min}}{H_k}
}{
    1-\frac{9}{16dH_k}
}.
\end{equation*}
In particular, since $d\ge 1$ and $H_k\ge 1$,
\begin{equation*}
\mathbb E[\regret_k^{\mathrm{obs}}]
\le
\frac{64}{7}d\,\bar S_k
+
\frac{72}{7}\frac{S_{\min}}{H_k}.
\end{equation*}
Consequently, the final sensitivity budget in phase $k$ satisfies
\begin{equation*}
\mathbb E[S_{k,E_k}]
\le
S_{\min}
+
\frac{1}{8d}
\cdot
\frac{
    4d\,\bar S_k+\frac{9}{2}\frac{S_{\min}}{H_k}
}{
    1-\frac{9}{16dH_k}
},
\end{equation*}
and, in particular,
\begin{equation*}
\mathbb E[S_{k,E_k}]
\le
\frac{8}{7}\bar S_k
+
\frac{16}{7}S_{\min}.
\end{equation*}
\end{lemma}

\begin{proof}
Let $Q_k=\mathbb E[\regret_k^{\mathrm{obs}}]$. From Corollary~\ref{cor:second-moment-play}, for every round $t\in I_k$, we have
\begin{equation*}
\mathbb E\!\left[
    \|\hat g_t-m_t\|^2
    \mid \mathcal H_{t-1}
\right]
\le
4d\|\nabla f_t(x_t)-m_t\|^2
+
(d^2/2+4d)\beta^2\delta_t^2,
\end{equation*}
where $\delta_t$ is the perturbation radius used at round $t$. Summing this
inequality over all rounds $t\in I_k$ and then taking total expectation gives
\begin{equation*}
\begin{aligned}
    Q_k
    &=
    \mathbb E\left[
        \sum_{t\in I_k}\|\hat g_t-m_t\|^2
    \right]  \\
    &\le
    4d\,
    \mathbb E\left[
        \sum_{t\in I_k}\|\nabla f_t(x_t)-m_t\|^2
    \right]
    +
    (d^2/2+4d)\beta^2
    \mathbb E\left[
        \sum_{t\in I_k}\delta_t^2
    \right]  \\
    &=
    4d\,\bar S_k
    +
    (d^2/2+4d)\beta^2
    \mathbb E\left[
        \sum_{t\in I_k}\delta_t^2
    \right].
\end{aligned}
\end{equation*}
We now decompose the last sum epoch by epoch. Since the perturbation radius is
fixed within each epoch $(k,e)$, we can write
\begin{equation*}
    \sum_{t\in I_k}\delta_t^2
    =
    \sum_{e=1}^{E_k}|J_{k,e}|\delta_{k,e}^2.
\end{equation*}
For every epoch $(k,e)$, by the definition of $\delta_{k,e}$,
\begin{equation*}
    \delta_{k,e}
    \le
    \frac{\sqrt{S_{k,e}}}{d\beta H_k}.
\end{equation*}
Therefore,
\begin{equation*}
\begin{aligned}
    (d^2/2+4d)\beta^2\delta_{k,e}^2
    &\le
    (d^2/2+4d)\beta^2
    \cdot
    \frac{S_{k,e}}{d^2\beta^2H_k^2}  \\
    &=
    \left(\frac12+\frac4d\right)
    \frac{S_{k,e}}{H_k^2}.
\end{aligned}
\end{equation*}
Since $d\ge1$, we have $\frac12+\frac4d\le \frac92$. Hence,
\begin{equation*}
    (d^2/2+4d)\beta^2\delta_{k,e}^2
    \le
    \frac92\frac{S_{k,e}}{H_k^2}.
\end{equation*}
Substituting this bound into the epoch-wise decomposition yields
\begin{equation*}
\begin{aligned}
    (d^2/2+4d)\beta^2
    \sum_{e=1}^{E_k}|J_{k,e}|\delta_{k,e}^2
    &\le
    \frac92
    \sum_{e=1}^{E_k}
    |J_{k,e}|\frac{S_{k,e}}{H_k^2}.
\end{aligned}
\end{equation*}
The budgets $S_{k,e}$ are nondecreasing in $e$, and hence
$S_{k,e}\le S_{k,E_k}$ for every $e\le E_k$. Moreover, the total number of
rounds in phase $k$ is at most $H_k$, so
\begin{equation*}
    \sum_{e=1}^{E_k}|J_{k,e}|=|I_k|\le H_k.
\end{equation*}
Therefore, pathwise,
\begin{equation*}
\begin{aligned}
    \sum_{e=1}^{E_k}
    |J_{k,e}|\frac{S_{k,e}}{H_k^2}
    &\le
    \sum_{e=1}^{E_k}
    |J_{k,e}|\frac{S_{k,E_k}}{H_k^2}  \\
    &=
    \frac{S_{k,E_k}}{H_k^2}
    \sum_{e=1}^{E_k}|J_{k,e}|  \\
    &\le
    \frac{S_{k,E_k}}{H_k}.
\end{aligned}
\end{equation*}
Combining the above estimates gives
\begin{equation*}
    Q_k
    \le
    4d\,\bar S_k
    +
    \frac{9}{2H_k}\mathbb E[S_{k,E_k}].
\end{equation*}

It remains to control the final sensitivity budget $S_{k,E_k}$. By
Lemma~\ref{lem:finalS-phase}, we have the pathwise bound
\begin{equation*}
    S_{k,E_k}
    \le
    S_{\min}
    +
    \frac{R_k^{\mathrm{obs}}}{8d}.
\end{equation*}
Taking expectation on both sides gives
\begin{equation*}
    \mathbb E[S_{k,E_k}]
    \le
    S_{\min}
    +
    \frac{Q_k}{8d}.
\end{equation*}
Substituting this inequality into the previous estimate for $Q_k$, we obtain
\begin{equation*}
\begin{aligned}
    Q_k
    \le
    4d\,\bar S_k
    +
    \frac{9}{2H_k}
    \left(
        S_{\min}
        +
        \frac{Q_k}{8d}
    \right)  =
    4d\,\bar S_k
    +
    \frac{9}{2}\frac{S_{\min}}{H_k}
    +
    \frac{9}{16dH_k}Q_k.
\end{aligned}
\end{equation*}
Moving the last term to the left-hand side gives
\begin{equation*}
    \left(
        1-\frac{9}{16dH_k}
    \right)Q_k
    \le
    4d\,\bar S_k
    +
    \frac{9}{2}\frac{S_{\min}}{H_k}.
\end{equation*}
Therefore,
\begin{equation*}
    Q_k
    \le
    \frac{
        4d\,\bar S_k+\frac{9}{2}\frac{S_{\min}}{H_k}
    }{
        1-\frac{9}{16dH_k}
    }.
\end{equation*}
This proves the first claimed bound.

Since $d\ge1$ and $H_k\ge1$, we have
\begin{equation*}
    1-\frac{9}{16dH_k}
    \ge
    1-\frac{9}{16}
    =
    \frac{7}{16}.
\end{equation*}
Thus,
\begin{equation*}
\begin{aligned}
    Q_k
    \le
    \frac{16}{7}
    \left(
        4d\,\bar S_k
        +
        \frac{9}{2}\frac{S_{\min}}{H_k}
    \right) =
    \frac{64}{7}d\,\bar S_k
    +
    \frac{72}{7}\frac{S_{\min}}{H_k}.
\end{aligned}
\end{equation*}
This proves the simplified residual bound.

Finally, applying Lemma~\ref{lem:finalS-phase} once more,
\begin{equation*}
    \mathbb E[S_{k,E_k}]
    \le
    S_{\min}
    +
    \frac{Q_k}{8d}.
\end{equation*}
Using the explicit bound on $Q_k$ gives
\begin{equation*}
\begin{aligned}
    \mathbb E[S_{k,E_k}]
    \le
    S_{\min}
    +
    \frac{1}{8d}
    \left(
        \frac{64}{7}d\,\bar S_k
        +
        \frac{72}{7}\frac{S_{\min}}{H_k}
    \right) =
    S_{\min}
    +
    \frac{8}{7}\bar S_k
    +
    \frac{9}{7dH_k}S_{\min}.
\end{aligned}
\end{equation*}
Since $dH_k\ge1$,
\begin{equation*}
    \frac{9}{7dH_k}S_{\min}
    \le
    \frac97S_{\min}.
\end{equation*}
Hence,
\begin{equation*}
    \mathbb E[S_{k,E_k}]
    \le
    \frac87\bar S_k
    +
    \frac{16}{7}S_{\min}.
\end{equation*}
This completes the proof.
\end{proof}

\begin{lemma}
    \label{lem:phase-regret} Let $\Lambda_{k}= \bar S_{k}+ 2S_{\min}$. The expected regret associated with phase $k$ is bounded by:
    \begin{equation*}
        \mathbb{E}[\regret(I_{k})] \le \widetilde O\left(D\sqrt{d\Lambda_k}
    \right).
    \end{equation*}
\end{lemma}

\begin{proof}
    The regret in phase $k$ is the sum of the regrets incurred in each epoch $e=1,
    \dots,E_{k}$. Based on Lemma~\ref{lem:epoch-regret-anytime}, the regret for each
    epoch is bounded by terms proportional to $\sqrt{S_{k,e}}$ and $1/\sqrt{S_{k,e}}$.
    Summing these bounds over all epochs and applying the fact
    $\sum \sqrt{S_{k,e}}\le 4\sqrt{S_{k,E_k}}$
    from Lemma~\ref{lem:sum-epochs-in-phase}, we obtain:
    \begin{equation*}
        \sum_{e=1}^{E_k}\mathbb{E}[\regret(J_{k,e})] \le 12\sqrt{2}D\sqrt{d\,S_{k,E_k}}
        + 2 D\sqrt{S_{k,E_k}}+4\left(L+\frac{LD}{r}\right)\frac{\sqrt{S_{k,E_k}}}{d\beta}+LDE_k.
    \end{equation*}

    Taking the total expectation over $S_{k,E_k}$ and applying Jensen's inequality to the concave function $\sqrt{x}$,
    we have:
    \begin{equation*}
        \mathbb{E}[\regret(I_{k})] \le 12\sqrt{2}\, D\sqrt{d\,\mathbb{E}[S_{k,E_k}]}
        + 2 D\sqrt{\mathbb{E}[S_{k,E_k}]}+4\left(L+\frac{LD}{r}\right)\frac{\sqrt{\mathbb{E}[S_{k,E_k}]}}{d\beta}+LD\log T.
    \end{equation*}

    From Lemma~\ref{lem:phase-residual-vs-Sbar}, we have
    \begin{equation*}
        \mathbb{E}[S_{k,E_k}] \le \frac87\bar S_k+\frac{16}{7}S_{\min}. 
    \end{equation*}
    Substituting it into the regret bound, we have 
    \begin{equation*}
        \mathbb{E}[\regret(I_{k})] \le \widetilde{O}\left(D\sqrt{d\Lambda_k}\right)
    \end{equation*}
    This completes the proof.
\end{proof}

\subsection{Proof of Theorem~\ref{thm:tp-vr-adaptive}}
\begin{proof}
    The total time horizon $T$ is partitioned into $K$ phases, indexed by
    $k=1, \dots, K$. Since the time budget doubles ($H_{k} = 2^{k-1}$), the
    number of phases is bounded by
    $K \le \lceil \log_{2} T \rceil + 1 = O(\log T)$. The total expected regret is
    the sum of the expected regret in all phases:
    \begin{equation*}
        \mathbb{E}[\regret_{T}] = \sum_{k=1}^{K} \mathbb{E}[\regret(I_{k})].
    \end{equation*}
    From Lemma~\ref{lem:phase-regret}, in each phase $k$, we have:
    \begin{equation*}
        \mathbb{E}[\regret(I_{k})] \le \widetilde{O}\left( D\sqrt{d\,\Lambda_{k}}\right),
    \end{equation*}
    where $\Lambda_{k} = \bar S_{k}+ 2S_{\min}$ and $\bar S_{k} = \mathbb{E}[\sum_{t
    \in I_k}\|\nabla f_{t}(x_{t}) - m_{t}\|^{2}]$. Then, by summing it over $k=1,
    \dots, K$:
    \begin{equation*} 
        \mathbb{E}[\regret_{T}] \le O\left( D\sqrt{d}\sum_{k=1}^{K}
        \sqrt{\Lambda_{k}}\right).
    \end{equation*}
    Using the Cauchy-Schwarz inequality, i.e., $\sum_{k=1}^{K} \sqrt{x_{k}}\le \sqrt{K
    \sum_{k=1}^{K} x_{k}}$ on the first term:
    \begin{equation*}
        \sum_{k=1}^{K} \sqrt{\Lambda_{k}}\le \sqrt{K \sum_{k=1}^{K} (\bar S_{k}+2S_{\min})}= \sqrt{K (\bar S_{T} + 2KS_{\min})}\le \sqrt{K \bar S_{T}}+ K\sqrt{2S_{\min}}
        ,
    \end{equation*}
    where we apply the fact that $\sum_{k=1}^K \bar S_{k} = \bar S_{T}$. Substituting $K = O(\log
    T)$, this term becomes:
    \begin{equation*}
        \sum_{k=1}^{K} \sqrt{\Lambda_{k}}\le O\left(\sqrt{\bar S_{T} \log T}+
        \log T\right).
    \end{equation*}
    The total expected regret is
    bounded by:
    \begin{equation*}
        \mathbb{E}[\regret_{T}] \le O\left( D\sqrt{d\,\bar S_{T} \log T}+D\sqrt{d}\log T\right).
    \end{equation*}
\end{proof}

\subsection{Proof of Theorem~\ref{thm:lb-linear}}

\begin{proof}
The proof proceeds by constructing a specific hard instance within the convex set $\cX$.
There exist two points $x_+, x_- \in \cX$ such that $\|x_+ - x_-\| = D$.
Define the unit vector $v= (x_+ - x_-)/D$ and the midpoint $x_0 = (x_+ + x_-)/2$.
By convexity, the line segment $J = \{x_0 + \alpha v : \alpha \in [-D/2, D/2]\}$ is contained entirely in $\cX$.

Let $(\zeta_t)_{t=1}^T$ be a sequence of i.i.d.\ Rademacher random variables, i.e., $\mathbb{P}(\zeta_t = +1) = \mathbb{P}(\zeta_t = -1) = 1/2$. 
We will define the linear loss function sequence by:
\begin{equation*}
f_t(x) = \langle g_t, x\rangle, \qquad \text{where} \quad g_t = \mu \zeta_t v,
\end{equation*}
for a scaling parameter $\mu \in (0, L]$ yet to be determined. The gradients satisfy $\|\nabla f_t(x)\| = \|\mu \zeta_t v\| = \mu \le L$, ensuring the functions are $L$-Lipschitz.

We now verify the intrinsic prediction-error level. Let $\Pi$ be the class of all
admissible non-anticipating prediction policies. We next show that the zero predictor is optimal among all admissible
non-anticipating predictors in the conditional mean-square sense. For every
admissible policy $\pi$, the prediction $m_t^\pi$ is $\mathcal H_{t-1}$-measurable,
whereas $g_t=\mu\zeta_t v$ is independent of $\mathcal H_{t-1}$ and satisfies
$\mathbb E[g_t\mid\mathcal H_{t-1}]=0$. Hence, we have
\begin{align}
    \mathbb E\|g_t-m_t^\pi\|^2
    &=
    \mathbb E\|g_t\|^2
    +
    \mathbb E\|m_t^\pi\|^2
    -
    2\mathbb E\langle g_t,m_t^\pi\rangle \nonumber \\
    &=
    \mu^2
    +
    \mathbb E\|m_t^\pi\|^2
    -
    2\mathbb E\left[
        \left\langle
            \mathbb E[g_t\mid\mathcal H_{t-1}],
            m_t^\pi
        \right\rangle
    \right] \nonumber \\
    &=
    \mu^2+\mathbb E\|m_t^\pi\|^2
    \ge
    \mu^2\label{eq:expectation_lower}.
\end{align}

The lower bound in~\eqref{eq:expectation_lower} is attained by the
admissible prediction policy $\pi^0$ that outputs
$m_t^{\pi^0}\equiv 0$ for all $t$. Hence $m_t\equiv0$ is the
conditional mean-square optimal non-anticipating predictor for this hard
instance. Consequently,
\begin{equation*}
    \bar S_T^\star(\mathcal D_S,A)
    =
    \inf_{\pi\in\Pi}
    \bar S_T^\pi(\mathcal D_S,A)
    =
    T\mu^2 .
\end{equation*}
Choosing $\mu=\sqrt{\frac{S}{T}}$ and using $S\le L^2T$, we have $\mu\le L$, so the losses remain
$L$-Lipschitz. With this choice, we have
\begin{equation*}
    \bar S_T^\star(\mathcal D_S,A)=S .
\end{equation*}

It remains to prove the regret lower bound. Fix any admissible prediction policy
$\pi\in\Pi$, and let $x_t^\pi$ denote the point at which the algorithm incurs loss at
round $t$ when using this prediction policy. Since $x_t^\pi$ is
$\mathcal H_{t-1}$-measurable, while $\zeta_t$ is independent of $\mathcal H_{t-1}$
and has mean zero, we obtain
\begin{equation*}
\begin{aligned}
    \mathbb E[f_t(x_t^\pi)]
    &=
    \mathbb E
    \left[
        \mathbb E[
            \langle \mu\zeta_t v,x_t^\pi\rangle
            \mid \mathcal H_{t-1}
        ]
    \right]  \\
    &=
    \mu\,
    \mathbb E
    \left[
        \langle v,x_t^\pi\rangle
        \mathbb E[\zeta_t\mid\mathcal H_{t-1}]
    \right]
    =
    0.
\end{aligned}
\end{equation*}
Hence, it follows that
\begin{equation*}
    \mathbb E\Big[\sum_{t=1}^T f_t(x_t^\pi)\Big]=0.
\end{equation*}

We now lower bound the regret by showing that the best fixed comparator in
hindsight can achieve a substantially negative cumulative loss. Let $Z_T=\sum_{t=1}^T \zeta_t$. Since the segment $J=\{x_0+\alpha v:\alpha\in[-D/2,D/2]\}$ is contained in $\cX$, the best comparator over $\cX$ is at least as good as the
best comparator restricted to $J$. Hence,
\begin{equation}
    \min_{x\in\cX}\sum_{t=1}^T f_t(x)
    \le
    \min_{\alpha\in[-D/2,D/2]}
    \sum_{t=1}^T f_t(x_0+\alpha v).
    \label{eq:restrict_to_segment}
\end{equation}
For any $\alpha\in[-D/2,D/2]$, using $f_t(x)=\langle \mu\zeta_t v,x\rangle$ gives
\begin{equation}
\begin{aligned}
    \sum_{t=1}^T f_t(x_0+\alpha v)
    &=
    \sum_{t=1}^T
    \langle \mu\zeta_t v,x_0+\alpha v\rangle  \\
    &=
    \mu\langle v,x_0\rangle Z_T
    +
    \mu\alpha Z_T .
\end{aligned}
\label{eq:loss_on_segment}
\end{equation}
The first term in~\eqref{eq:loss_on_segment} does not depend on $\alpha$.
Therefore, minimizing over $\alpha$ only requires minimizing the linear term
$\mu\alpha Z_T$ over the interval $[-D/2,D/2]$. The minimizer is attained at the
endpoint opposite to the sign of $Z_T$:
\begin{equation*}
    \alpha^\star
    =
    -\frac D2\,\operatorname{sign}(Z_T),
\end{equation*}
with the convention $\operatorname{sign}(0)=0$. Substituting this choice into
\eqref{eq:loss_on_segment} yields
\begin{equation}
    \min_{\alpha\in[-D/2,D/2]}
    \sum_{t=1}^T f_t(x_0+\alpha v)
    \le
    \mu\langle v,x_0\rangle Z_T
    -
    \frac{\mu D}{2}|Z_T|.
    \label{eq:segment_min_loss}
\end{equation}
Combining~\eqref{eq:restrict_to_segment} and~\eqref{eq:segment_min_loss}, we obtain
\begin{equation}
    \min_{x\in\cX}\sum_{t=1}^T f_t(x)
    \le
    \mu\langle v,x_0\rangle Z_T
    -
    \frac{\mu D}{2}|Z_T|.
    \label{eq:comparator_loss_upper}
\end{equation}

We now take expectations. Since $\mathbb E[Z_T]=0$, the first term on the
right-hand side of~\eqref{eq:comparator_loss_upper} vanishes. Moreover, by
Khintchine's inequality~\cite[Lemma~A.9]{cesa2006prediction}, there exists a universal constant $c_0>0$ such that
\begin{equation*}
    \mathbb E|Z_T|
    =
    \mathbb E\left|\sum_{t=1}^T\zeta_t\right|
    \ge
    c_0\sqrt T .
\end{equation*}
Therefore, we have
\begin{equation}
    \mathbb E\left[
        \min_{x\in\cX}
        \sum_{t=1}^T f_t(x)
    \right]
    \le
    -
    \frac{c_0\mu D}{2}\sqrt T .
    \label{eq:expected_comparator_loss}
\end{equation}

Recall that for any admissible prediction policy $\pi\in\Pi$, the learner's
decision $x_t^\pi$ is $\mathcal H_{t-1}$-measurable, while $\zeta_t$ is independent
of $\mathcal H_{t-1}$ and has mean zero. Hence we have already shown that
\begin{equation*}
    \mathbb E\sum_{t=1}^T f_t(x_t^\pi)=0 .
\end{equation*}
Thus, for every admissible prediction policy $\pi\in\Pi$,
\begin{equation}
\begin{aligned}
    \mathbb E[\regret_T(\cA)]
    &=
    \mathbb E\sum_{t=1}^T f_t(x_t^\pi)
    -
    \mathbb E\left[
        \min_{x\in\cX}
        \sum_{t=1}^T f_t(x)
    \right]  \\
    &\ge
    0
    -
    \left(
        -\frac{c_0\mu D}{2}\sqrt T
    \right)  \\
    &=
    \frac{c_0\mu D}{2}\sqrt T .
\end{aligned}
\label{eq:regret_lower_before_scaling}
\end{equation}
Finally, substituting $\mu=\sqrt{S/T}$ into
\eqref{eq:regret_lower_before_scaling} gives
\begin{equation*}
    \mathbb E[\regret_T(\cA)]
    \ge
    \frac{c_0D}{2}\sqrt S .
\end{equation*}
The theorem follows with $c=c_0/2$.

\end{proof}

\begin{remark}[Nonzero Bayes predictors]
The lower bound is not tied to the special case in which the Bayes-optimal
predictor is zero. Indeed, one may add any predictable component to the
gradient sequence without changing the residual hardness. For example, on a
two-dimensional product domain, let $\nabla f_t = \mu e_1 + \sigma \zeta_t e_2$,
where $\mu e_1$ is known before round $t$, while $\zeta_t$ is an
independent Rademacher sign. The Bayes-optimal non-anticipating predictor is
$m_t^\star=\mu e_1$, and hence
$\sum_{t=1}^T \|\nabla f_t-m_t^\star\|^2=T\sigma^2$. The predictable component $\mu e_1$ can be perfectly removed by the
predictor, but the residual component $\sigma\zeta_t e_2$ still induces the
same standard randomized linear lower bound, yielding
\begin{equation*}
    \mathbb E[\regret_T]\ge cD\sigma\sqrt T
    =
    cD\sqrt{S}.
\end{equation*}
Thus the lower bound should be interpreted as a lower bound on the
unpredictable residual, not merely on the total gradient energy.
\end{remark}

    \section{Omitted Details of Section~\ref{sec:implications}}
\label{app:implications}

\subsection{Useful Lemmas in Section~\ref{sec:implications}}

\begin{lemma}[Dynamic optimistic projected-gradient inequality]
    \label{lem:dynamic_optimistic_ineq} Suppose that the sequences $\{y_{t}\}_{t=1}^{T}$ and $\{y'_{t}\}_{t=1}^{T}$ were generated using the optimistic update as rule Algorithm~\ref{alg:tp-vr-opt} with stepsize $\eta > 0$ on the set $\mathcal{X}_{\alpha}$. Then, for any comparator sequence $(u'_{t})_{t=1}^{T} \subseteq \mathcal{X}_{\alpha}$, we have the inequality:
    \begin{equation*}
        \sum_{t=1}^{T} \langle \hat g_{t},\, y_{t}-u'_{t}\rangle \le\frac{D^{2}+2D\,P_{T}(u'_{1:T})}{2\eta}
        +\eta\sum_{t=1}^{T} \|\hat g_{t}-m_{t}\|^{2},
    \end{equation*}
    where $D$ is the diameter of the domain ($\|\cX_{\alpha}\|\leq\|\cX\|\leq D$) and $P_{T}(u'_{1:T})=\sum_{t=2}^{T}\|u'_{t}-u'_{t-1}\|$ is the path-length of $\{u'_{t}\}_{t=1}^T$.
\end{lemma}

\begin{proof}
    For any fixed round $t$ and any fixed comparator $u'_{t} \in \mathcal{X}_{\alpha}$,
    Lemma~\ref{lem:optimistic-ogd} gives us the standard one-step analysis of the optimistic update for round $t$, which yields the following inequality:
    \begin{equation*}
        \langle \hat g_{t},\, y_{t}-u'_{t}\rangle \le\frac{\|u'_{t}-y'_{t}\|^{2}-\|u'_{t}-y'_{t+1}\|^{2}}{2\eta}
        +\eta\|\hat g_{t}-m_{t}\|^{2}.
    \end{equation*}
    Summing both sides of this inequality over all rounds gives us:
    \begin{equation*}
        \sum_{t=1}^{T} \langle \hat g_{t},\, y_{t}-u'_{t}\rangle \le\frac{1}{2\eta}
        \underbrace{\sum_{t=1}^T\Bigl(\|u'_t-y'_t\|^2-\|u'_t-y'_{t+1}\|^2\Bigr)}
        _{\Phi}+\eta\sum_{t=1}^{T}\|\hat g_{t}-m_{t}\|^{2}.
    \end{equation*}
    Now let us compute the potential term $\Phi$. The terms can be rearranged so that $y'_{t+1}$ can be combined:
    \begin{align*}
        \Phi & =\|u'_{1}-y'_{1}\|^{2} - \|u'_{T}-y'_{T+1}\|^{2} + \sum_{t=1}^{T-1}\|u'_{t}-y'_{t+1}\|^{2} - \sum_{t=2}^{T} \|u'_{t}-y'_{t}\|^{2} \\
        & =\|u'_{1}-y'_{1}\|^{2} - \|u'_{T}-y'_{T+1}\|^{2}+ \sum_{t=2}^{T} \Bigl( \|u'_{t}-y'_{t}\|^{2} - \|u'_{t-1}-y'_{t}\|^{2} \Bigr).
    \end{align*}

    Note that $\|u'_{T}-y'_{T+1}\|^{2}\le 0$ is non-positive, so we can omit it from $\Phi$. The upper bound $\|u'_{1}-y'_{1}\|^{2}\le D^{2}$. To compute the summation term, use the equality $\|a\|^{2}-\|b\|^{2}=\langle a-b,\,a+b\rangle$. Thus, for $a=u'_{t}-y'_{t}$ and $b=u'_{t-1}-y'_{t}$, we find that $a-b=u'_{t}-u'_{t-1}$ and $a+b=(u'_{t}-y'_{t})+(u'_{t-1}-y'_{t})$. So we get:
    \begin{align*}
        \|u'_{t}-y'_{t}\|^{2}-\|u'_{t-1}-y'_{t}\|^{2} & =\langle u'_{t}-u'_{t-1},\, (u'_{t}-y'_{t})+(u'_{t-1}-y'_{t})\rangle\\
        & \le\|u'_{t}-u'_{t-1}\| \cdot \Bigl(\|u'_{t}-y'_{t}\|+\|u'_{t-1}-y'_{t}\|\Bigr) \\
        & \le\|u'_{t}-u'_{t-1}\| \cdot (D+D) =2D\|u'_{t}-u'_{t-1}\|.
    \end{align*}
    Using the fact that $u'_{t},u'_{t-1},y'_{t}\in \cX_{\alpha}$, we obtain that the diameter of $\cX_{\alpha}$ is at most $D$. By summing this inequality over $t=2,\ldots,T$, we obtain 
    \begin{equation*}
        \Phi \le D^{2} + 2D \sum_{t=2}^{T} \|u'_{t}-u'_{t-1}\| = D^{2} +
        2D\,P_{T}(u'_{1:T}).
    \end{equation*}
    These complete the proof.
\end{proof}

\begin{lemma}
    \label{lem:epoch-dynamic-optimistic} Let $\mathcal{\cX}_{\alpha} \subseteq \mathbb{R}
    ^{d}$ be a convex set with Euclidean diameter at most $D$. We consider the
    optimistic updates applied to $\mathcal{X}_{\alpha}$:
    \begin{equation*}
        y_{t} = \Pi_{\mathcal{X}_\alpha}(y'_{t} - \eta m_{t}), \qquad y'_{t+1}= \Pi
        _{\mathcal{X}_\alpha}(y'_{t} - \eta \hat g_{t}),
    \end{equation*}
    where $\eta>0$. For any choice of a comparator sequence $\{u_t'\}_{t \in J} \subseteq \mathcal{X}_{\alpha}$ with respect to index set $J$,
    \begin{equation*}
        \sum_{t \in J}\langle \hat g_{t}, y_{t} - u'_{t}\rangle \le\frac{\|u'_{t_0}-y'_{t_0}\|^{2}}{2\eta}
        +\eta\sum_{t \in J}\|\hat g_{t}-m_{t}\|^{2} +\frac{D}{\eta}\sum_{t\in J,
        t>t_0}\|u'_{t}-u'_{t-1}\|,
    \end{equation*}
    where $t_{0} = \min J$.
\end{lemma}

\begin{proof}
    The proof is the same as in Lemma~\ref{lem:dynamic_optimistic_ineq}, except that we take our sum over the index set $J$, not over all rounds.
\end{proof}

\begin{definition}[Epoch Path-length]
    For any epoch $(k, e)$ in \tpvr++, let $J_{k, e}$ be the set of the active rounds. The \emph{epoch path-length} $P_{k, e}$ is defined to be the total movement of the comparator sequence across the indices of $J_{k,e}$. 
    \begin{equation*}
        P_{k,e}= \sum_{t \in J_{k,e} \setminus \{\min J_{k,e}\}}\|u_{t} - u_{t-1}\|.
    \end{equation*}
    For the good prefix $J_{k,e}^{\circ}$, we use the same notation $P_{k,e}$ as an upper bound on its path-length, since the path-length on $J_{k,e}^{\circ}$ is no larger than that on $J_{k,e}$.
\end{definition}

\begin{lemma}[Optimistic Hedge Regret]
    \label{lem:optimistic-hedge-geometric} The Optimistic Hedge algorithm is defined on the simplex $\Delta_{N}$ with a prior distribution $w_1$ and a step-size $\epsilon$. The loss vector at time $t$ is denoted by $\ell_{t}$ and the prediction vector at time $t$ is given by $M_{t}$. The regret to any expert $i^{\star}$ at any time $T$ is bounded by:
    \begin{equation*}
        \sum_{t=1}^{T} (\langle p_{t}, \ell_{t} \rangle - \ell_{t,i^\star}) \le \frac{\ln(1/w_{1,i^\star})}{\varepsilon}
        + \frac{\varepsilon}{2} \sum_{t=1}^{T} \|\ell_{t} - M_{t}\|_{\infty}^{2}.
    \end{equation*}
    Specifically, for the setting of Algorithm~\ref{alg:tp-vr-opt-pp} where $\ell_{t,i}
    = \langle \hat{g}_{t}, y_{t,i}\rangle$, and $M_{t,i}= \langle m_{t}, y_{t,i}\rangle$, we have the following:
    \begin{equation*}
        \|\ell_{t} - M_{t}\|_{\infty}^{2} \le D^{2} \|\hat{g}_{t} - m_{t}\|_{2}^{2}.
    \end{equation*}
\end{lemma}
\begin{proof}
    We note that the meta-algorithm is an example of Optimistic Mirror Descent (OMD) employing the regularizer $\psi(w)=\sum_{i=1}^{N} w_{i} \ln w_{i}$ on the probability simplex $\Delta_{N}$. The Bregman Divergence associated with this will be the KL Divergence, defined as $D_{\psi}(p,q) = \sum_{i=1}^{N} p_{i} \ln(p_{i}/q_{i})$. We can interpret the computations performed by the algorithm as an equivalent two-step OMD process: First, we predict 
    \begin{equation*}
    p_{t} = \argmin_{p \in \Delta_N}\{ \langle M_{t}
    , p \rangle + \frac{1}{\varepsilon}D_{\psi}(p, w_{t}) \}
    \end{equation*}
    and then update $w_{t+1}= \argmin_{w \in \Delta_N}\{ \langle \ell_{t}, w \rangle + \frac{1}{\varepsilon}
    D_{\psi}(w, w_{t}) \}$.

    The regret of the optimistic meta-algorithm at time $t$ is given by $\langle p_{t} - e_{i^\star},\ell_{t} \rangle$, where $e_{i^\star} \in \Delta_{N}$ is the standard basis vector corresponding to expert $i^\star$. We can decompose the meta regret as:  
    \begin{equation*}
        \langle \ell_{t}, p_{t} - e_{i^\star} \rangle = \langle \ell_{t}, p_{t} - w_{t+1}\rangle + \langle \ell_{t}, w_{t+1}- e_{i^\star} \rangle.
    \end{equation*} 
    The optimality condition for $w_{t+1}$ leads to 
    \begin{equation*}
        \langle \ell_{t}, w_{t+1}- e_{i^\star} \rangle \le \frac{1}{\varepsilon}(D_{\psi}(e_{i^\star}, w_{t}) - D_{\psi}(e_{i^\star}, w_{t+1}) - D_{\psi}(w_{t+1}, w_{t})).
    \end{equation*} 
    Plugging this term into the decomposition allow us to conclude that
    \begin{equation*}
        \langle p_{t} - e_{i^\star}, \ell_{t} \rangle \le \frac{D_{\psi}(e_{i^\star}, w_{t}) - D_{\psi}(e_{i^\star},
        w_{t+1})}{\varepsilon}+ \langle p_{t} - w_{t+1}, \ell_{t} - M_{t} \rangle
        + \langle p_{t} - w_{t+1}, M_{t} \rangle - \frac{1}{\varepsilon}D_{\psi}(
        w_{t+1}, w_{t}).
    \end{equation*}
    Similarly, the optimality condition for the prediction $p_{t}$ gives 
    \begin{equation*}
        \langle p_{t} - w_{t+1}, M_{t} \rangle \le \frac{1}{\varepsilon}(D_{\psi}(w_{t+1}, w_{t}) - D
    _{\psi}(w_{t+1}, p_{t}) - D_{\psi}(p_{t}, w_{t})).
    \end{equation*}
    Then, by canceling $\frac{1}{\varepsilon}D_{\psi}(w_{t+1}, w_{t})$ and dropping the non-positive term $-\frac{1}{\varepsilon}D_{\psi}(p_{t}, w_{t})$ we have
    \begin{equation*}
        \langle p_{t} - e_{i^\star}, \ell_{t} \rangle \le \frac{D_{\psi}(e_{i^\star}, w_{t}) - D_{\psi}(e_{i^\star},
        w_{t+1})}{\varepsilon}+ \langle p_{t} - w_{t+1}, \ell_{t} - M_{t} \rangle
        - \frac{1}{\varepsilon}D_{\psi}(w_{t+1}, p_{t}).
    \end{equation*}
    Next, we bound the two terms   
    \begin{equation*}
    \langle p_{t} - w_{t+1}, \ell_{t} - M_{t} \rangle - \frac{1}{\varepsilon}D_{\psi}(w_{t+1}, p_{t}).
    \end{equation*}

    The regularizer $\psi$ is $1$-strongly convex with respect to the $\ell_{1}$-norm on the simplex and by Pinsker's inequality (Lemma~\ref{lem:pinsker}), we have  
    \begin{equation*}
    D_{\psi}(w_{t+1}, p_{t}) \ge \frac{1}{2}\|w_{t+1}- p_{t}\|_{1}^{2}.
    \end{equation*} 
    By Holder's inequality, we have:
    \begin{equation*}
    \langle p_{t} - w_{t+1}, \ell_{t} - M_{t} \rangle \le \|p_{t} - w_{t+1}\|_{1} \|\ell_{t} - M_{t}\|_{\infty}.
    \end{equation*}
    Using the Young's inequality $ab \le \frac{a^{2}}{2\lambda}+ \frac{\lambda b^{2}}{2}$ with $a=\|p_{t} - w_{t+1}\|_{1}$,
    $b=\|\ell_{t} - M_{t}\|_{\infty}$, and $\lambda = \varepsilon$, we obtain:
    \begin{equation*}
        \langle p_{t} - w_{t+1}, \ell_{t} - M_{t} \rangle \le \frac{1}{2\varepsilon}
        \|p_{t} - w_{t+1}\|_{1}^{2} + \frac{\varepsilon}{2}\|\ell_{t} - M_{t}\|_{\infty}
        ^{2}.
    \end{equation*}
    Substituting this inequality for $\langle p_{t}-w_{t+1},\ell_{t}-M_{t}\rangle$ into the decomposition of the regret yields the following by canceling the non-positive term $\frac{1}{2\varepsilon}\|p_{t} - w_{t+1}\|_{1}^{2}$ as:
    \begin{equation*}
        \langle p_{t} - u, \ell_{t} \rangle \le \frac{D_{\psi}(u, w_{t}) - D_{\psi}(u,
        w_{t+1})}{\varepsilon}+ \frac{\varepsilon}{2} \|\ell_{t} - M_{t}\|_{\infty}^{2}.
    \end{equation*}
    When sum over $t=1,\dots,T$, the first term gives rise to a telescoping sum:
    \begin{equation*}
    \frac{1}{\varepsilon}(D_{\psi}(u, w_{1}) - D_{\psi}(u, w_{T+1})) \le \frac{D_{\psi}(u,
    w_{1})}{\varepsilon}.
    \end{equation*} 
    
    Since $D_{\psi}(e_{i^\star}, w_{1}) = \ln(1/w_{1,i^\star})$, we
    establish the first inequality of the lemma, where $\langle \ell_{t}, p_{t} - e_{i^\star} \rangle=\langle p_{t}, \ell_{t} \rangle - \ell_{t,i^\star}$.

    Next, we further bound  $\|\ell_{t} - M_{t}\|_{\infty}^{2}$. By definition,
    \begin{equation*}
        \ell_{t,i}= \langle \hat{g}_{t}, y_{t,i}\rangle, \qquad
    M_{t,i}= \langle m_{t}, y_{t,i}\rangle.
    \end{equation*}

    The $\ell_{\infty}$-norm is therefore $\|\ell_{t}-M_{t}\|_{\infty} = \max_{i}|\langle \hat{g}_{t}-m_{t}, y_{t,i}\rangle|$. Using Cauchy-Schwarz inequality gives $|\langle \hat{g}_{t}-m_{t},y_{t,i}\rangle| \le \|\hat{g}_{t}-m_{t}\|_{2}\|y_{t,i}\|_{2}$. Since $0 \in \cX_{\alpha}$ and all the expert instances $y_{t,i}$ lies in $\mathcal{X}_{\alpha}$, thus $\|y_{t,i}\|_{2} \le D$ for all $i$. We see consequently that for all time horizons $t$, $\|\ell_{t}-M_{t}\|_{\infty} \le D\|\hat{g}_{t}-m_{t}\|_{2}$. Therefore, by squaring our result, we have $\|\ell_{t} - M_{t}\|_{\infty}^{2} \le D^{2}\|\hat{g}_{t}-m_{t}\|_{2}^{2}$. These complete the proof.
\end{proof}

\begin{lemma}[Meta-Expert Epoch Regret]
\label{lem:dyn-epoch}
Fix an epoch $(k,e)$, where we write $J_{k,e}$ to denote the time interval associated with an epoch, according to the Meta-Expert algorithm in Algorithm~\ref{alg:tp-vr-opt-pp}. Let $R_{k,e} = \sum_{t \in J_{k,e}} \|\hat{g}_t - m_t\|^2$ be the total prediction error for that epoch. We also introduce $P_{k,e}$ as the length of the epoch path and denote the sensitivity budget for this epoch by $S_{k,e}$. 
The expected dynamic epoch regret can be bounded as follows:
\begin{align*}
\mathbb{E}\bigl[\regret^{\mathrm{dyn}}(J_{k,e})\bigr]
&\le 4 D \sqrt{8d S_{k,e} \ln N} + \frac{3}{\sqrt{2}}\sqrt{R_{k,e}^{\circ}\cdot(D^2 + 2D P_{k,e})} \\&+ 3D+\frac{D}{2}\sqrt{S_{k,e}}+\left(L+\frac{LD}{r}\right)\frac{\sqrt{S_{k,e}}}{d\beta}+LD.
\end{align*}
\end{lemma}

\begin{proof}
The proof follows from Lemma~\ref{lem:epoch-regret-anytime}. If the epoch terminates because the phase time budget is reached, define $O_{k,e}=\emptyset, J_{k,e}^{\circ}=J_{k,e}$.
Thus $O_{k,e}$ contains at most one round, and $J_{k,e}^{\circ}$ is the
good prefix of the epoch. Define the good-prefix observable residual as
$R_{k,e}^{\circ}=\sum_{t\in J_{k,e}^{\circ}}\|\hat g_t-m_t\|^2$. We decompose as follows:
\begin{equation*}
    \regret^{\mathrm{dyn}}_{T}(J_{k,e})=\regret^{\mathrm{dyn}}_{T}(J_{k,e}^{\circ})+\regret^{\mathrm{dyn}}_{T}(O_{k,e}).
\end{equation*}

The $\regret^{\mathrm{dyn}}_{T}(O_{k,e})$ part can be bounded as
\begin{equation*}
    \regret^{\mathrm{dyn}}_{T}(O_{k,e})\leq LD.
\end{equation*}
Then, we consider the $\regret^{\mathrm{dyn}}_{T}(J_{k,e}^{\circ})$ part as follows:
\begin{equation*}
 \regret^{\mathrm{dyn}}_{T}(J_{k,e}^{\circ}) \le\sum_{t \in J_{k,e}^{\circ}} \langle \hat{g}_t, y_t - u'_t \rangle+ \Bigl(\frac{Dd\beta}{2}+L+\frac{LD}{r}\Bigr)\delta |J_{k,e}^{\circ}|.
\end{equation*}
The choice of $\delta$ that is selected for \tpvrplusplus\, is identical to the that selected for \tpvrplus, and thus we have:
\begin{equation*}
    \Bigl(\frac{Dd\beta}{2}+L+\frac{LD}{r}\Bigr)\delta |J_{k,e}^{\circ}|\leq \frac{D}{2}\sqrt{S_{k,e}}+\left(L+\frac{LD}{r}\right)\frac{\sqrt{S_{k,e}}}{d\beta}
\end{equation*}

Using the general form of $\ell_t(y_t)=\langle \hat{g}_t, y_t\rangle$ for use in \tpvrplusplus, we can further decompose the $\sum_{t \in J_{k,e}^{\circ}} \langle \hat{g}_t, y_t - u'_t \rangle$ term into Meta-Regret (A) and Best Expert Regret (B):
\begin{equation*}
\sum_{t \in J_{k,e}^{\circ}} \langle \hat{g}_t, y_t - u'_t \rangle
= \underbrace{\sum_{t \in J_{k,e}^{\circ}} (\ell_t(y_t) - \ell_t(y_{t, i^\star}))}_{\text{(A)}} + \underbrace{\sum_{t \in J_{k,e}^{\circ}} (\ell_t(y_{t, i^\star}) - \ell_t(u'_t))}_{\text{(B)}},
\end{equation*}
where $i^\star$ is the index of the best expert in the grid.

\textbf{Term (A) Meta Regret:}
The algorithm uses the prior $w_{1,i} \propto \frac{1}{i(i+1)}$ and step size $\varepsilon=\frac{\sqrt{\ln N}}{D\sqrt{V_{\max}}}$, where $V_{\max}=8dS_{k,e}$. For any expert $i^\star \le N$, we know that the prior satisfies $\ln(1/w_{1,i^\star}) \le 2\ln(N+1)$. As a consequence, by applying the bound given as Lemma~\ref{lem:optimistic-hedge-geometric}, we have:
\begin{equation*}
\text{Term (A)} \le \frac{2\ln(N+1)}{\varepsilon} + \frac{\varepsilon}{2} D^2 R_{k,e}^{\circ}.
\end{equation*}
Substituting $\varepsilon = \frac{\sqrt{\ln N}}{D\sqrt{V_{\max}}}$:
\begin{align*}
\text{Term (A)} &= \frac{2\ln(N+1) D \sqrt{V_{\max}}}{\sqrt{\ln N}} + \frac{\sqrt{\ln N}}{2 D \sqrt{V_{\max}}} D^2 R_{k,e}^{\circ}
\\&= D \sqrt{V_{\max}} \left( \frac{2\ln(N+1)}{\sqrt{\ln N}} + \frac{R_{k,e}^{\circ} \sqrt{\ln N}}{2 V_{\max}} \right).
\end{align*}

We also note that epochs terminate if the accumulated error is greater than $8dS_{k,e}$. As demonstrated in Lemma~\ref{lem:oneshot-residual-anytime}, we have 
\begin{equation*}
R_{k,e}^{\circ} \le 8dS_{k,e} = V_{\max}.
\end{equation*}
Thus, the fraction $\frac{R_{k,e}^{\circ}}{V_{\max}} \le 1$. Furthermore, since $N \ge 2$, $\ln(N+1)\le 1.6\ln N$, therefore $\frac{2\ln(N+1)}{\sqrt{\ln N}}\le 3.2\sqrt{\ln N}$.
Thus, we combine these:
\begin{equation*}
\text{Term (A)} \le D \sqrt{V_{\max}} \left( 3.2\sqrt{\ln N} + 0.5\sqrt{\ln N} \right)
= 4 D \sqrt{V_{\max} \ln N}.
\end{equation*}
Finally, by the definition of $V_{\max}$, we know:
\begin{equation*}
\text{Term (A)} \le 4 D \sqrt{8dS_{k,e} \ln N}.
\end{equation*}

\textbf{Term (B) Expert Regret:}
Let $A_{k,e}=D^{2}+2DP_{k,e}$. By Lemma~\ref{lem:epoch-dynamic-optimistic}, we can bound the regret of any specific expert $i$ by $R_{i}(\eta_{i})\le\frac{A_{k,e}}{2\eta_{i}}+\eta_{i}R_{k,e}^{\circ}$.
The optimal step size is defined by $\eta^{\star}=\sqrt{\frac{A_{k,e}}{2R_{k,e}}}$.

Let us analyze cases, but first verify that our grid indicates that we have covered the region for the optimal step size from below.
We have seen that the termination conditions of epochs are governed by the alteration schedule being sufficient to exceed the error budget, that is $R_{k,e}^{\circ}\le V_{\max}$. Thus, we can also observe that $A_{k,e}\ge D^{2}$, which implies that the optimal step size is lower-bounded by:
\begin{equation*}
    \eta^\star = \sqrt{\frac{A_{k,e}}{2R_{k,e}^{\circ}}} \ge \sqrt{\frac{D^2}{2 V_{\max}}} = \eta_0.
\end{equation*}
Therefore, we can begin by analyzing two cases based on the upper bound of our grid $\eta_{\max}$.

\begin{itemize}
    \item \textbf{Case 1 (Inside Grid Range, $\eta_0 \le \eta^\star \le \eta_{\max}$):}
     As the grid is geometric with ratio 2, $\eta_{i}=\eta_{0}2^{i-1}$, we know that there exists an expert $i^\star \in \{1,\ldots,N\}$ such that $\eta_{i^\star}\le\eta^{\star}\le2\eta_{i^\star}$. For this specific expert, we have that:
    \begin{equation*}
    \text{Term (B)} \le \frac{A_{k,e}}{2\eta_{i^\star}} + \eta_{i^\star} R_{k,e}^{\circ}
    \le \frac{A_{k,e}}{2(\eta^\star/2)} + \eta^\star R_{k,e}^{\circ}
    = \frac{A_{k,e}}{\eta^\star} + \eta^\star R_{k,e}^{\circ}.
    \end{equation*}
    Substituting $\eta^\star = \sqrt{\frac{A_{k,e}}{2R_{k,e}^{\circ}}}$ gives:
    \begin{equation*}
    \text{Term (B)} \le \sqrt{2 A_{k,e} R_{k,e}^{\circ}} + \sqrt{\frac{A_{k,e} R_{k,e}^{\circ}}{2}} 
    = \left(\sqrt{2} + \frac{1}{\sqrt{2}}\right)\sqrt{A_{k,e} R_{k,e}^{\circ}} 
    = \frac{3}{\sqrt{2}}\sqrt{A_{k,e} R_{k,e}^{\circ}}.
    \end{equation*}

    \item \textbf{Case 2 (Exceeds Grid Range, $\eta^\star > \eta_{\max}$):}
    The condition $\eta^{\star}>\eta_{\max}$ indicates $\sqrt{\frac{A_{k,e}}{2R_{k,e}^{\circ}}} > \eta_{\max}$, which implies $R_{k,e}^{\circ} < \frac{A_{k,e}}{2\eta_{\max}^2}$.
    In this case, we consider the largest expert $i=N$ having step size $\eta_{\max}$. Term (B) for this expert becomes:
    \begin{equation*}
    \text{Term (B)} \le \frac{A_{k,e}}{2\eta_{\max}} + \eta_{\max} R_{k,e}^{\circ}
    < \frac{A_{k,e}}{2\eta_{\max}} + \eta_{\max} \left( \frac{A_{k,e}}{2\eta_{\max}^2} \right)
    = \frac{A_{k,e}}{\eta_{\max}}.
    \end{equation*}
    We now utilize our specific choice of $N = \lceil \log_2(DH/\eta_0) \rceil + 1$. This construction guarantees:
    \begin{equation*}
    \eta_{\max} = \eta_0 2^{N-1} \ge \eta_0 2^{\log_2(DH/\eta_0)} = DH.
    \end{equation*}
    Given that $H$ can be thought of as representing the total number of steps associated with the worst-case charging of the path length, then we have that $P_{k,e}\leq HD$. Combining this with the fact that $\eta_{\max}\ge H$, we have:
    \begin{equation*}
    \text{Term (B)} < \frac{D^2 + 2D P_{k,e}}{DH} 
    \le \frac{D^2 + 2D(HD)}{DH} 
    = \frac{D}{H} + 2D 
    \le 3D.
    \end{equation*}
\end{itemize}

Combining the two components gives us our overall regret, therefore completing the proof.
\end{proof}

\begin{lemma}[Phase Regret Aggregation]
\label{lem:dyn-phase}
Let $k$ be the phase, and $I_k$ be the corresponding round sequence, which can be partitioned into $J_{k,1}, J_{k,2}, \dots , J_{k,E_k}$ through the doubling condition of the prediction-sensitivity budget being given by $S_{k,e}=S_{\min}\cdot2^{e-1}$ with $e=1,2,\dotsc,E_k$. The phase-wise cumulative expected prediction error is denoted as $\bar S_k$, and the total path length taken by phase $k$ is denoted as $P_k$. The expected dynamic regret for phase $k$ is bounded by:

\begin{equation*}
\mathbb{E}\left[\regret^{\mathrm{dyn}}(I_k)\right]
\le\widetilde{O}\!\left(\sqrt{d\bar S_k\,(D^2 + D P_k)}\right).
\end{equation*}
\end{lemma}

\begin{proof}
Let the epochs be collected into $I_k$, and thus we have a summation over the different path elements:
\begin{equation*}
\regret^{\mathrm{dyn}}(I_k) = \sum_{e=1}^{E_k} \regret^{\mathrm{dyn}}(J_{k,e}).
\end{equation*}
Using Lemma~\ref{lem:dyn-epoch}, we can write:
\begin{align*}
\sum_{e=1}^{E_k} \mathbb{E}[\regret^{\mathrm{dyn}}(J_{k,e})] 
&\le \mathbb{E}\Bigl[\sum_{e=1}^{E_k}  \underbrace{4 D \sqrt{V_{k,e} \ln N}}_{\text{Term (A)}} + \underbrace{\frac{3}{\sqrt{2}}\sqrt{R_{k,e}^{\circ} (D^2 + 2D P_{k,e})}}_{\text{Term (B)}}\Bigr] \\&+ \underbrace{(3D+LD)E_k}_{\text{Term (C)}}+ \mathbb{E}\Bigl[\underbrace{\frac{D}{2}\sqrt{S_{k,e}}+\left(L+\frac{LD}{r}\right)\frac{\sqrt{S_{k,e}}}{d\beta}}_{\text{Term (D)}}
\Bigr],
\end{align*}
where we denote $V_{k,e} = 8d S_{k,e}$.

For Term (A): Using Lemma~\ref{lem:sum-epochs-in-phase}, we have $\sum_{e=1}^{E_k}\sqrt{S_{k,e}}\leq 4\sqrt{S_{k,E_k}}$, thus we obtain:
\begin{equation*}
    \sum_{e=1}^{E_k} 4 D \sqrt{V_{k,e} \ln N}\le 16 D \sqrt{8d S_{k,E_k} \ln N}.
\end{equation*}

For Term (B): 
Denote $A_{k,e}=D^2 + 2D P_{k,e}$. By
Cauchy--Schwarz, we have
\begin{equation*}
\begin{aligned}
    \sum_{e=1}^{E_k}
    \sqrt{A_{k,e}R_{k,e}^{\circ}}
    &\le
    \sqrt{
        \sum_{e=1}^{E_k}A_{k,e}
    }
    \sqrt{
        \sum_{e=1}^{E_k}R_{k,e}^{\circ}
    }.
\end{aligned}
\end{equation*}
We first bound the path-length factor. Therefore, we have
\begin{equation*}
\begin{aligned}
    \sum_{e=1}^{E_k}A_{k,e}=
    \sum_{e=1}^{E_k}(D^2+2DP_{k,e})=
    E_kD^2+2D\sum_{e=1}^{E_k}P_{k,e}.
\end{aligned}
\end{equation*}
The sum of the epoch path-lengths is at most the phase path-length, i.e.,  
$\sum_{e=1}^{E_k}P_{k,e}\le P_k$. Moreover, $E_k$ is logarithmic in the phase length and sensitivity range, and
therefore it is absorbed by $\widetilde O(\cdot)$. Hence, pathwise,
\begin{equation*}
    \sum_{e=1}^{E_k}A_{k,e}
    \le
    \widetilde O(D^2+DP_k).
\end{equation*}
For the residual factor, since $J_{k,e}^{\circ}\subseteq J_{k,e}$, we have
\begin{equation*}
    \sum_{e=1}^{E_k}R_{k,e}^{\circ}
    \le
    \sum_{e=1}^{E_k}
    \sum_{t\in J_{k,e}}\|\hat g_t-m_t\|^2
    =
    R_k^{\mathrm{obs}}.
\end{equation*}
Combining these estimates yields
\begin{equation*}
\begin{aligned}
    \sum_{e=1}^{E_k}
    \sqrt{A_{k,e}R_{k,e}^{\circ}}
    &\le
    \widetilde O\left(
        \sqrt{(D^2+DP_k)R_k^{\mathrm{obs}}}
    \right).
\end{aligned}
\end{equation*}
Taking expectation and applying Jensen's inequality to the concave function
$\sqrt{x}$ gives
\begin{equation*}
\begin{aligned}
    \mathbb E\left[
        \sum_{e=1}^{E_k}
        \sqrt{A_{k,e}R_{k,e}^{\circ}}
    \right]
    &\le
    \widetilde O\left(
        \sqrt{
            (D^2+DP_k)\,
            \mathbb E[R_k^{\mathrm{obs}}]
        }
    \right).
\end{aligned}
\end{equation*}
By Lemma~\ref{lem:phase-residual-vs-Sbar}, we have
\begin{equation*}
    \mathbb E[R_k^{\mathrm{obs}}]
    \le
    \frac{64}{7}d\bar S_k
    +
    \frac{72}{7}\frac{S_{\min}}{H_k}.
\end{equation*}
Therefore, it follows that
\begin{align*}
    \mathbb E\left[
        \sum_{e=1}^{E_k}
        \sqrt{A_{k,e}R_{k,e}^{\circ}}
    \right]
    &\le
    \widetilde O\left(
        \sqrt{
            (D^2+DP_k)
            \left(
                d\bar S_k+\frac{S_{\min}}{H_k}
            \right)
        }
    \right) \\
    &\le
    \widetilde O\left(
        \sqrt{d\bar S_k(D^2+DP_k)}
        +
        \sqrt{
            \frac{S_{\min}}{H_k}(D^2+DP_k)
        }
    \right).
\end{align*}

It remains to simplify the second term. Since all comparators lie in
$\mathcal X$, whose diameter is at most $D$, the movement in one round is at
most $D$. Since phase $k$ has length at most $H_k$, we have $P_k\le DH_k$.
Thus, we have
\begin{equation*}
\begin{aligned}
    \sqrt{
        \frac{S_{\min}}{H_k}(D^2+DP_k)
    }
    &\le
    \sqrt{
        \frac{S_{\min}}{H_k}(D^2+D^2H_k)
    }  \\
    &=
    D\sqrt{
        S_{\min}\left(\frac{1}{H_k}+1\right)
    }  \\
    &\le
    \sqrt2\,D\sqrt{S_{\min}}.
\end{aligned}
\end{equation*}
Since $d\ge1$, this is bounded by $\sqrt2\,D\sqrt{dS_{\min}}$. Therefore the best-expert contribution is also bounded by
\begin{equation*}
    \widetilde O\left(
        \sqrt{d\bar S_k(D^2+DP_k)}
        +
        D\sqrt{dS_{\min}}
    \right).
\end{equation*}

For Term (D): From above and using $\sum_{e=1}^{E_k} \sqrt{S_{k,e}} \le 4\sqrt{S_{k,E_k}}$, we have
\begin{equation*}
\sum_{e=1}^{E_k} \frac{D}{2}\sqrt{S_{k,e}}+\left(L+\frac{LD}{r}\right)\frac{\sqrt{S_{k,e}}}{d\beta} \le \widetilde{O}\left( D\sqrt{S_{k,E_k}}\right).
\end{equation*}

Combining Terms (A), (B), (C) and (D), and noting that terms with $E_k\leq O(\log T)$ are absorbed into $\widetilde{O}$, we have:
\begin{equation*}
\mathbb{E}\bigl[\mathfrak{R}^{\mathrm{dyn}}(I_k)\bigr] \le \widetilde{O}\left( \sqrt{d\mathbb{E}[S_{k,E_k}] (D^2 + D P_k)}\right).
\end{equation*}
Substituting the condition $\mathbb{E}[S_{k,E_k}] = O(\bar{S}_k + S_{\min})$ as shown in Lemma~\ref{lem:phase-regret}, the dominant term becomes:
\begin{equation*}
\mathbb{E}\bigl[\mathfrak{R}^{\mathrm{dyn}}(I_k)\bigr] \le \widetilde{O}\left( \sqrt{d\bar{S}_k (D^2 + D P_k)} \right).
\end{equation*}
\end{proof}

\subsection{Proof of Theorem~\ref{thm:dynamic_regret}}
\begin{proof}
    The regret of the $t$-th round can be expressed using the following decomposition:
    \begin{equation*}
        f_{t}(x_{t})-f_{t}(u_{t}) =\underbrace{f_t(x_t)-f_t(y_t)}_{\text{Term (A)}}
        +\underbrace{f_t(y_t)-f_t(u'_t)}_{\text{Term (B)}}+\underbrace{f_t(u'_t)-f_t(u_t)}
        _{\text{Term (C)}}.
    \end{equation*}
    The Terms (A) and (C) can be bounded using the $L$-Lipschitz property of the function $f_{t}$, as described below:
    \begin{itemize}
        \item For (A): Since $x_{t} = y_{t} + \delta v_{t}$ with
            $\|v_{t}\|_{2}=1$, we have $\|x_{t}-y_{t}\|_{2} = \delta$, implying
            $f_{t}(x_{t})-f_{t}(y_{t}) \le L\delta$.

        \item For (C): Since $u'_{t} = (1-\alpha)u_{t}$ with $\alpha = \delta/r$,
            we have $\|u'_{t}-u_{t}\|_{2} = \alpha\|u_{t}\|_{2} \le \alpha D = \frac{D}{r}
            \delta$. Thus, $f_{t}(u'_{t})-f_{t}(u_{t}) \le \frac{LD}{r}\delta$.
    \end{itemize}
    Summing over all rounds $t=1,\dotsc,T$, we have:
    \begin{equation*}
        \regret^{\mathrm{dyn}}_{T}(u_{1:T}) \le\sum_{t=1}^{T} \bigl(f_{t}(y_{t})-
        f_{t}(u'_{t})\bigr)+ \Bigl(L+\frac{LD}{r}\Bigr)\delta T.
    \end{equation*}

    Let $g_{t}^{y} = \nabla f_{t}(y_{t})$. By the convexity of $f_{t}$,
    \begin{equation*}
        f_{t}(y_{t})-f_{t}(u'_{t}) \le \langle g_{t}^{y},\, y_{t}-u'_{t} \rangle.
    \end{equation*}
    Thus, we can rewrite our expression for regret given as:
    \begin{equation*}
        \regret^{\mathrm{dyn}}_{T}(u_{1:T}) \le\sum_{t=1}^{T} \langle g_{t}^{y},\,
        y_{t}-u'_{t} \rangle +\Bigl(L+\frac{LD}{r}\Bigr)\delta T.
    \end{equation*}

    We also know that:
    \begin{equation*}
        \langle g_{t}^{y},\, y_{t}-u'_{t} \rangle =\langle \hat g_{t},\, y_{t}-u'
        _{t} \rangle +\langle g_{t}^{y}-\hat g_{t},\, y_{t}-u'_{t} \rangle.
    \end{equation*}
    For the first term on the right-hand side, we can use Lemma~\ref{lem:dynamic_optimistic_ineq} to apply the bound. 

    For the second term, we apply Lemma~\ref{lem:bias}. Recall that $\mathbb{E}[\hat
    g_{t} \mid \mathcal{H}_{t-1}] = g_{t}^{y} + b_{t}$ with $\|b_{t}\|_{2} \le \frac{d}{2}
    \beta\delta$. Taking an expectation of both sides gives:
    \begin{equation*}
        \mathbb{E}\Bigl[\langle g_{t}^{y}-\hat g_{t},\, y_{t}-u'_{t} \rangle\Bigr] = \mathbb{E}\Bigl[\langle -b_{t},\, y_{t}-u'_{t} \rangle\Bigr] \le \mathbb{E}
        \Bigl[\|b_{t}\|_{2} \|y_{t}-u'_{t}\|_{2}\Bigr] \le \frac{Dd\beta\delta}{2},
    \end{equation*}
    where we used $\|y_{t}-u'_{t}\|_{2} \le D$. To obtain the total contribution of this term, we sum over all rounds $t$:
    \begin{equation}
        \mathbb{E}\left[\sum_{t=1}^{T} \langle g_{t}^{y}-\hat g_{t},\, y_{t}-u'_{t}
        \rangle\right] \;\le\; \frac{Dd\beta\delta T}{2}. \label{eq:dyn_bias_term}
    \end{equation}

    Therefore, we have
    \begin{equation*}
        \mathbb{E}[\regret^{\mathrm{dyn}}_{T}] \le\frac{D^{2}+2DP_{T}}{2\eta}+\eta
        \sum_{t=1}^{T} \mathbb{E}\|\hat g_{t}-m_{t}\|_{2}^{2} +\frac{Dd\beta\delta
        T}{2}+\Bigl(L+\frac{LD}{r}\Bigr)\delta T.
    \end{equation*}
    Finally, we can bound $\sum_{t=1}^{T}\mathbb{E}\|\hat g_{t}-m_{t}\|^{2}$ using Corollary~\ref{cor:second-moment-play}. Plugging in the bound derived from the second moment expression gives us the main result:
    \begin{align*}
        \mathbb{E}\bigl[\regret^{\mathrm{dyn}}_{T}(u_{1:T})\bigr] \le  \frac{D^{2}+ 2D P_{T}}{2\eta}+\eta \Bigl( 4d\,\bar{S}_{T}+ (d^{2}/2 + 4d)\beta^{2}\delta^{2}T \Bigr) \nonumber+\frac{D d \beta \delta T}{2}+ \Bigl( L + \frac{LD}{r}\Bigr) \delta T.
    \end{align*}

    By applying bounds on $\eta$ and $\delta$ suitably chosen as in Theorem~\ref{thm:dynamic_regret}, we obtain our final bound on the regret.
\end{proof}

\subsection{Proof of Theorem~\ref{thm:tp-vr-opt-pp}}
\begin{proof}
We divide the total time horizon into $K$ phases $\{I_1,I_2,\ldots,I_K\}$. The expected dynamic regret is equal to the sum of the expected dynamic regrets in all $K$ phases.
Let $\regret_k = \mathbb{E}\big[\regret^{\mathrm{dyn}}(I_k)\big]$. By using the linearity of expectation, we have:

\begin{equation*}
\mathbb{E}\bigl[\regret_T^{\mathrm{dyn}}(u_{1:T})\bigr] = \sum_{k=1}^K \regret_k.
\end{equation*}
From Lemma~\ref{lem:dyn-phase}, we know the following for phase $k$:
\begin{equation*}
\regret_k \le C \left(\sqrt{d\bar{S}_k(D^2 + D P_k)}\right) \cdot \text{polylog}(T),
\end{equation*}
where $C$ is a universal constant, $\bar{S}_k = \mathbb{E}[\sum_{t \in I_k} \|\nabla f_t(x_t) - m_t\|^2]$ is the cumulative prediction error in phase $k$, and $P_k$ is the path-length of the comparator in phase $k$.

Further, we can write:
\begin{align*}
\sum_{k=1}^K \regret_k 
&\le \widetilde{O}(1) \sum_{k=1}^K \sqrt{d\bar{S}_k(D^2 + D P_k)} \\
&\le \widetilde{O}(1) \sqrt{\sum_{k=1}^K d\bar{S}_k} \cdot \sqrt{\sum_{k=1}^K (D^2 + D P_k)}.
\end{align*}

In $\sum_{k=1}^K d\bar{S}_k$, we can replace the cumulative prediction errors with the overall prediction error $\bar{S}_T$, and all constant terms sum to $Kd^2$. Therefore,
\begin{equation*}
    \sum_{k=1}^K d\bar{S}_k = d\bar{S}_T.
\end{equation*}
Similarly, we see that the path lengths are additive across non-overlapping intervals, therefore $\sum_{k=1}^K P_k = P_T$, and the $K$ terms from the diameter squared sum to $K D^2$:
\begin{equation*}
    \sum_{k=1}^K (D^2 + D P_k) = K D^2 + D P_T \le O(D^2 \log T + D P_T).
\end{equation*}

We substitute the results back into the inequality to arrive at the following:
\begin{align*}
\mathbb{E}\bigl[\regret_T^{\mathrm{dyn}}(u_{1:T})\bigr] 
&\le \widetilde{O}\left( \sqrt{d\bar{S}_T (\log T \cdot D^2 + D P_T)} \right).
\end{align*}
We simplify the expression to the leading order term:
\begin{equation*}
\mathbb{E}\bigl[\regret_T^{\mathrm{dyn}}(u_{1:T})\bigr] \le \widetilde{O}\left( \sqrt{d \bar{S}_T (D^2 + D P_T)} \right).
\end{equation*}

\end{proof}

}

    \bibliographystyle{IEEEtran}
    \bibliography{ref2}

\end{document}